\PassOptionsToPackage{table}{xcolor} %
 
\documentclass[pmlr]{jmlr}%

\usepackage{longtable}%

 \usepackage{booktabs}
 \usepackage{longtable}
 \usepackage{multirow}
\usepackage{rotating}
\usepackage{bm}
\usepackage{bbm}
\usepackage{tikz}
\usepackage{caption}
\usetikzlibrary{matrix,decorations.pathreplacing,calc}

\usepackage[load-configurations=version-1]{siunitx} %

\makeatletter
\def\set@curr@file#1{\def\@curr@file{#1}} %
\makeatother

\theorembodyfont{\upshape}
\theoremheaderfont{\scshape}
\theorempostheader{:}
\theoremsep{\newline}

\jmlrvolume{182}
\jmlryear{2022}
\jmlrworkshop{Machine Learning for Healthcare}

\newcommand{\disteq}{\overset{d}{=}}
\newcommand{\distneq}{\overset{d}{\neq}}

\newcommand{\R}{\mathbb{R}}

\newcommand{\x}{\mathbf{x}}

\definecolor{jet_min}{rgb}{0.0, 0.0, 0.5}
\definecolor{jet_max}{rgb}{0.5, 0.0, 0.0}

\definecolor{mimic_w}{RGB}{73, 80, 132}
\definecolor{mimic_b}{RGB}{72, 156, 158}

\title[Disparate Censorship \& Undertesting: A Source of Label Bias]{Disparate Censorship \& Undertesting: A Source of Label Bias in Clinical Machine Learning}
\author{\Name{Trenton Chang}
\Email{ctrenton@umich.edu}\\ 
\addr Division of Computer Science and Engineering\\
University of Michigan\\
Ann Arbor, MI, USA
\AND
\Name{Michael W. Sjoding}
\Email{msjoding@med.umich.edu}\\
\addr Department of Internal Medicine\\
University of Michigan\\
Ann Arbor, MI, USA
\AND
\Name{Jenna Wiens}
\Email{wiensj@umich.edu}\\ 
\addr Division of Computer Science and Engineering\\
University of Michigan\\
Ann Arbor, MI, USA
} 

\begin{document}

\maketitle

\begin{abstract}
As machine learning (ML) models gain traction in clinical applications, understanding the impact of clinician and societal biases on ML models is increasingly important. While biases can arise in the labels used for model training, the many sources from which these biases arise are not yet well-studied. In this paper, we highlight \textit{disparate censorship} (\textit{i.e.,} differences in testing rates across patient groups) as a source of label bias that clinical ML models may amplify, potentially causing harm. 
Many patient risk-stratification models are trained using the results of clinician-ordered diagnostic and laboratory tests of labels.
Patients without test results are often assigned a negative label, which assumes that untested patients do not experience the outcome.
Since orders are affected by clinical and resource considerations, testing may not be uniform in patient populations, giving rise to disparate censorship.
Disparate censorship in patients of equivalent risk leads to \textit{undertesting} in certain groups, and in turn, more biased labels for such groups.
Using such biased labels in standard ML pipelines could contribute to gaps in model performance across patient groups.
Here, we theoretically and empirically characterize conditions in which disparate censorship or undertesting affect model performance across subgroups. Our findings call attention to disparate censorship as a source of label bias in clinical ML models.
\end{abstract}

\begin{figure}[t]
    \centering
    \includegraphics[width=\linewidth]{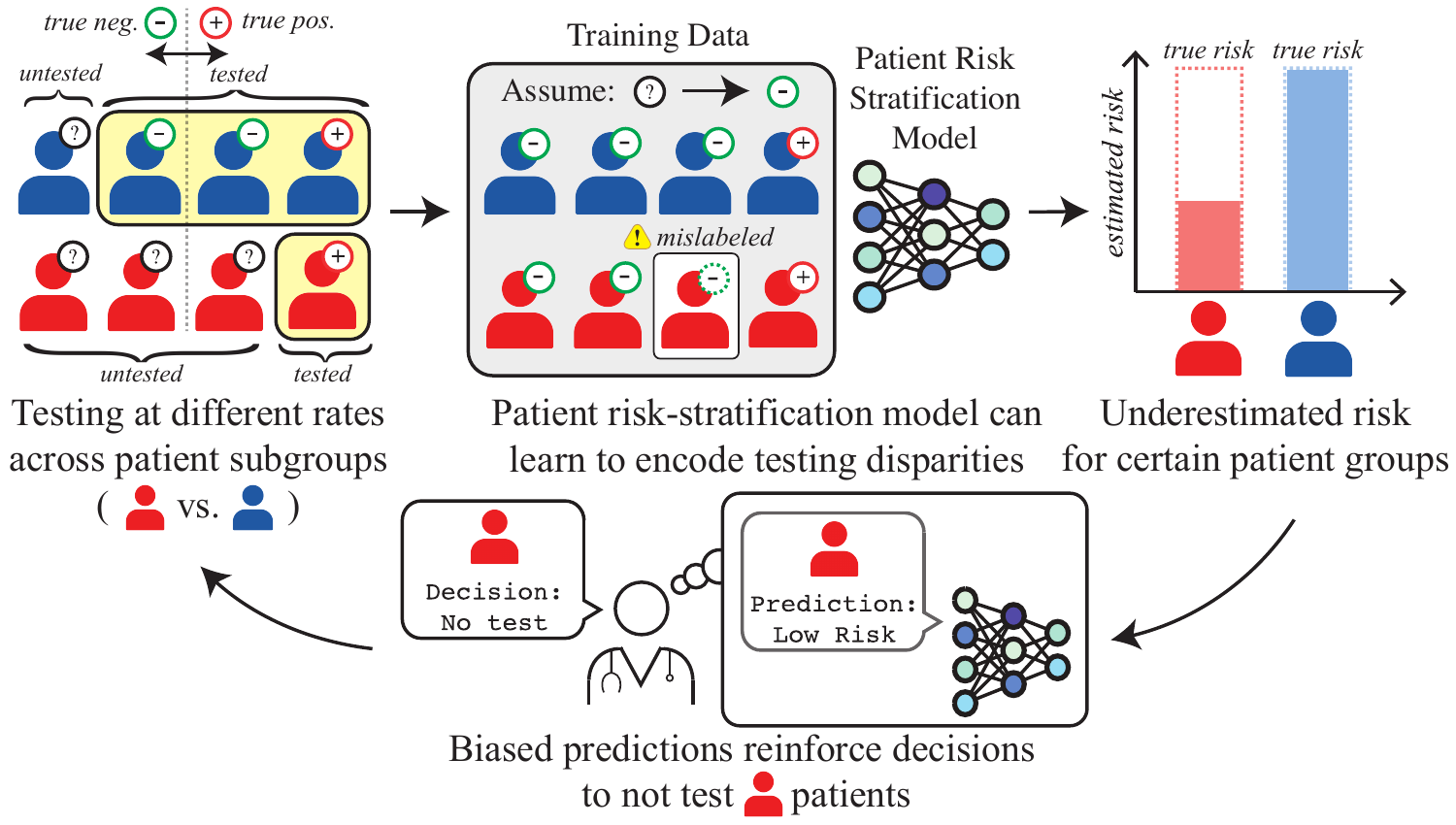}
    \caption{Disparate censorship can cause a harmful feedback loop in clinical ML workflows. Clockwise from top left: \textbf{a)} disparate censorship occurs when a patient subgroup is tested for some condition at a lower rate compared to other groups. \textbf{b)} When untested patients are assumed negative, if patients of equivalent risk are subject to disparate censorship (\emph{i.e.}, undertesting), standard ML training may learn to encode label bias (\emph{i.e.}, missed positives by clinicians) from undertesting. \textbf{c)} Such models may underestimate risk for certain patients. \textbf{d)} Acting on such predictions perpetuates undertesting---further harming already-underserved populations.}
    \label{fig:pull}
\end{figure}

\section{Introduction}

Medical applications are increasingly considering the usage of machine learning (ML) models. However, researchers have found that ML models may perform disproportionately poorly on marginalized groups~\citep{buolamwini2018gender, obermeyer2019dissecting, pierson2021algorithmic}. Biases in training data resulting from spurious correlations between the inputs and outputs have received much attention (\emph{i.e.}, ``shortcuts"~\cite{geirhos2020shortcut, jabbour2020deep}). However, biases can also arise in the labels used for model training. \cite{obermeyer2019dissecting} highlighted one such type of \textit{label bias} in equating healthcare need with healthcare cost led to downstream inequities in clinical care. We hypothesize that label bias may arise from other sources as well. Specifically, many researchers rely on assumptions about labels that could exacerbate pre-existing disparities in healthcare delivery.%

For the purposes of ML model training, patient outcomes are often defined based on laboratory/diagnostic test results extracted from the electronic health record (EHR; e.g.~\citep{rhee2020sepsis, seymour2016assessment, henry2019comparison}), since clinical chart review on large patient databases can be prohibitively costly.
In doing so, many researchers assign ``negative" labels to untested patients (the \emph{negativity} assumption in positive-unlabeled learning~\citep{bekker2020learning}). For example, many sepsis prediction models derive labels from laboratory test-based definitions, such that untested patients are negative by definition \citep{adams2022prospective,henry2015targeted, fleuren2020machine,reyna2019early}.
Beyond sepsis, this is also the case when building models to predict healthcare-associated infections~\citep{oh2018generalizable,teeple2020clinical, hartvigsen2018early}. Researchers typically justify this assumption, since without it, a model trained on \textit{only} patients who were tested may only apply to the small fraction of tested patients, limiting its utility.

This common approach to labeling patient outcomes in ML model development can have harmful downstream effects when patient groups are tested at different rates. We refer to this setting as \textit{disparate censorship}, which serves as the focus of our analyses. Examples of disparate censorship in clinical care include lower rates of colon cancer screening for Black patients~\citep{dolan2005colorectal}, or biases in cardiac evaluations for women~\citep{schulman1999effect}. %
When disparate censorship occurs in patients of equivalent risk from different groups, one group is \emph{undertested} relative to the other. Undertesting can result in higher rates of missed diagnoses/positives in certain patient group(s) as compared to other group(s), leading to disproportionate harm. In practice, disparate censorship and undertesting can be caused by pre-existing healthcare disparities such as clinician biases~\citep{schulman1999effect, daugherty2017implicit}, different levels of healthcare access or consent~\citep{spector2021respecting}, or different test performance across groups~\citep{gaffin2010clinically}.

\textbf{While the immediate harm of undertesting and missed diagnoses is clear~\citep{magesh2021disparities, berry2009examining}, ML has the potential to amplify this harm.} Patient risk-stratification models that do not account for the potential impact of disparate censorship and undertesting may underestimate the risk of the condition of interest.
This could reinforce a harmful feedback loop during ML model deployment (Figure~\ref{fig:pull}), in which models reinforce biased ``do not test" decisions.

In this paper, we characterize when disparate censorship and undertesting can result in ML model performance gaps across patient subgroups. %
We study three different settings.
In the first setting, both the covariates and the drivers of the outcome/disease are drawn from the same distributions for all patient subgroups of interest.
In the second setting, the \textit{marginal distribution} of covariates varies across groups. For example, social determinants of health such as education, socioeconomic status, and healthcare access may differ across race and gender~\citep{singh2017social}. Racial disparities in biomarkers indicating COVID-19 severity have also been documented~\citep{price2020hospitalization}.
In the third setting, the \textit{conditional probability} distribution of the outcome given the covariates may vary across subgroups. For example, the symptoms most indicative of coronary heart disease may differ between female and male patients~\citep{lichtman2018sex}. 

First, if the marginal and conditional risk distributions across groups are identical (first setting), neither disparate censorship nor undertesting result in performance gaps. However, in many healthcare settings, it is unlikely that the distribution of covariates and drivers of outcome/disease are identical. When differences in the marginal and conditional risk distributions arise, we show theoretically and validate empirically that disparate censorship can contribute to model performance gaps via certain patterns of undertesting. %
Then, we identify disparate censorship in clinical data and suggest practical approaches for identifying when disparate censorship may disproportionately negatively impact one group more than another. We encourage the ML for healthcare community to further explore methods for detecting and mitigating the negative effects of disparate censorship and undertesting.

\subsection*{Generalizable Insights about Machine Learning in the Context of Healthcare}

This paper highlights and analyzes how disparate censorship and undertesting can result in performance gaps in risk-stratification models when the underlying data generation processes differ across patient subgroups. Our contributions are as follows:

\begin{itemize}
    \item We introduce ``disparate censorship", in which patients are tested at different rates across groups, and formalize how undertesting, or testing disparities in patients with equivalent risk, can lead to gaps in ML model performance across patient subgroups.
    \item We show that undertesting can lead to model performance gaps across subgroups when certain differences in the marginal and conditional distributions of risk emerge.
    \item We validate our theory, demonstrating empirically how disparate censorship and undertesting lead to performance gaps across subgroups via a simulation study.
    \item We identify instances of disparate censorship in clinical data (MIMIC-IV) and suggest practical approaches for mitigating negative impacts.
\end{itemize}

\section{Problem Setup \& Definitions}

In this section, we propose a causal model of disparate censorship (Section~\ref{subsec:causal_model}), and outline three settings in which we study the impacts of disparate censorship (Section~\ref{subsec:distributional_diff}).

\subsection{Causal Model of Disparate Censorship}
\label{subsec:causal_model}

We formalize disparate censorship using a causal directed acyclic graph (DAG; Figure~\ref{fig:causal_dag}). We define the problem using five variables: $a \sim P(A)$ (binary),\footnote{In practice, $a$ is often categorical, for which our analysis generalizes, but we restrict $a$ to be binary for simplicity. We leave analyses of non-categorical/overlapping $a$ as future work.} representing a patient subgroup (\textit{e.g.}, race, biological sex, hospital location), $\x \sim P(X, A)$ (continuous), representing feature vectors/covariates for an individual patient, $t \sim P(T \mid A, X)$ (binary), representing whether a patient was tested for a condition of interest, $y \sim P(Y \mid X, A)$ (binary, unobserved), representing whether a patient has the condition of interest, and $\tilde{y} \sim P(\tilde{Y} \mid Y, T)$ (binary), representing whether the patient tested positive for the condition. We assume $X, Y$ may or may not depend on $A$. Throughout the paper, we notate subgroup-level versions of various distributions as $P_a(\cdot)$ for $a \in \{0, 1\}$. For example, $P_0(\x)$ denotes ``the distribution of covariates for patients in group $a=0$". %

The DAG (Figure~\ref{fig:causal_dag}) encodes two additional assumptions. First, we assume clinician decisions to test, \textit{i.e.}, $T$, depend on both $A$ and $X$. In other words, the level of testing disparity depends on both the values of subgroup $A$ and covariates $X$. 
Second, we assume no unobserved confounding between $X$ and $Y$, such that $X$ contains all direct causes of $Y$.
Later, we will relax our modeling assumptions by allowing either $X$ or $Y$ to causally depend on $A$ (indicated by the dashed blue arrows).\footnote{Although $A$ is an exogenous variable, we caution against general interpretations of social categories as inherent/static characteristics.} 

Suppose that we aim to train a patient risk-stratification model. We frame this as a supervised ML task, in which one aims to learn a mapping $s: \mathcal{X} \mapsto \R$, where $\mathcal{X}$ is the support of $\x$, and predicted values of $y$ are generated by thresholding $s(\x)$. In our setting, we assume that the true label, $y$, is unobserved. Instead, we observe $\tilde{y}$ (\emph{i.e.}, a test result), a (potentially) noisy proxy for $y$. The model $s$ is trained on a dataset $\mathcal{D} = \{(\x_i, \tilde{y}_i)\}_{i=1}^n$ by minimizing some loss function $\mathcal{L}: (x_i, \tilde{y}_i, s) \to \R$ (\textit{e.g.}, regularized binary cross-entropy loss) that aims to make model predictions $s(\x_i)$ close to the observed labels $\tilde{y}_i$. 

\begin{figure}[t]
    \centering
    \includegraphics[width=\linewidth]{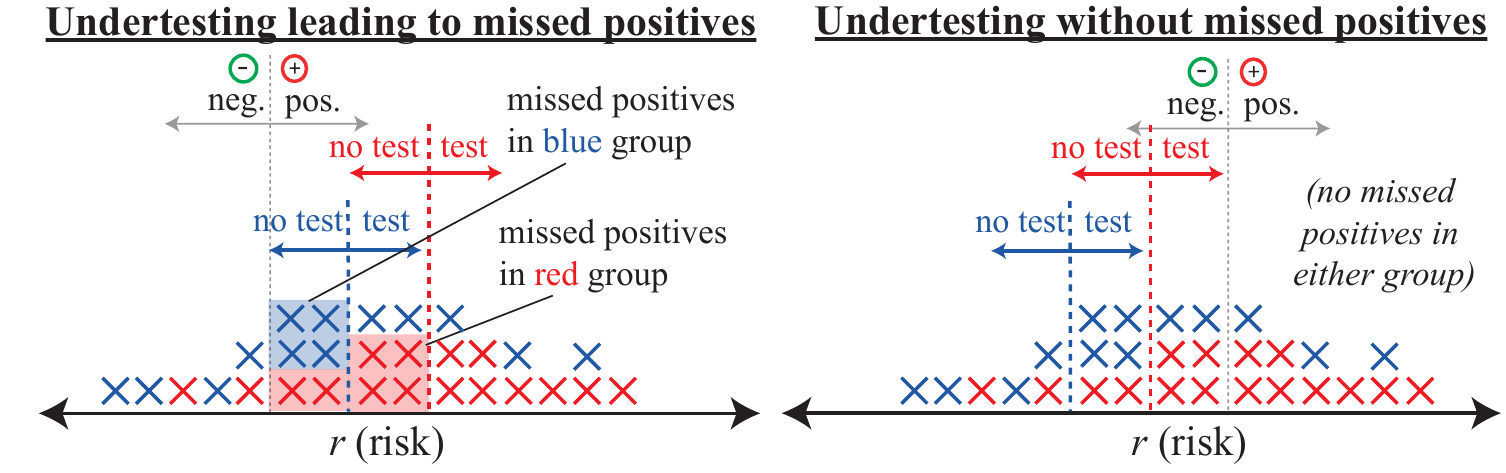}
    \caption{Stylized undertesting patterns. In both examples, the red group is undertested (higher risk $r$ required for a test). Left: undertesting is harmful due to the disproportionate rates of missed diagnoses. Right: the absence of missed diagnoses means that undertesting has no disproportionate impact in terms of missed positives. In this paper, we focus on the scenario on the left.}
    \label{fig:undertesting}
\end{figure}

We assume that the dataset contains disjoint subgroups $\mathcal{D}_a$, where $a \in \{0, 1\}$. Let $t \in \{0, 1\}$ be a variable indicating whether a patient was tested (0 = no, 1 = yes). To simplify, we assume  perfectly accurate tests.\footnote{In practice, tests may not be perfectly accurate, and may have different sensitivities in each group. Our analysis remains the same, since our argument is based on label noise rates. In such cases, we can define $T$ as a variable indicating whether a patient was tested \emph{and} the correctness of the test result, from which the results follow identically.} We define disparate censorship as follows:

\begin{definition}[Disparate censorship.]
Let $P_a(t)$ be the probability that a patient in group $a$ was tested for a condition of interest $y$. Disparate censorship occurs if $P_0(t) \neq P_1(t)$.
\end{definition}

Under disparate censorship, the true label $y$ is censored/unobserved at different rates in each patient subgroup. Consequently, if a clinician decides to not test a patient for condition $y$ (\textit{i.e.}, $t = 0$), then $y$ is censored, so $\tilde{y} = 0$. Conversely, if $t = 1$, then $\tilde{y} = y$---we observe whether the patient has condition/outcome $y$.

Similarly, undertesting can be defined as follows:

\begin{definition}[Undertesting.]
\label{def:undertesting}
Define $r$ as a random variable representing the probability of condition $y$ given covariates $\x$; i.e., $r \propto P(y \mid \x)$.
Let $P_a(t|r)$ denote the probability that a patient in group $a$ with risk $r$ received a test. Without loss of generality, we say that group $A=1$ is undertested relative to group $A=0$ if $\int_r \max(0, [P_0(t|r) - P_1(t|r)]) dr > 0$.
\end{definition}

This definition captures all \textit{positive} testing gaps between two groups across all levels of risk. %
Furthermore, this definition demonstrates how undertesting and disparate censorship differ: the absence of disparate censorship does not guarantee the absence of undertesting. However, undertesting can arise due to disparate censorship. 
In our work, we focus on undertesting such that one group suffers disproportionately high rates of missed diagnoses/positives compared to the other group (Figure~\ref{fig:undertesting}, left). In this stylized example, undertesting as defined raises fewer concerns if it occurs in patients without the condition of interest, since the missed positive rate is already zero (Figure~\ref{fig:undertesting}, right). %

\begin{figure}[t]
\centering
    \begin{minipage}[t]{0.5\textwidth}
        \centering
        \includegraphics[width=0.85\linewidth]{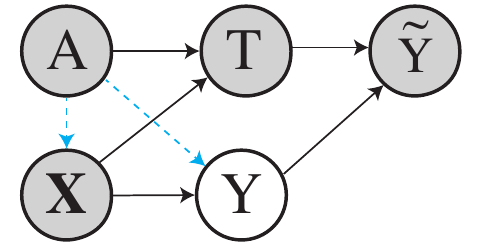}
        \captionsetup{format=plain}
        \caption{Causal DAG used in our analysis. (Un)shaded variables are (un)observed. Patient risk-stratification models aim to predict $y$ given $x$, but only $\tilde{y}$ is observed. Disparate censorship arises when $a$ affects $t$. Also, $a$ may affect $\x$ and $y$ (dashed arrows).}
        \label{fig:causal_dag}
    \end{minipage}\hfill%
\begin{minipage}[t]{0.48\textwidth}
        \centering
        \includegraphics[width=0.6\linewidth]{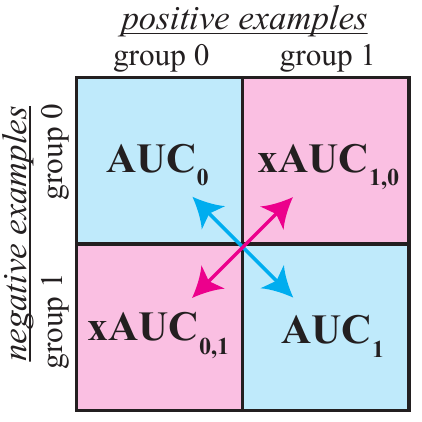}
        \captionsetup{format=plain}
        \caption{Decomposition of overall AUC into within-group AUCs (cyan) and xAUCs (magenta). Zero ranking performance gap between groups $a$ requires parity for both within-group AUCs (cyan arrow) and xAUCs (magenta arrow).}
        \label{fig:auc}
\end{minipage}
\end{figure}

\subsection{Data Generating Processes}
\label{subsec:distributional_diff}

We analyze undertesting and disparate censorship in three settings defined by different assumptions for the data generating processes. %
Within each setting, we assume the same distributional conditions hold at training and inference time.

\paragraph{Setting 1: No difference in marginal and conditional distributions.} Formally, $P_0(\x) \disteq P_1(\x), P_0(y \mid \x) \disteq P_1(y \mid \x)$.\footnote{$\disteq$: Equal in distribution.} This means that $A$ has no effect on $X$ or $Y$. In this setting, both population subgroups have the same distribution of features/covariates, and the same probability of disease given specific covariate values. Note that $A$ is independent of $X$: not only is $A$ not included in $X$, but no component of $X$ is affected by $A$. This is an idealized setting in which there are no differences between subgroups in (1) the distribution of covariates and (2) their relationship to the outcome of interest. However, this is unlikely in practice: even when $A$ is not used as a model input, proxies for $A$ can appear in $X$ such that $P_0(\x) \distneq P_1(\x)$, as discussed in~\cite{vyas2020hidden, ioannidis2021recalibrating}.\footnote{$\distneq$: not equal in distribution.} We provide examples of such proxies in the next setting.

\paragraph{Setting 2: Difference in the marginal distribution only.} In this setting, $P_0(y \mid \x) \disteq P_1(y \mid \x)$, but $P_0(\x) \distneq P_1(\x)$. This means that patients with the same covariates are at equal risk for the disease, but the covariate distribution may vary across subgroups $a$. In causal terms, variable $A$ is a cause of/associated with components of $X$; the distribution of the covariates of interest varies between patient subgroups. Examples of this setting arise in the study of social determinants of health~\citep{singh2017social}, disparities in hypertension rates~\citep{lackland2014racial}, or multicenter datasets~\citep{bhuva2019multicenter}.
In addition, differences in healthcare utilization (alluded to in~\cite{price2020hospitalization}) and/or consent~\citep{spector2021respecting} may also give rise to covariate differences.

\paragraph{Setting 3: Difference in the conditional distribution only.} Formally, $P_0(\x) \disteq P_1(\x)$, but $P_0(y \mid \x) \distneq P_1(y \mid \x)$. In causal terms, this means that variable $A$ is a cause of/associated with $Y$. In this setting, we make a simplifying assumption that covariate distributions are identical across $a$. However, patients with the same covariates that belong to different subgroups may have different probabilities of disease. For example, female patients with acute myocardial infarction (heart attack) may present differently than male patients~\citep{lichtman2018sex}.%

If both the marginal risk distribution (Setting 2) and the conditional risk distribution (Setting 3) differ (\textit{i.e.}, $P_0(\x) \distneq P_1(\x)$ and $P_0(y \mid \x) \distneq P_1(y \mid \x)$, then negative impacts from either setting would apply as well. 

\section{Theoretical Analysis of Disparate Censorship \& Undertesting}
\label{sec:theory}

We develop a theoretical framework for analyzing disparate censorship.
We characterize settings under which disparate censorship and undertesting are unlikely to lead to performance gaps across patient subgroups. We focus on ranking performance gaps: differences in within-group area under the receiver operating characteristic curve (AUC) and cross-AUC (xAUC, ~\cite{kallus2019fairness}).%

\subsection{Measuring Ranking Performance Gaps}
\label{subsec:metrics}

To quantify performance gaps, we focus on ranking metrics, which are frequently used for evaluating clinical risk-stratification models. In this setting, clinicians aim to identify the top-$k$ patients to treat, where $k$ may be determined by resource constraints. %
To evaluate ranking performance gaps, we report two metrics: (1) the gap in within-group AUC and (2) the gap in cross-AUC (xAUC, ~\cite{kallus2019fairness}).

As intuition, recall that the AUC is the probability that a randomly chosen positive example, $\x_i$, has a greater risk score, than a randomly chosen negative example, $\x_j$ (\textit{i.e.}, $P(s(\x_i)>s(\x_j))$). Then the within-group AUC for group $a$, written as $\text{AUC}_a$, is the probability that two patients from group $a$ (one positive, one negative) were correctly ranked (Figure~\ref{fig:auc}, cyan). The xAUC has a similar interpretation: $\text{xAUC}_{a, a'}$ is the probability that a random positive patient in group $a$ was ranked above a negative patient from group $a'$ (Figure~\ref{fig:auc}, magenta). Both within-group AUC and xAUC are key to explaining ranking performance gaps: while within-group AUC only captures misranking error between patients in one group, xAUC quantifies misranking error between groups. %

Ideally, group ranking performance is equal across groups (Figure~\ref{fig:auc}, cyan arrow), while cross-group ranking performance should be symmetric (Figure~\ref{fig:auc}, magenta arrow).
This yields the following performance gap metrics---$\Delta \text{AUC}$ and $\Delta \text{xAUC}$ (lower is better):
\begin{align}
    \Delta \text{AUC} &\triangleq \left|\text{AUC}_1 - \text{AUC}_0\right|\\
    \Delta \text{xAUC} &\triangleq \left|\text{xAUC}_{0,1} - \text{xAUC}_{1,0}\right| %
\end{align}

Assuming \emph{perfect separation}, all patient pairs can be perfectly ranked, and the optimal overall AUC is 1.
In such cases, the optimal $\text{AUC}_a, \text{xAUC}_{a, a'}$ values are also 1, so $\Delta \text{AUC} = \Delta \text{xAUC} = 0$ (no performance gap across groups). For details, see Appendix~\ref{appdx:eval}.

\subsection{Impact of distributional differences on performance gaps}
\label{subsec:diffs}

Here, we prove conditions that lead to zero $\Delta \text{AUC}, \Delta \text{xAUC}$ in the presence of disparate censorship and/or undertesting for the three distributional settings studied.

\subsubsection{No difference in marginal or conditional risk distributions}
\label{subsec:nodiff}

In this idealized case (\textit{i.e.}, Setting 1), a model would converge to zero performance gap in expectation, even when trained on data exhibiting disparate censorship. 
In fact, Setting 1 is a special case in which even undertesting, or different testing rates for patients of equal risk, would not result in performance gaps. 
This is because the two groups are indistinguishable in $X, Y$, since $A$ is independent of both $X$ and $Y$, so any risk-stratification model will have identical performance across groups.
Thus, $\Delta \text{AUC}, \Delta \text{xAUC}$ converge to zero.

\subsubsection{Difference in marginal distribution only}
\label{subsec:marginal_diff}

In Setting 2, the marginal distribution of the covariates varies across groups. We show that undertesting the low-risk group leads to zero $\Delta\text{AUC}$, $\Delta \text{xAUC}$.

First, we consider a censorship model in which clinicians apply thresholds to clinical risk estimates to make testing/treatment decisions (the ``threshold approach"  in clinical decision-making~\citep{pauker1980threshold}).
In covariate space, we call such testing thresholds ``censorship boundaries."
Suppose that clinician decisions to test ($P(T \mid X, A)$) and the true condition indicators ($P(Y \mid X$)) %
can be written as $\mathbbm{1}[s(\x) > (\cdot)]$ for some ``scoring function" $s: \mathcal{X} \to \R$.
An example is provided in Figure~\ref{fig:boundaries}: note that the censorship boundaries are ``parallel" to the true decision boundary.
Concretely, we assume that there exists some $\tau_a \in \R$ for $a \in \{0, 1\}$, which act as a group-wise censorship thresholds, such that $t = \mathbbm{1}[s(\x) > \tau_a \lor p = 1]$, where $p \sim Bernoulli(c)$ for a suitably small $c \in (0, 1]$.\footnote{\label{foot:c}The constraint $c > 0$ ensures that no one is tested with zero probability, or else there may be no signal to learn for sufficiently high censorship thresholds.}
Furthermore, we assume there exists some $b \in \R$, which acts as a decision boundary, such that $y = \mathbbm{1}[s(\x) > b]$.
That is; if a patient in group $a$ has risk $s(\x)$ greater than $b$, their true outcome $y$ is 1.
The patient is tested for the condition if $s(\x) > \tau_a$; otherwise, they are tested with probability $c$.

We show that, in Setting 2, when training a model using $\x$, when undertesting results in missed positives, either undertesting the low-risk group or the absence of undertesting (given that $a=0$ is the low-risk group, $\tau_0 \geq \tau_1$) yields zero performance gap.
When undertesting does not result in any missed positives, no performance gap is expected. 

\begin{figure}
    \centering
    \includegraphics[width=0.7\linewidth]{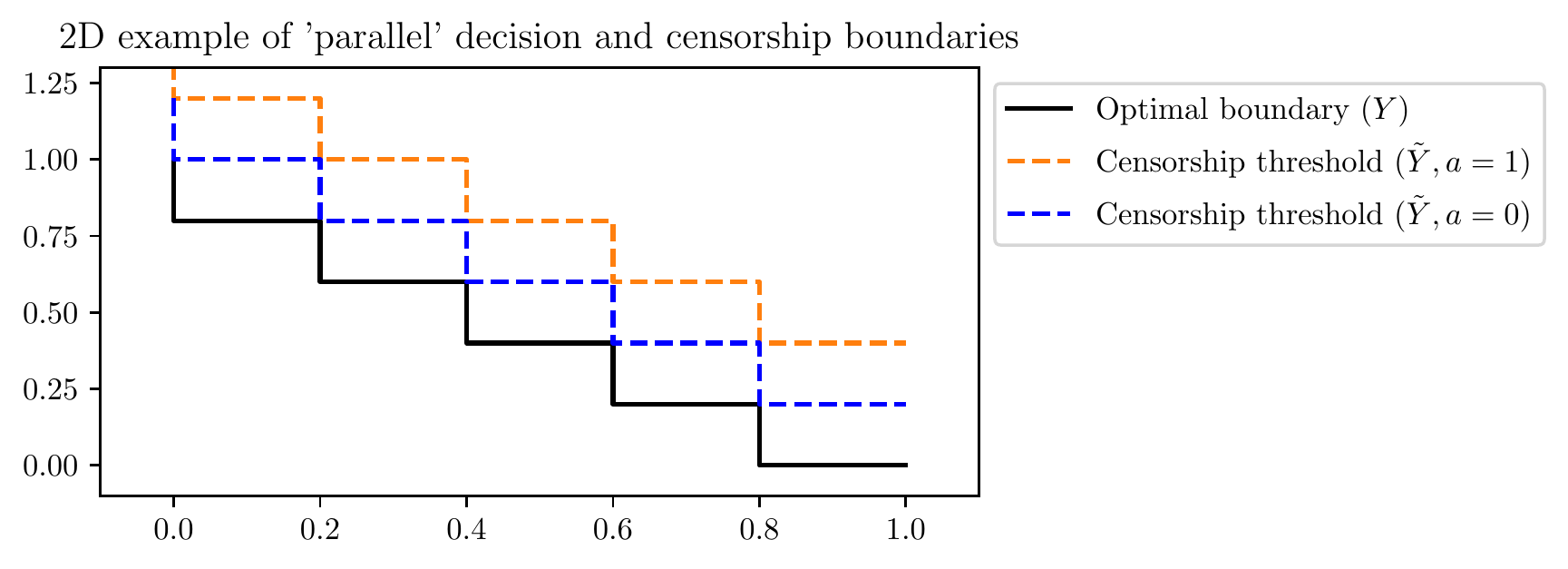}
    \caption{Example decision (black) and censorship boundaries (blue, orange) in 2D. In our noise model, we assume that the clinician ``test decision" ($t$) and true ``condition status" ($y$) decision boundaries are based on a threshold with the same functional form (\textit{i.e.} ``parallel").}
    \label{fig:boundaries}
\end{figure}

\begin{theorem}[Zero performance gap under different marginal distributions.]
\label{thrm:marginal}
Suppose that the causal graph in Figure~\ref{fig:causal_dag} is a correctly-specified structural model, $A$ causally affects $X$, and $A$ is not associated with $Y$. Assume that $Y$ is perfectly separable in $\x$, and $z \triangleq s(\x)$ is distributed for each group as $z \mid a=0 \sim \mathcal{N}(\mu_0, \sigma^2), z \mid a=1 \sim \mathcal{N}(\mu_1, \sigma^2)$.
Without loss of generality, suppose that $\mu_1 \geq \mu_0$. Then if $\tau_1 \leq \tau_0$ or no positives are missed, the $\Delta \text{AUC}, \Delta \text{xAUC}$ of a model using $\x$ as features converges to 0 for suitable values of $\tau_1$.\footnote{In practice, there is a broad range of suitable values for $\tau_1$; see Appendix~\ref{appdx:theory} for details.}
\end{theorem}

This result depends on $\tau_0, \tau_1$ directly instead of the absolute difference in testing rates. 
First, when undertesting does not result in missed positives, disparate censorship and undertesting may not lead to performance gaps.
Furthermore, when patients in the low-risk group $a=0$ undertested with respect to $a=1$ (or no undertesting is present), the performance gap converges to zero.
This is a special case of the ``boundary-consistent noise model"~\citep{menon2018learning}.
As applied to our setting, this noise model requires that the probability of a missed positive (testing) only decreases (increases) as the risk of condition $y$ increases.

\subsubsection{Difference in conditional distribution only}
\label{subsec:conditional_diff}

In Setting 3, the probability of condition $y$ differs across patient groups given features $\x$. We show that, for the case of linear group-wise decision boundaries and censorship boundaries, disparate censorship does not result in a performance gap if the decision boundary for each group is parallel to the censorship boundary (\textit{i.e.,} the two differ by a scalar offset).%

\begin{theorem}[Zero performance gap under different conditionals.]
\label{thrm:conditional}
Suppose that the causal graph in Figure~\ref{fig:causal_dag} is a correctly-specified structural model, and $A$ causally affects $Y$. Assume that $Y$ is perfectly separable in $(\x, a)$. For each $a$: suppose that censorship/testing decisions for group $a$ are expressible as $t = \mathbbm{1}[\bm{\theta}^\top \x + \beta > 0 \lor p = 1]$ for some $\bm{\theta} \in \R^d, \beta \in \R$, $p \sim Bernoulli(c)$, and all $\x \in \mathcal{X}$. Furthermore, assume that $s_a: \mathcal{X} \times \{0, 1\} \to \R$ has functional form $\bm{\theta}_a^\top \x + b_a$ for each group $a$, where $y = \mathbbm{1}[s_a(\x) > 0]$ for $\bm{\theta}_a \in \R^d, b_a \in \R$. If there exists $\delta \in \R, \delta > 0$ such that $\bm{\theta}_a = \delta \bm{\theta}$, the $\Delta \text{AUC}, \Delta \text{xAUC}$ of a model with features $(\x, a)$ converges to 0.
\end{theorem}

This result states that a model using both $\x$ \emph{and} $a$ as features may achieve a zero performance gap if, within each group, the corresponding decision and censorship boundaries are ``parallel" with one other.
We refer to this as the ``parallel boundaries" assumption.
For intuition, we reason about each group separately. 
We proceed as if the censorship boundary lies ``above" the decision boundary, such that some positives are censored, or else there are no missed positives and the result is immediate. A formal proof can be found in Appendix~\ref{appdx:theory}.

The ``parallel boundaries" assumption is a special case in which undertesting leading to disproportionate missed positives in one group does not affect ranking. 
This is because any censorship threshold parallel to the decision boundary preserves relative ordering of patient risk within each group, because higher-risk positives are always less likely to be missed. 
Thus, if the ``parallel boundaries" assumption holds, the resulting pattern of missed positives does not affect ranking.
We can generalize this argument to non-linear decision boundaries if the decision and testing boundary parameters are expressible in an appropriate reproducing kernel Hilbert space.
However, the ``parallel boundaries" assumption is restrictive. We later explore violations of this assumption in our simulation study.

\vspace{12pt}

Our theoretical results show that disparate censorship and undertesting do not always lead to performance gaps. If there are no differences across groups in the marginal or conditional distribution of covariates (Setting 1), no gap is expected, even if undertesting is present. When the marginal distribution of covariates differs by group (Setting 2), zero gap is possible if the lower-risk group is undertested (\textit{i.e.}, for $\mu_1 \geq \mu_0$, we must have $\tau_0 \geq \tau_1$). Lastly, when the conditional distribution of covariates differs by group (Setting 3), zero gap is possible if within each group, the censorship and true condition boundaries are ``parallel," differing only in an offset term.
However, these conditions are restrictive.
In the next section, we empirically demonstrate how performance gaps emerge across distributional settings when differences in the marginal or conditional distributions across groups emerge.

\section{Empirical Analysis of Disparate Censorship \& Undertesting}

We empirically investigate the impacts of disparate censorship via a simulation study. In this section, we describe the data generation processes and ML modeling details. Then, we present our findings on ranking performance gaps under disparate censorship. In summary, when undertesting leads to missed positives, performance gaps arise when the higher-risk group is undertested (Setting 2) or if testing standards vs. the condition decision boundary increasingly violate the ``parallel boundaries" assumption (Setting 3).  %

\subsection{Simulation data generating process}

We simulate a setting in which zero performance gap is theoretically possible when all patients are tested (\emph{i.e.}, perfect separation). %
Following the theoretical setup, we consider the case of binary $a \in \{0, 1\}$ (without loss of generality). We generate an equal number of ``patients" with $a = 0$ and $a = 1$, and simulate 10 covariates $\x$ for each patient by randomly sampling a multivariate Gaussian. We generate true labels $y$ and testing decisions $t$ by applying a non-linear decision boundary $s$ to covariates $\x$ (see Figure~\ref{fig:example} for a 2D example), where $s$ may depend on $\x$ and $a$. Concretely, $y=1$ if $s(\x) \geq 5$, and $y=0$ otherwise. Similarly, we use one parameter $\tau_a$ per group to determine testing decisions: $t=1$ if $s(\x) \geq \tau_a$, and $t=0$ otherwise with probability $1-c$ for some small $c > 0$.\footnote{See Footnote~\ref{foot:c}; $c$ ensures that this problem is learnable.}
Note that values of $\tau_a < 5$ have no effect on censorship, since all $\x$ such that $s(\x) < 5$ are negative by definition and cannot be censored. 
We then generate observed labels $\tilde{y}$ using $y$ and $t$, flipping $y$ from 1 to 0 if $t = 0$. 
Full simulation details are provided in Appendix~\ref{appdx:sim}.  %

By design, this problem is perfectly separable, so an overall AUC of 1 is feasible. Thus, when the conditions of Theorem~\ref{thrm:marginal} or Theorem~\ref{thrm:conditional} are met, in Settings 2 and 3, respectively, $\Delta \text{AUC}, \Delta \text{xAUC}$ converge to zero in expectation.

\begin{figure}[t]
    \centering
    \includegraphics[width=\linewidth]{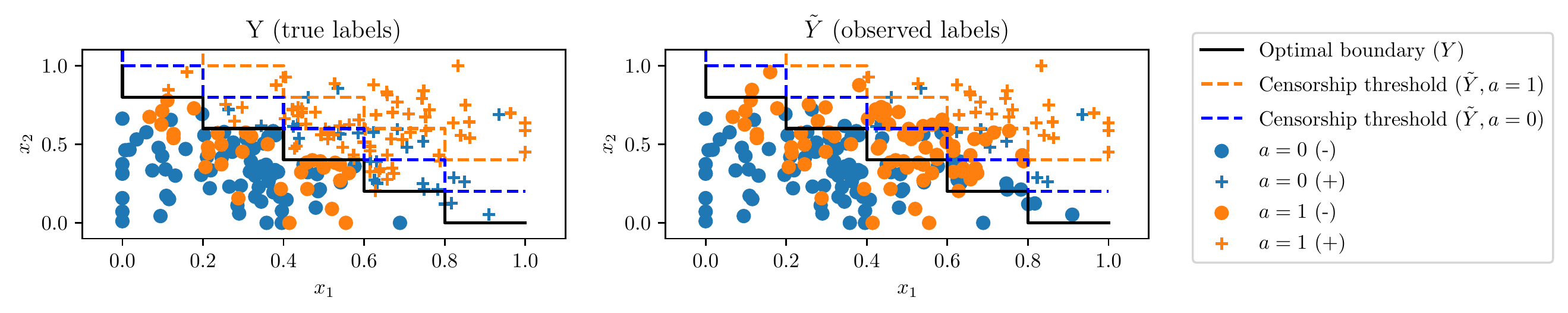}
    \caption{Example of disparate censorship in 2D with groups $a = 0$ (in blue) and $a = 1$ (in orange), and decision boundary in black. All blue positives ($a=0$, \textbf{\textcolor{blue}{+}}) below the blue dashed line are censored; likewise for orange positives ($a=1$, \textbf{\textcolor{orange}{+}}).}
    \label{fig:example}
\end{figure}

\paragraph{Simulating testing disparities.} Following the theory, we represent decision thresholds for testing each group as $\tau_0, \tau_1$ to induce undertesting. Intuitively, $\tau_0, \tau_1$ represent different clinical standards for testing. Under the threshold model, in Setting 2, the level of undertesting (as in Definition~\ref{def:undertesting}) in group $A=1$ relative to group $A=0$ is $(1-c)(\tau_1 - \tau_0)$.
In Setting 3, the level of undertesting depends on both $\tau_1-\tau_0$ and the similarity between $s_0, s_1$, so a general form for undertesting does not exist.

\paragraph{Simulating distributional differences.} To simulate $P_0(\x) \distneq P_1(\x)$ (Setting 2), we generate covariates $\x$ from multivariate Gaussians with different means, but identical covariance matrices. For $P_0(y \mid \x) \distneq P_1(y \mid \x)$ (Setting 3), we rotate the decision boundary by $\phi$ degrees, where $\phi \in [0, 360)$.

\subsection{Model Setup for Evaluating Impact of Disparate Censorship}

Throughout our experiments, we consider two probabilistic models. The first is trained on simulated true condition labels $y$, which serves as an upper bound on performance. The second is trained on observed condition labels $\tilde{y}$. This model represents the realistic setting in which test results are used as labels for model training. %
All models are evaluated with respect to the simulated condition labels $y$. 
Since our focus is on model performance gaps rather than the impact of model specification, we describe modeling details in Appendix~\ref{appdx:models}.

\section{Experiments \& Results}
\label{sec:results}

We examine the impacts of disparate censorship and undertesting under certain distributional differences between groups. We investigate:
\begin{itemize}
    \item How do disparate censorship and undertesting impact model performance gaps when there are differences in the marginal distribution only? (Setting 2, Section~\ref{subsec:marginal})
    \item How do disparate censorship and undertesting impact model performance gaps when there are differences in the conditional distribution only? (Setting 3, Section~\ref{subsec:conditional})
\end{itemize}

Setting 1 (no marginal or conditional distributional difference) is a special case of Setting 2 and 3 where the levels of marginal or conditional distributional difference are both zero.

\subsection{Undertesting may lead to performance gaps when marginal distributions of covariates differ}
\label{subsec:marginal}

In this subsection, we assess the impact of disparate censorship on model performance gaps across groups when the marginal distributions differ across groups. In line with our theoretical results, we find that performance gaps between groups arise when the high-risk group is increasingly undertested.%

\begin{figure}[t]
    \centering
    \includegraphics[width=0.85\linewidth]{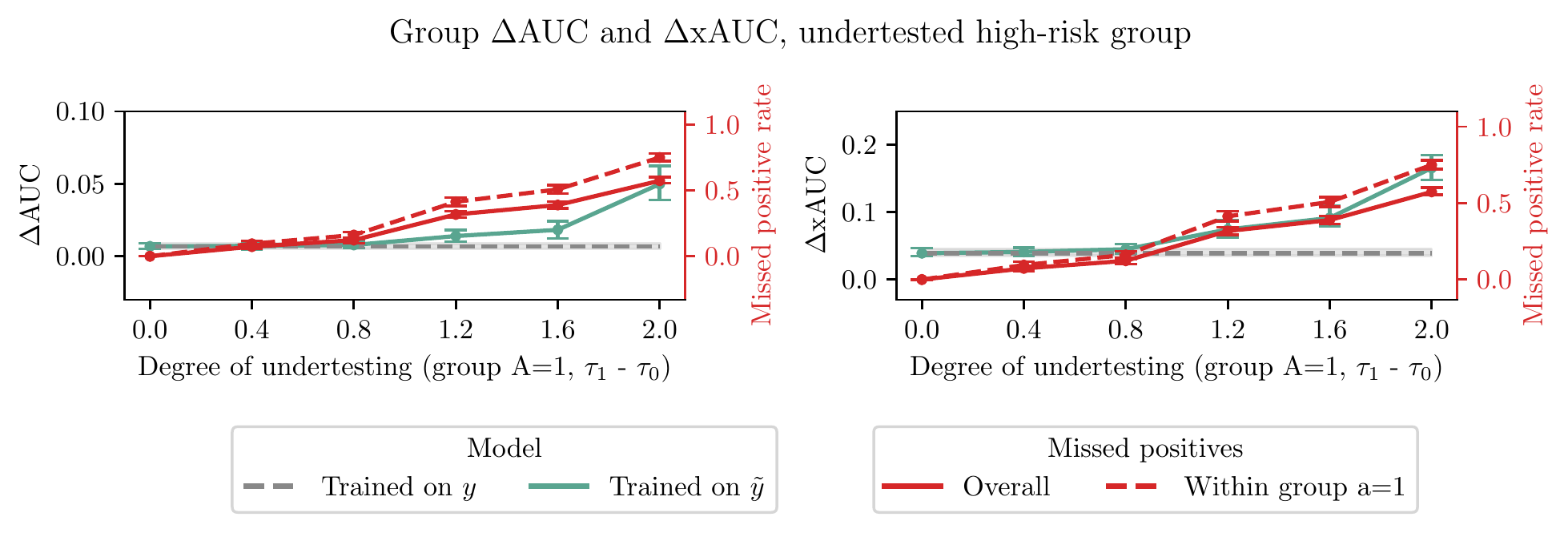}
    \caption{$\Delta \text{AUC}$ (left) and $\Delta \text{xAUC}$ (right) with 95\% empirical CIs for the model trained on $y$ (gray) and the model trained on $\tilde{y}$ (green).  As the degree of undertesting experienced by the high-risk group $a=1$ ($\tau_1 - \tau_0$) increases (dashed red), with a sharply increasing amount of overall missed positives (solid red), $\Delta \text{AUC}, \Delta \text{xAUC}$ rises.} %
    \label{fig:highrisk_undertested}
\end{figure}

\paragraph{Undertesting the high-risk patient group results in large model performance gaps.} %
In this experiment, we define group $a = 1$ to be the ``high-risk" group. Since censorship occurs at values of $\tau_0, \tau_1$ greater than or equal to 5, we set $\tau_0 = 5$ and then vary the amount of undertesting for group $a = 1$ by choosing $\tau_1 \in \{5, 5.4, 5.8, 6.2, 6.6, 7\}$. The degree of undertesting experienced by the high-risk group $a=1$ corresponds to $\tau_1 - \tau_0$.
In this setting, patients in group $a=0$ are fully tested (no missed positives), while patients in group $a=1$ are increasingly undertested, leading to more missed positives in group $a=1$.

At $\tau_0 = 5, \tau_1 = 5$, there is no undertesting; furthermore, there are no missed positives in either group, so the performance gap is near zero, and matches the performance of a model trained using true labels $y$. 
In practice, convergence to $\Delta\text{AUC}, \Delta \text{xAUC} = 0$ is data intensive in high dimensions due to the prevalence of sparsely-sampled regions near the decision boundary (\emph{i.e.}, the curse of dimensionality). Thus, a small performance gap \textit{independent of the degree of undertesting} is expected, as demonstrated by the constant performance gap for the model trained on $y$.
However, as $\tau_1$ increases, both the rate of missed positives in group $a=1$ and overall missed positive rate rise sharply. %
Consistent with Theorem~\ref{thrm:marginal}, $\Delta \text{AUC}, \Delta \text{xAUC}$ increase when the high-risk group is undertested such that missed positives result (Figure~\ref{fig:highrisk_undertested}).%
We also examine increasing $\tau_0$ from 5, inducing missed positives in both groups; the correlation between the performance gap and the value of $\tau_1 - \tau_0$ still holds (Appendix~\ref{appdx:all_results}).

\begin{figure}[t]
    \centering
    \includegraphics[width=0.85\linewidth]{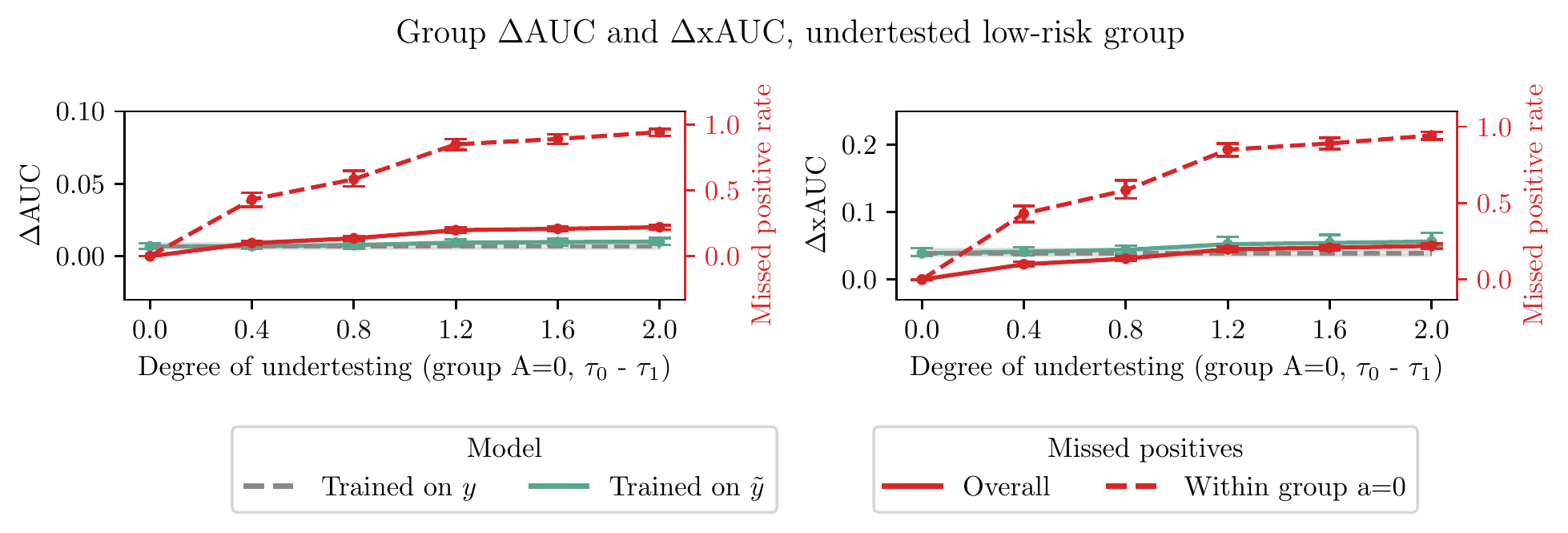}
    \caption{
    $\Delta \text{AUC}$ (left) and $\Delta \text{xAUC}$ (right) with 95\% empirical CIs for the model trained on $y$ (gray) and the model trained on $\tilde{y}$ (green). Even as the degree of undertesting experienced by the low-risk group $a=0$ ($\tau_0- \tau_1$) increases, such that the missed positive rate approaches 1 (dashed red), all metrics stay near oracle performance while the overall missed positive rate plateaus (red).}%
    \label{fig:lowrisk_undertested}
\end{figure}

\paragraph{Undertesting the low-risk group results in small performance gaps.} Reversing the direction of undertesting changes the impact of disparate censorship. We simulate undertesting in group $a = 0$ by setting $\tau_1 = 5$, and selecting $\tau_0 \in \{5, 5.4, 5.8, 6.2, 6.6, 7\}$. 
The level of undertesting experienced by group $a=0$ corresponds to $\tau_0 - \tau_1$ (reversed with respect to the previous experiment).
In this setting, group $a=0$ is increasingly undertested, resulting in more missed positives, while group $a=1$ is fully tested (no missed positives). %
As $\mu_1 < \mu_0$, and $\tau_1 \leq \tau_0$, Theorem~\ref{thrm:marginal} suggests that the performance gap converges to 0, regardless of the level of disparate censorship---even as the missed positive rate grows. 

In contrast to the previous experiment, the number of overall positives in the low-risk group is lower relative to that of the high-risk group, limiting the overall missed positive rate.
Thus, even when the low-risk group is severely undertested (high missed positive rate), training with labels $\tilde{y}$ versus $y$ does not significantly affect $\Delta \text{AUC}$ and $\Delta \text{xAUC}$.
Indeed, Figure~\ref{fig:lowrisk_undertested} shows that the AUC gap is at most 0.01, while the xAUC ranges from 0.04 to 0.06.

\vspace{12pt}
In summary, when the marginal distributions of risk vary by group, the harm (in terms of $\Delta \text{AUC}, \Delta \text{xAUC}$) from undertesting the high-risk group is greater than the harm from undertesting the low-risk group.
Intuitively, this is because the high-risk group comprises the majority of positive patients, such that undertesting them yields more missed positives.
On the other hand, the lower number of positive patients in the low-risk group limits the potential for undertesting to cause missed positives, yielding a smaller performance gap.
This result highlights a need to understand whether testing disparities disproportionately impact higher- or lower-risk patient groups.
We also examined increasing the distributional distance between groups (with $\tau_0, \tau_1$ constant); while increasing distribution distance also widens observed performance gaps, overall trends are similar (Appendix~\ref{appdx:all_results}).

\begin{figure}[t]
    \centering
    \includegraphics[width=0.75\linewidth]{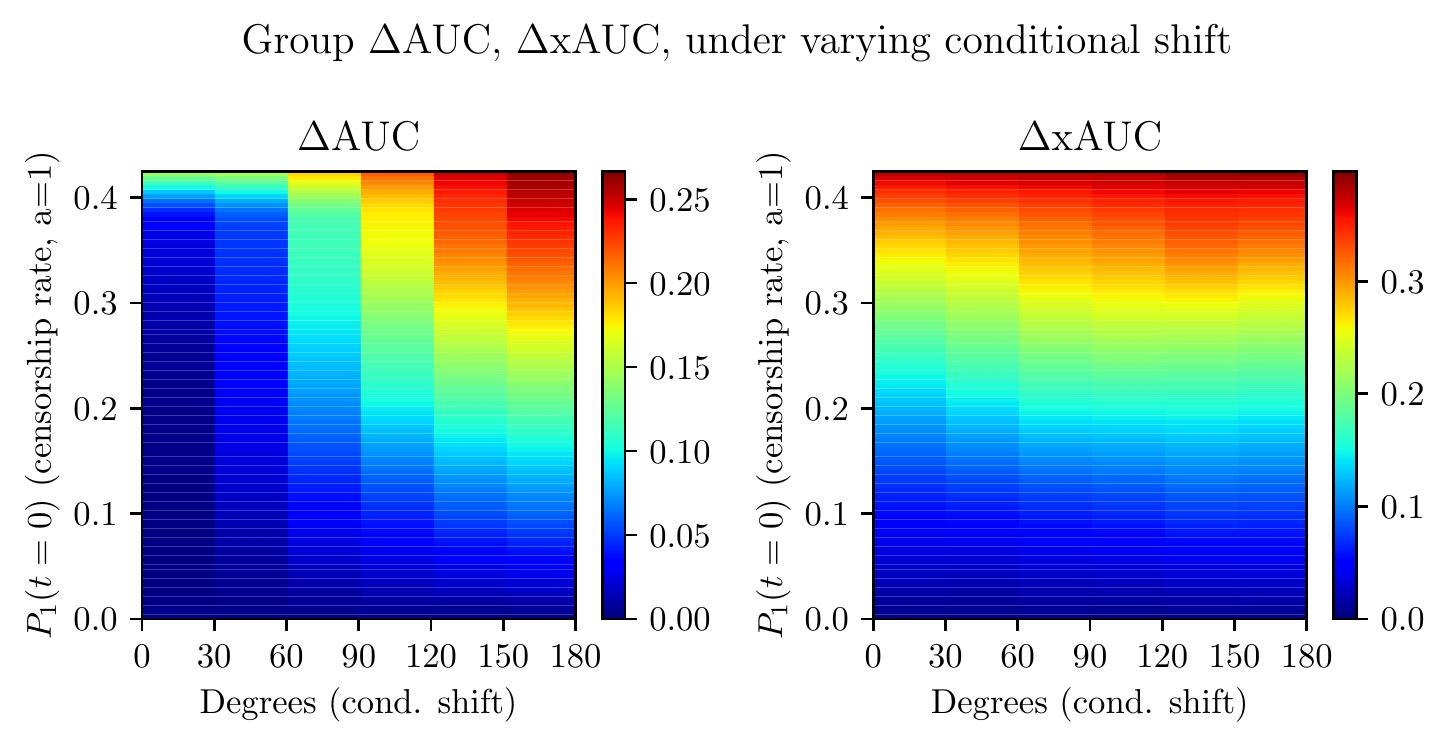}
    \caption{Heatmap showing median $\Delta \text{AUC}$ (left) and $\Delta \text{xAUC}$ (right) at varying levels of conditional shift ($\phi$; $x$-axis) and censorship rate in $a=1$ ($P_1(t=0)$, $y$-axis); 4 dimensions rotated. Regions with smaller performance gap are in \textcolor{jet_min}{dark blue}, while larger  gaps are in  \textcolor{jet_max}{dark red}. Performance gaps widen as conditional shift or disparate censorship intensify.}
    \label{fig:cond_shift}
\end{figure}
\subsection{Differences in conditional risk distribution also lead to performance gaps under clinician bias}
\label{subsec:conditional}

To simulate Setting 3, we vary the level of rotation $\phi$ from $\{0, 30, 60, 90, 120, 150, 180\}$ applied to $d'$ dimensions of the decision boundary (see Table~\ref{tab:settings} for details), and report performance gaps at $\tau_0 = 5$ and varying $\tau_1 \in \{5, 5.4, 5.8, 6.2, 6.6, 7\}$ to induce testing disparities. %

We show results for $d' = 4$, but trends are similar for other $d' \in \{2, 4, 6, 8, 10\}$ (Appendix~\ref{appdx:all_results}). %
We plot median $\Delta \text{AUC}, \Delta \text{xAUC}$ in terms of testing disparity $P_0(t=1) - P_1(t=1)$ (equal to $P_1(t=0)$, since $\tau_0 = 5$), and the level of conditional shift $\phi$ as a heatmap. As $d' > 0, \phi \neq 0$ violates Theorem~\ref{thrm:conditional} assumptions, performance gaps could emerge.

Consistent with Theorem~\ref{thrm:conditional}, increasing $\phi$ leads to larger $\Delta \text{AUC}, \Delta \text{xAUC}$ (Figure~\ref{fig:cond_shift}).
As the decision boundary rotates further, more true positives at varying levels of risk are rotated beneath the censorship boundary---increasing the number of missed positives, and $\Delta \text{AUC}, \Delta \text{xAUC}$ increase. 
Thus, even with no marginal distributional differences between groups, standard ML model training using $\tilde{y}$ may result in significant disparities in model performance when conditional distributional differences are present. This result highlights the importance of recognizing disparities in conditional risk distributions, \emph{i.e.}, when the mechanism of condition $y$ varies by group.

\begin{figure}
    \centering
    \includegraphics[width=\linewidth]{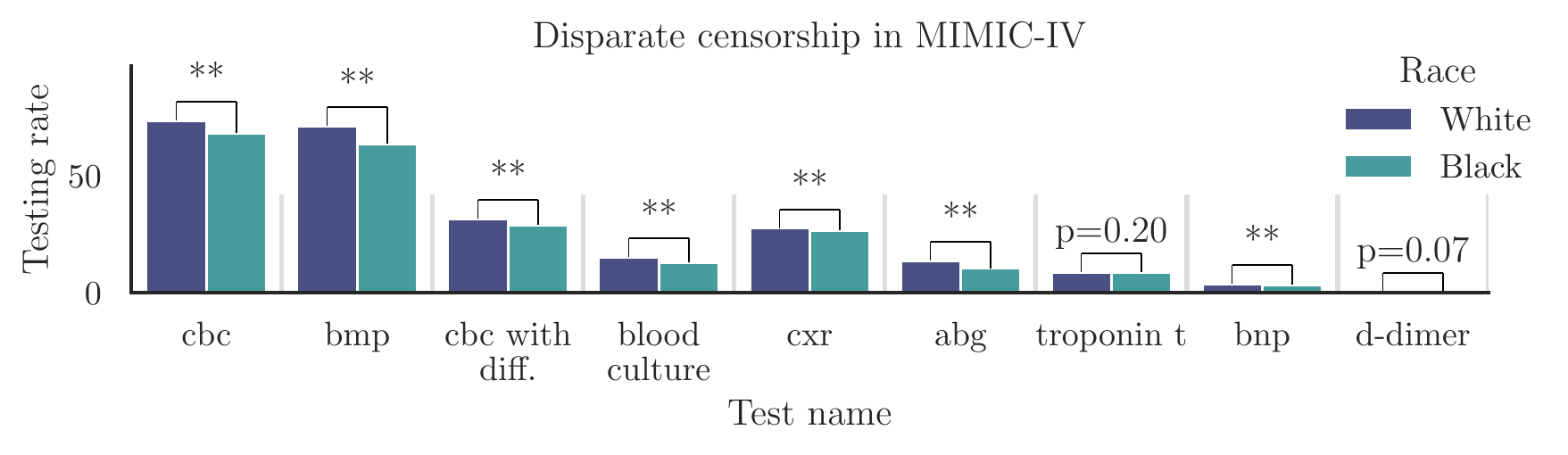}
    \caption{Disparate censorship in terms of the testing rate disparity ($P(T=1 \mid \text{admission for White patient}) - P(T=0 \mid \text{admission for Black patient})$) in common laboratory tests in MIMIC-IV. Testing rates plotted by race (White = \textcolor{mimic_w}{indigo}, Black = \textcolor{mimic_b}{teal}). ``**" denotes a statistically significant difference ($\alpha = 1.1 \times 10^{-3}$ post-Bonferroni); $p$-value noted otherwise.}
    \label{fig:mimic_dc}
\end{figure}

\section{Practical Concerns \& Guidance for Addressing Disparate Censorship}

Our empirical results demonstrate that, in some settings, disparate censorship may lead to performance gaps. Here, we show the extent to which disparate censorship occurs in common laboratory/diagnostic tests in MIMIC-IV, a popular dataset used in ML for healthcare, and suggest ways to address disparate censorship. 

\paragraph{Disparate censorship in the MIMIC-IV dataset.} We validate the existence of disparate censorship in laboratory/diagnostic tests in the MIMIC-IV dataset~\citep{johnson2020mimic} (version 1.0), a widely-used clinical dataset consisting of electronic health record data from hospital admissions to emergency/intensive care units at the Beth Israel Deaconess Medical Center (BIDMC). We limit our analysis to ``White" and ``Black/African-American" (category names from MIMIC-IV) patients, since they constitute the two most prevalent racial groups in the dataset. This yields a sample of 337630 admissions corresponding to White patients and 80293 admissions corresponding to ``Black/African-American" patients (referred to here as ``Black" patients). 

We test for disparate censorship via two-sample $z$-tests (1\% significance threshold with Bonferroni correction; $\alpha=1.1 \times 10^{-3}$) in standard laboratory/diagnostic tests such as complete blood counts (CBC, with and without differential), base metabolic panels (BMP), blood cultures, chest X-ray orders (CXR), arterial blood gas tests (ABG), troponin T tests, brain natriuretic peptide (BNP) tests, and d-dimer tests.
These tests are chosen since they may be used to help diagnose certain conditions (\textit{e.g.}, anemia or infection from CBCs, kidney injury from BMPs) or feature directly in clinical definitions (\textit{e.g.}, blood culture orders and sepsis), which may impact label definitions for training downstream ML models. 
Figure~\ref{fig:mimic_dc} shows that statistically significant disparate censorship occurs in CBCs (with and without diff.), BMPs, blood cultures, CXRs, ABGs, and BNP tests.

Importantly, Black patients are significantly less likely to be tested in instances of disparate censorship. Even though these results do not definitively prove undertesting, the consistently lower testing rates in Black patients raise concerns about whether ML models trained on such data could encode such testing disparities, leading to the negative impacts we highlight in our study. This is particularly concerning due to the wide usage of MIMIC-IV in clinical ML. More detailed results can be found in Appendix~\ref{appdx:all_results}.

\paragraph{Addressing disparate censorship and undertesting.}  Recall that, if the marginal risk distribution differs (Condition 1, Figure~\ref{fig:tree}) and the high-risk group is undertested leading to missed positives (Condition 2, Figure~\ref{fig:tree}), performance gaps may emerge. If the conditional risk distribution differs instead (Condition 3, Figure~\ref{fig:tree}) and the decision and censorship boundaries are non-parallel (Condition 4, Figure~\ref{fig:tree}), model performance gaps may also emerge.
Identifying settings in which disparate censorship and undertesting can have adverse effects could inform interventions in health policy, clinical care delivery, or even computational solutions. Here, we suggest potential methods for detecting and mitigating disparate censorship and undertesting.

Via our causal model (Figure~\ref{fig:causal_dag}), in the presence of disparate censorship (\textit{i.e.}, $A \rightarrow T$), removing the dependence between $A$ and $X$ in Setting 2 (\emph{i.e.}, $P_0(\x) \distneq P_1(\x)$) and/or the dependence between $A$ and $Y$ in Setting 3 (\emph{i.e.}, $P_0(y \mid \x) \distneq P_1(y \mid \x)$) would transform instances of Settings 2/3 into instances of Setting 1, making the patient groups $A$ indistinguishable. 
These approaches could potentially mitigate performance gaps, since a model trained on such data would behave identically across groups. 
Identifying whether differences in the marginal and conditional risk distributions (Settings 2 and 3, respectively) can provide further information for model design decisions.

To check whether $P_0(\x) \disteq P_1(\x)$ (Condition 1 of Figure~\ref{fig:tree}/Setting 2), standard hypothesis tests for distributional equality such as Kolmogorov-Smirnov \citep{massey1951kolmogorov} can be used for each covariate. Non-parametric distributional distances (\textit{e.g.}, 2-Wasserstein, MMD~\citep{gretton2012kernel}) could help quantify to what extent $P_0(\x) \disteq P_1(\x)$ holds. 
For distinguishing the high-risk group in particular, beyond hypothesis testing, incorporating domain knowledge on health disparities in the relevant covariates may be necessary.
To determine if the high-risk group is undertested (and therefore, performance gaps may arise), one may estimate $\tau_a$ (Condition 2, Figure~\ref{fig:tree}) via threshold tests~\citep{simoiu2017problem, pierson2018fast, patel2021learning}.
For mitigation, the condition $P_0(\x) \disteq P_1(\x)$ is reminiscent of covariate shift in domain adaptation. Hence, standard approaches (\textit{e.g.}, reweighing methods~\citep{jiang2008literature}, optimal transport~\citep{courty2017joint}, feature augmentation~\citep{daume2009frustratingly}) may apply.

Checking whether $P_0(y \mid \x) \disteq P_1(y \mid \x)$ holds (Condition 3 of Figure~\ref{fig:tree}/Setting 3) is less straightforward, as we observe $y$ with potentially varying noise rates in each group.
Metrics of conditional distributional similarity (\textit{i.e.}, Bregman correntropy~\citep{yu2020measuring}) have been proposed, but determining the presence of undertesting remains an open question for Setting 3, as the covariates most predictive of $y$ may vary by group. %
Positive-unlabeled (PU) learning is a promising direction here, but instance-dependent/group-wise PU learning remains underexplored. We highlight~\cite{gong2021instance} as a potential approach.
Lastly, resolving the ``parallel boundaries" assumption (Condition 4, Figure~\ref{fig:tree}) requires modeling both boundaries, so the same limitations for checking Condition 3 apply.
Using domain knowledge may be most practical for verifying Conditions 3 and 4.

\section{Discussion} 
\label{sec:discussion}

\begin{figure}[t]
    \centering
    \includegraphics[width=0.85\linewidth]{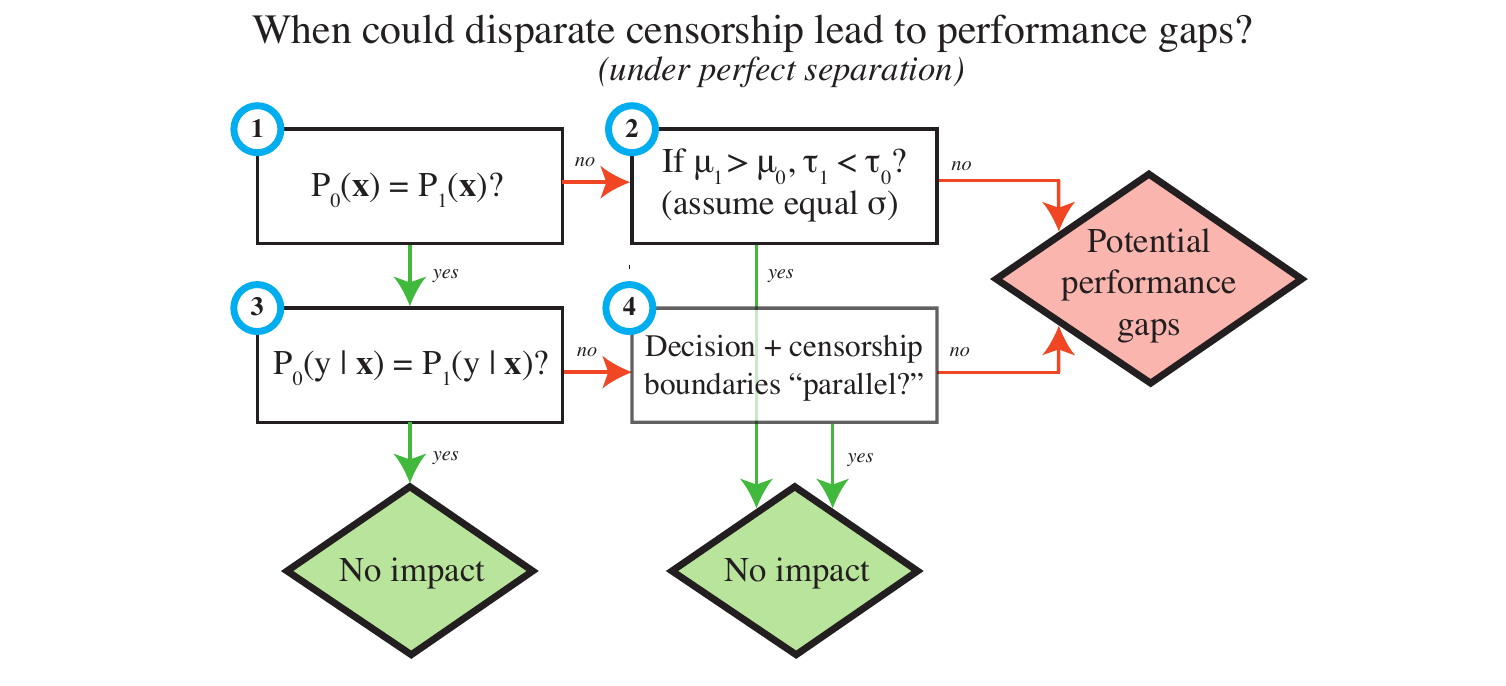}
    \caption{Decision tree for identifying when disparate censorship may negatively impact model performance gaps (assuming perfect separation). In summary, if either the marginal risk distribution differs (1) and the high-risk group is undertested (2), or the conditional risk distribution differs (3) and the decision and censorship boundaries are non-parallel (4), performance gaps may emerge ($\Delta \text{AUC}, \Delta \text{xAUC} > 0$).}
    \label{fig:tree}
\end{figure}
We investigate the impact of disparate censorship and undertesting across different settings.
Recall that disparate censorship can lead to harmful \textit{undertesting} if differences in testing lead to disproportionately higher rates of missed positives in certain groups. While missed positives has immediate negative clinical implications, it can also lead to additional harm if these tests are used in defining labels for training ML models.
We theoretically show that when the marginal or conditional risk distributions differ, performance gaps may arise if certain patient subgroups are undertested. These findings are further supported through a simulation study. 
Our findings raise awareness to undertesting and its impact on training labels---which can potentially drive model performance gaps between groups.

Although systematic biases in any part of the ML pipeline can negatively impact model performance gaps, we focus on biases in clinical data caused by disparities in diagnostic/laboratory test orders.
Specifically, we identify an understudied type of label bias caused by disparities in the delivery of or access to clinical care. Previous work on label bias in ML for healthcare studied label misspecification:~\cite{obermeyer2019dissecting, pierson2021algorithmic} find that using outcomes such as healthcare costs or certain risk scales as proxies for patient need may disproportionately harm Black patients. 
They suggest that training models on redefined outcomes that align better with patient needs may mitigate the negative impacts. 

However, such solutions may be less applicable to \emph{disparate censorship}. In our setting, the labels correspond to the outcome of interest, but are (partially) observed at different rates for each group.
We found such instances of disparate censorship with respect to race in common diagnostic/laboratory tests in the MIMIC-IV dataset, with significantly lower test rates for Black patients as compared to White patients in multiple tests.
This finding is particularly concerning due to the wide usage of MIMIC-IV in clinical ML.

Addressing the adverse effects of disparate censorship and undertesting requires thoroughly understanding one's problem setting. Disparate censorship and undertesting raise few concerns in terms of model performance gaps when marginal and conditional risk distributions are identical across patient groups, but this is unlikely to hold in practice. While we discuss methods for identifying when disparate censorship and undertesting may result in performance gaps, there remain gaps in algorithmic approaches for measuring conditional distributional differences. Addressing algorithmic gaps may require domain knowledge from the clinical literature and health disparities research and/or data with complete observations (\textit{i.e.}, no missed positives).

Once disparate censorship and undertesting are identified, noisy-label and censored ML methods represent a possible direction for mitigating negative impacts~\citep{jiang2020identifying, cheng2020learning, berthon2021confidence,wang2021fair}.
Beyond ML methods, modeling techniques in the presence of censored/missing data include inverse probability weighting-based methods in the epidemiology/causal inference literature~\citep{hernan2004structural} or the Heckman correction in the quantitative social sciences~\citep{heckman1976common}.

In addition to computational solutions, the harmful effects of disparate censorship and undertesting can be minimized by mitigating clinician biases and reducing disparities in covariates across groups.
Mitigating clinician biases targets undertesting, towards ensuring that patients with equal risk are equally likely to be tested. Reducing disparities in covariates addresses Setting 2, mitigating model performance gaps that arise due to undertesting the high-risk group.  While some covariate disparities may be due to physiological differences (\textit{i.e.}, pediatric vs. adult patients), others emerge due to disparities in healthcare access or structural inequality, disadvantaging various groups such as Black and Latinx patients~\citep{brondolo2009race}, immigrant communities~\citep{misra2021structural}, and Black gender minorities intersectionally~\citep{lett2020intersectionality}. These studies further suggest that minimizing covariate differences requires addressing underlying structural inequities in healthcare access and delivery.

The main limitations of our work lie in our theoretical assumptions. 
First, our simulation design implicitly treats testing as diagnosis.
While testing is often a prerequisite to diagnosis, diagnostic decisions may be updated over time, which standard ML development may not capture.
For Settings 2 and 3, our theoretical results predicting convergence to zero performance gap require clinician testing and condition status thresholds to be expressed via a hard threshold with the same functional form.
For Setting 2 in particular, we apply normality assumptions on risk score distributions and patient covariates.
These assumptions are necessary to make the theory tractable. In many cases, we know that clinicians order tests/interventions on the basis of symptoms or clinical suspicion (as studied in~\cite{dolan2005colorectal, schulman1999effect}).
Thus, while we expect our assumptions to partially hold, it is unclear to what extent they hold in practice.

Nevertheless, our work provides a foundation for a deeper exploration of the impacts of disparate censorship and undertesting.
We suggest and identify a plausible mechanism of dataset bias, and demonstrate conditions under which model performance gaps ($\Delta \text{AUC}, \Delta \text{xAUC}$) may arise.
Our findings motivate diligence in understanding and mitigating health disparities, and raise warnings about the responsible deployment of ML systems in healthcare. 
Ultimately, we believe that a combination of computational tools alongside social policy and public health interventions will provide a path to recognize and address the impacts of disparate censorship.

\acks{We thank (in alphabetical order) Donna Tjandra, Erkin \"{O}tle{\c s}, Fahad Kamran, Jiaxuan Wang, Jung Min Lee, Meera Krishnamoorthy, Sarah Jabbour, Shengpu Tang, and Stephanie Shepard for helpful discussions. We also thank Harry Rubin-Falcone, Maggie Makar, Michael Ito for their feedback in preparing the paper presentation. Special thanks to Sarah Jabbour for assisting with chest X-ray data processing. This work was supported by the Agency of Health Research and Quality (grant no. 5R01HS027431-03) and the National Heart Lung and Blood Institute (grant no. 1R01HL158626-01). The views and conclusions in this document are those of the authors and should not be interpreted as necessarily representing the official policies, either expressed or implied, of the Agency of Health Research and Quality or the National Heart Lung and Blood Institute.}

\bibliography{references}

\begin{thebibliography}{55}
\providecommand{\natexlab}[1]{#1}
\providecommand{\url}[1]{\texttt{#1}}
\expandafter\ifx\csname urlstyle\endcsname\relax
  \providecommand{\doi}[1]{doi: #1}\else
  \providecommand{\doi}{doi: \begingroup \urlstyle{rm}\Url}\fi

\bibitem[Adams et~al.(2022)Adams, Henry, Sridharan, Soleimani, Zhan, Rawat,
  Johnson, Hager, Cosgrove, Markowski, et~al.]{adams2022prospective}
Roy Adams, Katharine~E Henry, Anirudh Sridharan, Hossein Soleimani, Andong
  Zhan, Nishi Rawat, Lauren Johnson, David~N Hager, Sara~E Cosgrove, Andrew
  Markowski, et~al.
\newblock Prospective, multi-site study of patient outcomes after
  implementation of the trews machine learning-based early warning system for
  sepsis.
\newblock \emph{Nature Medicine}, pages 1--6, 2022.

\bibitem[Bekker and Davis(2020)]{bekker2020learning}
Jessa Bekker and Jesse Davis.
\newblock Learning from positive and unlabeled data: A survey.
\newblock \emph{Machine Learning}, 109\penalty0 (4):\penalty0 719--760, 2020.

\bibitem[Berry et~al.(2009)Berry, Bumpers, Ogunlade, Glover, Davis,
  Counts-Spriggs, Kauh, and Flowers]{berry2009examining}
Jamillah Berry, Kevin Bumpers, Vickie Ogunlade, Roni Glover, Sharon Davis,
  Margaret Counts-Spriggs, John Kauh, and Christopher Flowers.
\newblock Examining racial disparities in colorectal cancer care.
\newblock \emph{Journal of psychosocial oncology}, 27\penalty0 (1):\penalty0
  59--83, 2009.

\bibitem[Berthon et~al.(2021)Berthon, Han, Niu, Liu, and
  Sugiyama]{berthon2021confidence}
Antonin Berthon, Bo~Han, Gang Niu, Tongliang Liu, and Masashi Sugiyama.
\newblock Confidence scores make instance-dependent label-noise learning
  possible.
\newblock In \emph{International Conference on Machine Learning}, pages
  825--836. PMLR, 2021.

\bibitem[Bhuva et~al.(2019)Bhuva, Bai, Lau, Davies, Ye, Bulluck, McAlindon,
  Culotta, Swoboda, Captur, et~al.]{bhuva2019multicenter}
Anish~N Bhuva, Wenjia Bai, Clement Lau, Rhodri~H Davies, Yang Ye, Heeraj
  Bulluck, Elisa McAlindon, Veronica Culotta, Peter~P Swoboda, Gabriella
  Captur, et~al.
\newblock A multicenter, scan-rescan, human and machine learning cmr study to
  test generalizability and precision in imaging biomarker analysis.
\newblock \emph{Circulation: Cardiovascular Imaging}, 12\penalty0
  (10):\penalty0 e009214, 2019.

\bibitem[Brondolo et~al.(2009)Brondolo, Gallo, and Myers]{brondolo2009race}
Elizabeth Brondolo, Linda~C Gallo, and Hector~F Myers.
\newblock Race, racism and health: disparities, mechanisms, and interventions.
\newblock \emph{Journal of Behavioral Medicine}, 32\penalty0 (1):\penalty0
  1--8, 2009.

\bibitem[Buolamwini and Gebru(2018)]{buolamwini2018gender}
Joy Buolamwini and Timnit Gebru.
\newblock Gender shades: Intersectional accuracy disparities in commercial
  gender classification.
\newblock In \emph{Conference on {F}airness, {A}ccountability and
  {T}ransparency}, pages 77--91. PMLR, 2018.

\bibitem[Chang and Lin(2011)]{chang2011libsvm}
Chih-Chung Chang and Chih-Jen Lin.
\newblock {LIBSVM}: a library for support vector machines.
\newblock \emph{ACM {T}ransactions on {I}ntelligent {S}ystems and {T}echnology
  (TIST)}, 2\penalty0 (3):\penalty0 1--27, 2011.

\bibitem[Cheng et~al.(2020)Cheng, Liu, Ramamohanarao, and
  Tao]{cheng2020learning}
Jiacheng Cheng, Tongliang Liu, Kotagiri Ramamohanarao, and Dacheng Tao.
\newblock Learning with bounded instance and label-dependent label noise.
\newblock In \emph{International Conference on Machine Learning}, pages
  1789--1799. PMLR, 2020.

\bibitem[Courty et~al.(2017)Courty, Flamary, Habrard, and
  Rakotomamonjy]{courty2017joint}
Nicolas Courty, R{\'e}mi Flamary, Amaury Habrard, and Alain Rakotomamonjy.
\newblock Joint distribution optimal transportation for domain adaptation.
\newblock \emph{Advances in Neural Information Processing Systems}, 30, 2017.

\bibitem[Daugherty et~al.(2017)Daugherty, Blair, Havranek, Furniss, Dickinson,
  Karimkhani, Main, and Masoudi]{daugherty2017implicit}
Stacie~L Daugherty, Irene~V Blair, Edward~P Havranek, Anna Furniss, L~Miriam
  Dickinson, Elhum Karimkhani, Deborah~S Main, and Frederick~A Masoudi.
\newblock Implicit gender bias and the use of cardiovascular tests among
  cardiologists.
\newblock \emph{Journal of the American Heart Association}, 6\penalty0
  (12):\penalty0 e006872, 2017.

\bibitem[Daum{\'e}~III(2009)]{daume2009frustratingly}
Hal Daum{\'e}~III.
\newblock Frustratingly easy domain adaptation.
\newblock \emph{arXiv preprint arXiv:0907.1815}, 2009.

\bibitem[Dolan et~al.(2005)Dolan, Ferreira, Fitzgibbon, Davis, Rademaker, Liu,
  Lee, Wolf, Schmitt, and Bennett]{dolan2005colorectal}
Nancy~C Dolan, M~Rosario Ferreira, Marian~L Fitzgibbon, Terry~C Davis, Alfred~W
  Rademaker, Dachao Liu, June Lee, Michael Wolf, Brian~P Schmitt, and Charles~L
  Bennett.
\newblock Colorectal cancer screening among {A}frican-{A}merican and white male
  veterans.
\newblock \emph{American {J}ournal of {P}reventive {M}edicine}, 28\penalty0
  (5):\penalty0 479--482, 2005.

\bibitem[Fleuren et~al.(2020)Fleuren, Klausch, Zwager, Schoonmade, Guo,
  Roggeveen, Swart, Girbes, Thoral, Ercole, et~al.]{fleuren2020machine}
Lucas~M Fleuren, Thomas~LT Klausch, Charlotte~L Zwager, Linda~J Schoonmade,
  Tingjie Guo, Luca~F Roggeveen, Eleonora~L Swart, Armand~RJ Girbes, Patrick
  Thoral, Ari Ercole, et~al.
\newblock Machine learning for the prediction of sepsis: a systematic review
  and meta-analysis of diagnostic test accuracy.
\newblock \emph{Intensive care medicine}, 46\penalty0 (3):\penalty0 383--400,
  2020.

\bibitem[Gaffin et~al.(2010)Gaffin, Shotola, Martin, and
  Phipatanakul]{gaffin2010clinically}
Jonathan~M Gaffin, Nancy~Lichtenberg Shotola, Thomas~R Martin, and Wanda
  Phipatanakul.
\newblock Clinically useful spirometry in preschool-aged children: evaluation
  of the 2007 american thoracic society guidelines.
\newblock \emph{Journal of Asthma}, 47\penalty0 (7):\penalty0 762--767, 2010.

\bibitem[Geirhos et~al.(2020)Geirhos, Jacobsen, Michaelis, Zemel, Brendel,
  Bethge, and Wichmann]{geirhos2020shortcut}
Robert Geirhos, J{\"o}rn-Henrik Jacobsen, Claudio Michaelis, Richard Zemel,
  Wieland Brendel, Matthias Bethge, and Felix~A Wichmann.
\newblock Shortcut learning in deep neural networks.
\newblock \emph{Nature Machine Intelligence}, 2\penalty0 (11):\penalty0
  665--673, 2020.

\bibitem[Gong et~al.(2021)Gong, Wang, Liu, Han, You, Yang, and
  Tao]{gong2021instance}
Chen Gong, Qizhou Wang, Tongliang Liu, Bo~Han, Jane~J You, Jian Yang, and
  Dacheng Tao.
\newblock Instance-dependent positive and unlabeled learning with labeling bias
  estimation.
\newblock \emph{IEEE Transactions on Pattern Analysis and Machine
  Intelligence}, 2021.

\bibitem[Gretton et~al.(2012)Gretton, Borgwardt, Rasch, Sch{\"o}lkopf, and
  Smola]{gretton2012kernel}
Arthur Gretton, Karsten~M Borgwardt, Malte~J Rasch, Bernhard Sch{\"o}lkopf, and
  Alexander Smola.
\newblock A kernel two-sample test.
\newblock \emph{The Journal of Machine Learning Research}, 13\penalty0
  (1):\penalty0 723--773, 2012.

\bibitem[Hartvigsen et~al.(2018)Hartvigsen, Sen, Brownell, Teeple, Kong, and
  Rundensteiner]{hartvigsen2018early}
Thomas Hartvigsen, Cansu Sen, Sarah Brownell, Erin Teeple, Xiangnan Kong, and
  Elke~A Rundensteiner.
\newblock Early prediction of mrsa infections using electronic health records.
\newblock In \emph{HEALTHINF}, pages 156--167, 2018.

\bibitem[Heckman(1976)]{heckman1976common}
James~J Heckman.
\newblock The common structure of statistical models of truncation, sample
  selection and limited dependent variables and a simple estimator for such
  models.
\newblock In \emph{Annals of economic and social measurement, volume 5, number
  4}, pages 475--492. NBER, 1976.

\bibitem[Henry et~al.(2015)Henry, Hager, Pronovost, and
  Saria]{henry2015targeted}
Katharine~E Henry, David~N Hager, Peter~J Pronovost, and Suchi Saria.
\newblock A targeted real-time early warning score (trewscore) for septic
  shock.
\newblock \emph{Science translational medicine}, 7\penalty0 (299):\penalty0
  299ra122--299ra122, 2015.

\bibitem[Henry et~al.(2019)Henry, Hager, Osborn, Wu, and
  Saria]{henry2019comparison}
Katharine~E Henry, David~N Hager, Tiffany~M Osborn, Albert~W Wu, and Suchi
  Saria.
\newblock Comparison of automated sepsis identification methods and electronic
  health record--based sepsis phenotyping: improving case identification
  accuracy by accounting for confounding comorbid conditions.
\newblock \emph{Critical care explorations}, 1\penalty0 (10), 2019.

\bibitem[Hern{\'a}n et~al.(2004)Hern{\'a}n, Hern{\'a}ndez-D{\'\i}az, and
  Robins]{hernan2004structural}
Miguel~A Hern{\'a}n, Sonia Hern{\'a}ndez-D{\'\i}az, and James~M Robins.
\newblock A structural approach to selection bias.
\newblock \emph{Epidemiology}, pages 615--625, 2004.

\bibitem[Ioannidis et~al.(2021)Ioannidis, Powe, and
  Yancy]{ioannidis2021recalibrating}
John~PA Ioannidis, Neil~R Powe, and Clyde Yancy.
\newblock Recalibrating the use of race in medical research.
\newblock \emph{Jama}, 325\penalty0 (7):\penalty0 623--624, 2021.

\bibitem[Jabbour et~al.(2020)Jabbour, Fouhey, Kazerooni, Sjoding, and
  Wiens]{jabbour2020deep}
Sarah Jabbour, David Fouhey, Ella Kazerooni, Michael~W Sjoding, and Jenna
  Wiens.
\newblock Deep learning applied to chest x-rays: Exploiting and preventing
  shortcuts.
\newblock In \emph{Machine Learning for Healthcare Conference}, pages 750--782.
  PMLR, 2020.

\bibitem[Jiang and Nachum(2020)]{jiang2020identifying}
Heinrich Jiang and Ofir Nachum.
\newblock Identifying and correcting label bias in machine learning.
\newblock In \emph{International Conference on Artificial Intelligence and
  Statistics}, pages 702--712. PMLR, 2020.

\bibitem[Jiang(2008)]{jiang2008literature}
Jing Jiang.
\newblock A literature survey on domain adaptation of statistical classifiers.
\newblock \emph{URL: http://sifaka. cs. uiuc.
  edu/jiang4/domainadaptation/survey}, 3\penalty0 (1-12):\penalty0 3, 2008.

\bibitem[Johnson et~al.(2020)Johnson, Bulgarelli, Pollard, Horng, Celi, and
  Mark]{johnson2020mimic}
Alistair Johnson, Lucas Bulgarelli, Tom Pollard, Steven Horng, Leo~Anthony
  Celi, and Roger Mark.
\newblock Mimic-iv.
\newblock \emph{version 0.4). PhysioNet. https://doi. org/10.13026/a3wn-hq05},
  2020.

\bibitem[Kallus and Zhou(2019)]{kallus2019fairness}
Nathan Kallus and Angela Zhou.
\newblock The fairness of risk scores beyond classification: {B}ipartite
  ranking and the xauc metric.
\newblock \emph{Advances in {N}eural {I}nformation {P}rocessing {S}ystems},
  32:\penalty0 3438--3448, 2019.

\bibitem[Lackland(2014)]{lackland2014racial}
Daniel~T Lackland.
\newblock Racial differences in hypertension: implications for high blood
  pressure management.
\newblock \emph{The American {J}ournal of the {M}edical {S}ciences},
  348\penalty0 (2):\penalty0 135--138, 2014.

\bibitem[Lett et~al.(2020)Lett, Dowshen, and Baker]{lett2020intersectionality}
Elle Lett, Nadia~L Dowshen, and Kellan~E Baker.
\newblock Intersectionality and health inequities for gender minority blacks in
  the us.
\newblock \emph{American Journal of Preventive Medicine}, 59\penalty0
  (5):\penalty0 639--647, 2020.

\bibitem[Lichtman et~al.(2018)Lichtman, Leifheit, Safdar, Bao, Krumholz,
  Lorenze, Daneshvar, Spertus, and D’Onofrio]{lichtman2018sex}
Judith~H Lichtman, Erica~C Leifheit, Basmah Safdar, Haikun Bao, Harlan~M
  Krumholz, Nancy~P Lorenze, Mitra Daneshvar, John~A Spertus, and Gail
  D’Onofrio.
\newblock Sex differences in the presentation and perception of symptoms among
  young patients with myocardial infarction: evidence from the virgo study
  (variation in recovery: role of gender on outcomes of young ami patients).
\newblock \emph{Circulation}, 137\penalty0 (8):\penalty0 781--790, 2018.

\bibitem[Magesh et~al.(2021)Magesh, John, Li, Li, Mattingly-App, Jain, Chang,
  and Ongkeko]{magesh2021disparities}
Shruti Magesh, Daniel John, Wei~Tse Li, Yuxiang Li, Aidan Mattingly-App, Sharad
  Jain, Eric~Y Chang, and Weg~M Ongkeko.
\newblock Disparities in covid-19 outcomes by race, ethnicity, and
  socioeconomic status: a systematic-review and meta-analysis.
\newblock \emph{JAMA network open}, 4\penalty0 (11):\penalty0
  e2134147--e2134147, 2021.

\bibitem[Massey~Jr(1951)]{massey1951kolmogorov}
Frank~J Massey~Jr.
\newblock The {K}olmogorov-{S}mirnov test for goodness of fit.
\newblock \emph{Journal of the American statistical Association}, 46\penalty0
  (253):\penalty0 68--78, 1951.

\bibitem[Menon et~al.(2018)Menon, Van~Rooyen, and Natarajan]{menon2018learning}
Aditya~Krishna Menon, Brendan Van~Rooyen, and Nagarajan Natarajan.
\newblock Learning from binary labels with instance-dependent noise.
\newblock \emph{Machine Learning}, 107\penalty0 (8):\penalty0 1561--1595, 2018.

\bibitem[Misra et~al.(2021)Misra, Kwon, Abra{\'\i}do-Lanza, Chebli,
  Trinh-Shevrin, and Yi]{misra2021structural}
Supriya Misra, Simona~C Kwon, Ana~F Abra{\'\i}do-Lanza, Perla Chebli, Chau
  Trinh-Shevrin, and Stella~S Yi.
\newblock Structural racism and immigrant health in the united states.
\newblock \emph{Health Education \& Behavior}, 48\penalty0 (3):\penalty0
  332--341, 2021.

\bibitem[Obermeyer et~al.(2019)Obermeyer, Powers, Vogeli, and
  Mullainathan]{obermeyer2019dissecting}
Ziad Obermeyer, Brian Powers, Christine Vogeli, and Sendhil Mullainathan.
\newblock Dissecting racial bias in an algorithm used to manage the health of
  populations.
\newblock \emph{Science}, 366\penalty0 (6464):\penalty0 447--453, 2019.

\bibitem[Oh et~al.(2018)Oh, Makar, Fusco, McCaffrey, Rao, Ryan, Washer, West,
  Young, Guttag, et~al.]{oh2018generalizable}
Jeeheh Oh, Maggie Makar, Christopher Fusco, Robert McCaffrey, Krishna Rao,
  Erin~E Ryan, Laraine Washer, Lauren~R West, Vincent~B Young, John Guttag,
  et~al.
\newblock A generalizable, data-driven approach to predict daily risk of
  clostridium difficile infection at two large academic health centers.
\newblock \emph{infection control \& hospital epidemiology}, 39\penalty0
  (4):\penalty0 425--433, 2018.

\bibitem[Patel et~al.(2021)Patel, Steinberg, Pfohl, and
  Shah]{patel2021learning}
Birju~S Patel, Ethan Steinberg, Stephen~R Pfohl, and Nigam~H Shah.
\newblock Learning decision thresholds for risk stratification models from
  aggregate clinician behavior.
\newblock \emph{Journal of the American Medical Informatics Association},
  28\penalty0 (10):\penalty0 2258--2264, 2021.

\bibitem[Pauker and Kassirer(1980)]{pauker1980threshold}
Stephen~G Pauker and Jerome~P Kassirer.
\newblock The threshold approach to clinical decision making.
\newblock \emph{New England Journal of Medicine}, 302\penalty0 (20):\penalty0
  1109--1117, 1980.

\bibitem[Pierson et~al.(2018)Pierson, Corbett-Davies, and
  Goel]{pierson2018fast}
Emma Pierson, Sam Corbett-Davies, and Sharad Goel.
\newblock Fast threshold tests for detecting discrimination.
\newblock In \emph{International Conference on Artificial Intelligence and
  Statistics}, pages 96--105. PMLR, 2018.

\bibitem[Pierson et~al.(2021)Pierson, Cutler, Leskovec, Mullainathan, and
  Obermeyer]{pierson2021algorithmic}
Emma Pierson, David~M Cutler, Jure Leskovec, Sendhil Mullainathan, and Ziad
  Obermeyer.
\newblock An algorithmic approach to reducing unexplained pain disparities in
  underserved populations.
\newblock \emph{Nature Medicine}, 27\penalty0 (1):\penalty0 136--140, 2021.

\bibitem[Platt et~al.(1999)]{platt1999probabilistic}
John Platt et~al.
\newblock Probabilistic outputs for support vector machines and comparisons to
  regularized likelihood methods.
\newblock \emph{Advances in {L}arge {M}argin {C}lassifiers}, 10\penalty0
  (3):\penalty0 61--74, 1999.

\bibitem[Price-Haywood et~al.(2020)Price-Haywood, Burton, Fort, and
  Seoane]{price2020hospitalization}
Eboni~G Price-Haywood, Jeffrey Burton, Daniel Fort, and Leonardo Seoane.
\newblock Hospitalization and mortality among black patients and white patients
  with {C}ovid-19.
\newblock \emph{New England Journal of Medicine}, 382\penalty0 (26):\penalty0
  2534--2543, 2020.

\bibitem[Reyna et~al.(2019)Reyna, Josef, Seyedi, Jeter, Shashikumar, Westover,
  Sharma, Nemati, and Clifford]{reyna2019early}
Matthew~A Reyna, Chris Josef, Salman Seyedi, Russell Jeter, Supreeth~P
  Shashikumar, M~Brandon Westover, Ashish Sharma, Shamim Nemati, and Gari~D
  Clifford.
\newblock Early prediction of sepsis from clinical data: the
  physionet/computing in cardiology challenge 2019.
\newblock In \emph{2019 Computing in Cardiology (CinC)}, pages Page--1. IEEE,
  2019.

\bibitem[Rhee and Klompas(2020)]{rhee2020sepsis}
Chanu Rhee and Michael Klompas.
\newblock Sepsis trends: increasing incidence and decreasing mortality, or
  changing denominator?
\newblock \emph{Journal of Thoracic Disease}, 12\penalty0 (Suppl 1):\penalty0
  S89, 2020.

\bibitem[Schulman et~al.(1999)Schulman, Berlin, Harless, Kerner, Sistrunk,
  Gersh, Dube, Taleghani, Burke, Williams, et~al.]{schulman1999effect}
Kevin~A Schulman, Jesse~A Berlin, William Harless, Jon~F Kerner, Shyrl
  Sistrunk, Bernard~J Gersh, Ross Dube, Christopher~K Taleghani, Jennifer~E
  Burke, Sankey Williams, et~al.
\newblock The effect of race and sex on physicians' recommendations for cardiac
  catheterization.
\newblock \emph{New England Journal of Medicine}, 340\penalty0 (8):\penalty0
  618--626, 1999.

\bibitem[Seymour et~al.(2016)Seymour, Liu, Iwashyna, Brunkhorst, Rea, Scherag,
  Rubenfeld, Kahn, Shankar-Hari, Singer, et~al.]{seymour2016assessment}
Christopher~W Seymour, Vincent~X Liu, Theodore~J Iwashyna, Frank~M Brunkhorst,
  Thomas~D Rea, Andr{\'e} Scherag, Gordon Rubenfeld, Jeremy~M Kahn, Manu
  Shankar-Hari, Mervyn Singer, et~al.
\newblock Assessment of clinical criteria for sepsis: for the third
  international consensus definitions for sepsis and septic shock (sepsis-3).
\newblock \emph{Jama}, 315\penalty0 (8):\penalty0 762--774, 2016.

\bibitem[Simoiu et~al.(2017)Simoiu, Corbett-Davies, and
  Goel]{simoiu2017problem}
Camelia Simoiu, Sam Corbett-Davies, and Sharad Goel.
\newblock The problem of infra-marginality in outcome tests for discrimination.
\newblock \emph{The Annals of Applied Statistics}, 11\penalty0 (3):\penalty0
  1193--1216, 2017.

\bibitem[Singh et~al.(2017)Singh, Daus, Allender, Ramey, Martin, Perry,
  De~Los~Reyes, and Vedamuthu]{singh2017social}
Gopal~K Singh, Gem~P Daus, Michelle Allender, Christine~T Ramey, Elijah~K
  Martin, Chrisp Perry, Andrew~A De~Los~Reyes, and Ivy~P Vedamuthu.
\newblock Social determinants of health in the {U}nited {S}tates: addressing
  major health inequality trends for the nation, 1935-2016.
\newblock \emph{International Journal of MCH and AIDS}, 6\penalty0
  (2):\penalty0 139, 2017.

\bibitem[Spector-Bagdady et~al.(2021)Spector-Bagdady, Tang, Jabbour, Price,
  Bracic, Creary, Kheterpal, Brummett, and Wiens]{spector2021respecting}
Kayte Spector-Bagdady, Shengpu Tang, Sarah Jabbour, W~Nicholson Price, Ana
  Bracic, Melissa~S Creary, Sachin Kheterpal, Chad~M Brummett, and Jenna Wiens.
\newblock Respecting autonomy and enabling diversity: The effect of eligibility
  and enrollment on research data demographics.
\newblock \emph{Health Affairs}, 40\penalty0 (12):\penalty0 1892--1899, 2021.

\bibitem[Teeple et~al.(2020)Teeple, Hartvigsen, Sen, Claypool, and
  Rundensteiner]{teeple2020clinical}
Erin Teeple, Thomas Hartvigsen, Cansu Sen, Kajal~T Claypool, and Elke~A
  Rundensteiner.
\newblock Clinical performance evaluation of a machine learning system for
  predicting hospital-acquired clostridium difficile infection.
\newblock In \emph{HEALTHINF}, pages 656--663, 2020.

\bibitem[Vyas et~al.(2020)Vyas, Eisenstein, and Jones]{vyas2020hidden}
Darshali~A Vyas, Leo~G Eisenstein, and David~S Jones.
\newblock Hidden in plain sight—reconsidering the use of race correction in
  clinical algorithms, 2020.

\bibitem[Wang et~al.(2021)Wang, Liu, and Levy]{wang2021fair}
Jialu Wang, Yang Liu, and Caleb Levy.
\newblock Fair classification with group-dependent label noise.
\newblock In \emph{Proceedings of the 2021 ACM Conference on Fairness,
  Accountability, and Transparency}, pages 526--536, 2021.

\bibitem[Yu et~al.(2020)Yu, Shaker, Alesiani, and Principe]{yu2020measuring}
Shujian Yu, Ammar Shaker, Francesco Alesiani, and Jose~C Principe.
\newblock Measuring the discrepancy between conditional distributions: Methods,
  properties and applications.
\newblock \emph{arXiv preprint arXiv:2005.02196}, 2020.

\end{thebibliography}

\clearpage

\appendix

\renewcommand{\thesection}{\Alph{section}}

\section{Proofs}
\label{appdx:theory}

\subsection{Preliminaries}

First, we restate the definition of a boundary-consistent noise (BCN) model from~\cite{menon2018learning}, which is useful for our proofs.

\begin{definition}[Boundary-consistent noise (BCN) model.]
\label{def:bcn}
Define class probability function $\eta(\x) = P(Y = 1 \mid \x)$. Consider a data generating process in which $(\x, \tilde{y})$ is generated by drawing an instance $(\x, Y)$ and flipping $Y$ with instance- and label-dependent probability $\rho_Y(\x)$. Suppose that label flip-probability functions $\rho_{y}$ can be written in the form $\rho_{y} = f_y \circ s$, where $f_y: \R \to [0, 1]$ for $y \in \{0, 1\}$ and $s: \mathcal{X} \to \R$, and $\rho_{0}(\x) + \rho_{1}(\x) < 1$ for all $\x$. Then, a noise model $(f_0, f_1, s, \eta)$ is BCN-admissible if the following conditions are satisfied:
\begin{itemize}
    \item \textbf{Feasible ranking:} $s$ is order-preserving in $\x$ for $\eta$; that is, for any $(\x, \x') \in \mathcal{X}$, then $\eta(\x) < \eta(\x')$ implies $s(\x) < s(\x')$.
    \item \textbf{Piecewise-monotonicity:} $f_{0}$ and $f_1$ are non-decreasing where $\eta \leq \frac{1}{2}$, and non-increasing otherwise.
    \item \textbf{Flip-probability monotonicity:} $f_1(z) - f_0(z)$ is non-increasing in $z$.
\end{itemize}
\end{definition}

Note that, for our setting, since $f_0$ is constant, we can combine the two monotonicity constraints: it is sufficient that $f_1$ is non-increasing in $z$. Furthermore, we restate an important property of the BCN model shown in~\cite{menon2018learning} with some notation adapted to our setting:

\begin{theorem}[Theorem 2 of~\cite{menon2018learning}.]
\label{thrm:bcn_menon}
Pick any distribution $D$. Let $\bar{D}$ be a corrupted distribution. Suppose that for any $\x, \x' \in \mathcal{X}$, $\eta(\x) < \eta(\x')$ implies $\tilde{\eta}(\x) < \tilde{\eta}(\x')$, where $\tilde{\eta}$ is the analogue of $\eta$ on noisy labels $\tilde{y}$ (\textit{i.e.} $\tilde{\eta}(\x) - P(\tilde{Y} = 1 \mid X = \x)$), and there exists a constant $C$ such that $|\eta(\x) < \eta(\x')| \leq C \cdot |\tilde{\eta}(\x) < \tilde{\eta}(\x')|$. Then, for any scorer $s$, 

\begin{equation}
    \text{reg}_{rank}(s; D) \leq C \cdot \frac{\bar{\pi}(1 - \bar{\pi})}{\pi (1 - \pi)} \cdot \text{reg}_{rank}(s; \bar{D}) 
\end{equation}

where $\text{reg}_{rank}$ is the excess ranking risk of a scorer $s$, and $\pi = P(Y = 1), \bar{\pi} = P(\tilde{Y} = 1)$. In particular, if $\bar{D} = BCN(D, f_0,f_1, \eta)$
where $(f_0,f_1,s,\eta)$ are BCN-admissible, then $C=(1- 2 \cdot \rho_{max})-1$, where $\rho_{max} = \underset{\x \in \mathcal{X}}{\sup}P(\tilde{Y} = 0 \mid Y = 1, X = \x)$ is sufficient.
\end{theorem}

We defer to~\cite{menon2018learning} for the proof. This states that optimizing a model for AUC on noisy labels in a BCN model is consistent with optimizing a model for AUC on clean labels: both converge to the same Bayes-optimal scorer.

Lastly, we prove a Lemma relating the Bayes-optimal overall AUC to the Bayes-optimal within-group AUC and xAUC under perfect separability.

\begin{lemma}[No ranking performance gap under perfect separation.]
\label{lemma:no_gap}
Let $\eta(\x) = P(Y = 1 \mid X = \x)$. If $Y$ is perfectly separable in $\x$; that is, there exists some $s: \mathcal{X} \to \{0, 1\}$ such that $s(\x) = y$ for all $\x \in \mathcal{X}$, and $P(X \mid Y = y, A = a) > 0$ for any $y, a \in \{0, 1\}$, then $\Delta \text{AUC}, \Delta \text{xAUC} = 0$. 
\end{lemma}

\begin{proof}
Under perfect separation, for the Bayes-optimal ranker $\eta$, we know that $P(\eta(\x) < \eta(\x')) = 1$ where $\x \in \{x \mid x \in S, Y = 1\}$, $\x' \in \{x \mid x \in S', Y = 0\}$ for any $S, S' \subseteq \mathcal{X}$. 

First, we show that $\Delta \text{AUC} = 0$. Choose $S = \{x \mid x \in \mathcal{X}, Y = 1, A = a\}$ and $S' = \}x \mid x \in \mathcal{X}, Y = 0, A = a\}$ for arbitrary $a \in \{0, 1\}$. Then the Bayes-optimal within-group AUC is 1 for each $a$, so $\Delta \text{AUC} = 0$ as required.
Now, we show that $\Delta \text{xAUC} = 0$. Choose $S = \{x \mid x \in \mathcal{X}, Y = 1, A = 1\}$ and $S' = \{x \mid x \in \mathcal{X}, Y = 0, A = 0\}$ (without loss of generality in assignment of $A$). Then the Bayes-optimal xAUC is 1 for each assignment of $A$, so $\Delta \text{xAUC} = 0$.
We have shown that $\Delta \text{AUC}, \Delta \text{xAUC} = 0$, concluding the proof.
\end{proof}

\subsection{Proof of Theorem~\ref{thrm:marginal}}

\begin{proof}
First, we show feasible ranking holds. As $y \in \{0, 1\}$, $\eta(\x) < \eta(\x')$ implies that $s(\x) \leq b$ and $s(\x') > b$, from which $s(\x) < s(\x')$ follows as required.

Now, define $\eta(\x) = P(Y = 1 \mid X = \x)$, and $f_1(z) = P(\tilde{Y} = 0 \mid Y = 1, Z = z)$ for $z = s(\x)$, and $p_a$ as $P(A=0)$. We show that $f_1$ is monotonically non-increasing on the interval where $\eta > 1/2$ if either no positives are tested or $\tau_0 \geq \tau_1$. If no positives are tested, then the theorem is vacuously true. Otherwise, we first define $b = \underset{\x \in \mathcal{X}: y=1}{\inf}(s(\x))$. In other words, $b$ is the threshold that perfectly separates negative from positives examples. Then, choose some $\tau_1 > b$, and some $\tau_0 \geq \tau_1$. %
We can write $f_1$, which is the probability that a positive label is flipped, as

\begin{equation}
    f_1(z) = \begin{cases}
            1 - c & z < \tau_1 \\
            \frac{p_a \cdot \exp\left(-\frac{(z - \mu_0)^2}{2\sigma^2}\right)}{p_a  \cdot \exp\left(-\frac{(z - \mu_0)^2}{2\sigma^2}\right) + (1-p_a) \cdot (1-c) \cdot \exp\left(-\frac{(z - \mu_1)^2}{2\sigma^2}\right)} & \tau_1 \leq z < \tau_0 \\
            0 & z \geq \tau_0
        \end{cases}.
\end{equation}

This is clearly non-increasing on $(-\infty, \tau_1)$, $[\tau_0, \infty)$, so it suffices to show that $f_1(z)$ is non-increasing on $[\tau_1, \tau_0)]$, and that $f_1(\tau_1) \leq 1 - c, f_1(\tau_0) \geq 0$.  Note that the portion of $f_1$ for $\tau_0 \leq z < \tau_1$ is simply $Pr[A = 1] / (Pr[A = 0] + Pr[A = 1])$, since group $A = 1$ is censored with probability $1-c$ and group $A = 0$ is not censored in that region. Denote this function as $a(z)$.\footnote{If $\tau_1 \leq b$ instead, then only the portion of $f_1$ where $f_1(z) = 0$ is observed, which is trivially non-increasing.} We rewrite the portion of $f_1$ for $\tau_1 \leq z < \tau_0$ as a sigmoid function:
\begin{align}
    a(z) &= \frac{p_a \cdot \exp\left(-\frac{(z - \mu_0)^2}{2\sigma^2}\right)}{p_a  \cdot \exp\left(-\frac{(z - \mu_0)^2}{2\sigma^2}\right) + (1-p_a) \cdot (1-c) \cdot \exp\left(-\frac{(z - \mu_1)^2}{2\sigma^2}\right)} \\
    &= \frac{1}{1 + \exp\left(\log \frac{(1-p_a)(1-c)}{p_a}+\frac{(z-\mu_0)^2 - (z - \mu_1)^2}{2\sigma^2}\right)}\\
    &= \frac{1}{1 + \exp\left(\log \frac{(1-p_a)(1-c)}{p_a }+\frac{1}{2\sigma^2}\left[(2z - \mu_0 - \mu_1)(\mu_1 - \mu_0)\right]\right)}\label{eq:sigmoid_form}.
\end{align}

We further rewrite Eq.~\ref{eq:sigmoid_form} in the form $\sigma(g(z))$, where 
\begin{equation}
    g(z) = \log \frac{p_a}{(1-p_a)(1-c)} + \frac{(2z - \mu_0 - \mu_1)(\mu_0 - \mu_1)}{2\sigma^2}.
\end{equation}
Then, taking derivatives:
\begin{equation}
    \frac{d}{dz}a(z) = \sigma(g(z)) (1 - \sigma(g(z))) \frac{d}{dz} g(z) = \sigma(g(z)) (1 - \sigma(g(z))) \left((\mu_0 - 
\mu_1) / 2\sigma^2\right) \leq 0,
\end{equation}
where the final inequality follows since $\mu_1 \geq \mu_0$ by assumption. Thus, $a(z)$ is monotonically non-increasing, which is what we wanted to show. To conclude, since $g: \R \to \R$ and $\sigma: \R \to (0, 1)$, clearly $f_1(\tau_0) \geq 0$. 
Furthermore, rewriting the constraint $f_1(z) = \sigma(g(\tau_1)) \leq 1-c$ and simplifying yields
\begin{equation}
    \tau_1 \leq \log \frac{(1-c)^2 (1-p_a)}{c \cdot p_a} \cdot \frac{2\sigma^2}{\mu_0 - \mu_1} + \frac{\mu_0 + \mu_1}{2},
\end{equation}
\textit{i.e.}, some negative offset of the midpoint between the group-wise means $\mu_0, \mu_1$.
So $f_1(\tau_1) \leq 1 - c$ for such choices of $\tau_1$.
Thus, $(f_0, f_1, s, \eta)$ is BCN admissible. Applying Theorem 5 and Lemma 2 concludes the proof.
\end{proof}

\begin{remark}
Our simulation yields a risk score distribution in random variable $s_{\alpha}(\x)$. Assuming the effect of clipping with respect to range $[0, 1]$ is negligible, which is true for $\mu_a$ near 0.5, the distribution $s_{\alpha}(\x)$ is approximately univariate Gaussian, allowing the application of Theorem~\ref{thrm:marginal}. This is because the sum of independent Gaussians is Gaussian, so by the definition of $s_{\alpha}(\x)$, the distribution of scores is approximately a discretized univariate Gaussian distribution.
\end{remark}

\begin{remark}
We can apply the same argument in this proof to risk distributions beyond homoscedastic Gaussians, based on whether $Pr[A = 0] / (Pr[A = 0] + Pr[A = 1])$ is non-decreasing on $[\tau_1, \tau_0)$. Let $R_0(z), R_1(z)$ be the probability density functions of risk scores $z$ for group $a=0, a=1$, respectively. Clearly, if $R_0(z) / [R_0(z) + R_1(z)]$ is non-decreasing on $[\tau_1, \tau_0)$, that would satisfy boundary consistency. Alternately, applying quotient rule shows that $R'_0(z) \cdot R_1(z) - R_0(z) \cdot R'_1(z) > 0$ is also sufficient to violate boundary consistency.
\end{remark}

\subsection{Proof of Theorem~\ref{thrm:conditional}}

\begin{proof}
We prove that each BCN condition is satisfied. First, we show that feasible ranking holds. Choose any $\x, \x' \in \mathcal{X}$. Recall that $t = \mathbbm{1}[\bm{\theta}^\top \x + \beta > 0 \lor p=1]$ where $p \sim Bernoulli(c)$ and $y = \mathbbm{1}[s_a(\x) > 0]$, where $s_a(\x) = \bm{\theta}_a^\top \x + b_a$. Since $y \in \{0, 1\}$, $\eta(\x) < \eta(\x')$ implies that $s_a(\x) \leq 0$ and $s_a(\x') > 0$, from which $s_a(\x) < s_a(\x')$ follows, as required.

Next, we show that the piecewise-monotonicity constraint holds. To do so, we need to show that the flip probability function $f_1$ is non-increasing where $\eta(\x) \geq 1/2$, or on $s(x) = \theta_a^\top x + b_a > 0$ for all $a$. Consider an arbitrary group $a$ with corresponding $\theta_a, b_a$.
We can write $f_1$ as
\begin{equation}
    f_1(\cdot) = \begin{cases}
        c & \bm{\theta}^\top \x + \beta \leq 0 \\
        0 & \text{otherwise}
    \end{cases}.
\end{equation}
Now, suppose that there exists some $\delta \in \R, \delta > 0$ such that $\bm{\theta} = \bm{\theta}_a \delta$. Choose any $\x, \x' \in \mathcal{X}$ such that $0 < \bm{\theta}_a^\top \x + b_a < \bm{\theta}_a^\top \x' + b_a$.\footnote{Note that, by the definition of $y$ (which is $\eta(\cdot)$ in this setting), we only need to (and can only) show monotonicity for $\x$ such that $\theta_a^\top \x + b_a > 0$.} Then:
\begin{align}
   \bm{\theta}_a^\top (\x - \x') > 0 &\Longleftrightarrow \delta\bm{\theta}_a^\top (\x - \x') > 0\\ 
   &\Longleftrightarrow \bm{\theta}^\top (\x - \x') > 0 \\
   &\Longleftrightarrow \bm{\theta}^\top \x + \beta > \bm{\theta}^\top \x' + \beta.\label{eq:final_ineq}
\end{align}
Let $L = \bm{\theta}^\top \x + \beta, R = \bm{\theta}^\top \x' + \beta$. There are three possible value-pairs in the final inequality: (1) $L < 0, R < 0$, (2) $L < 0, R \geq 0$, and (3) $L \geq 0, R \geq 0$. %
We need to show that, in each setting, $f(L) \geq f(R)$. For the first case, $L < 0, R < 0$ implies that $f_1(L) = f_0(L) = c$, so $f(L) \geq f(R)$ clearly. For the second case, $L < 0, R \geq 0$ implies that $f_1(L) = c, f_0(L) = 0$, so $f(L) > f(R)$, from which $f(L) \geq f(R)$ is assured. Lastly, the third case is identical to the first as $f_1(L) = f_0(L) = 0$ if both $L, R \geq 0$. Thus, piecewise-monotonicity is satisfied.

Lastly, as $f_0$ is the zero function, flip-probability monotonicity follows for free from the preceding. As all three BCN conditions are satisfied, and the choice of $a$ was arbitrary, $(f_0, f_1, s, \eta)$ is BCN-admissible, which is what we wanted to show. Applying Theorem~\ref{thrm:bcn_menon}, Lemma~\ref{lemma:no_gap} concludes the proof.
\end{proof}

\begin{remark}
The generalization to non-linear decision boundaries follows naturally from using reproducing kernel Hilbert spaces $\phi(\x), \phi(\x') \in \mathcal{H}$ where the kernel for $\mathcal{H}$, $K: \mathcal{X} \times \mathcal{X} \to \R$ corresponds to a feature map $\phi: \mathcal{X} \to \mathcal{H}$ (\textit{i.e.} $\phi(\x)^\top \phi(\x') = K(\x, \x')$). Then, since a Hilbert space is a complete metric space (\textit{i.e.} inner product is well-defined), the proof proceeds identically.
\end{remark}

\begin{remark}
The generalization to $s$ of the form $s_a(\x) = \bm{\theta}_a^\top g(\x) +\beta$ where $g$ is element-wise monotonic non-decreasing is also straightforward. We say that a function $g: \R^d \to \R$ is element-wise monotonic non-decreasing if for any $\x \prec \x'$, $g(\x) \succeq g(\x')$, where $\prec, \succeq$ are element-wise inequality operators. Then, substituting $g(\x), g(\x')$ for $\x, \x'$ in the proof, in Eq.~\ref{eq:final_ineq}, since $g$ is element-wise monotonic non-decreasing, then $\bm{\theta}^\top g(\x) > \bm{\theta}^\top g(\x')$ implies $\bm{\theta}^\top \x > \bm{\theta}^\top \x'$, from which the proof proceeds identically.
\end{remark}

\section{Simulation Study Details}
\label{appdx:sim}

The simulation takes global parameters $\mu_a \in \R, \sigma^2 \in \R, \tau_a \in \R , c \in (0, 1]$, and functions $s_a: \mathcal{X} \to \R$. Note that parameters subscripted by $a \in \{0, 1\}$ may vary by group $a$. For each individual, the data generating process proceeds:
\begin{align}
    \x &\sim \max(0, \min(1, \mathcal{N}(\mu_a\mathbf{1}, \sigma^2\mathbf{I})))\label{eq:x}\\
    y &= \mathbbm{1}\left[ s_{a}(\x) > 5 \right] \label{eq:y}\\
    t &\sim \max\left(\mathbbm{1} \left[s_{a}(\x) > \tau_a \lor p = 1 \right] \right), \quad p \sim Bernoulli(c)  \label{eq:t}\\
    \tilde{y} &= y \cdot t \label{eq:y_obs}
\end{align}

Concretely, we generate $\x \in \R^{10}$ for each individual from a multivariate normal distribution. Each Gaussian is assigned mean $\mu_a\mathbf{1}$ based on $a$, where $\mathbf{1}$ is a 10-dimensional all-ones vector and $\mathbf{I}$ is the identity matrix of size $10 \times 10$. We clip all covariates between 0 and 1 (Statement~\ref{eq:x}). The process for generating $y$ follows directly from our theoretical setup.
For $s_a(\x)$, we use an 10-dimensional ceiling function (see Figure~\ref{fig:example} for a 2D example):
\begin{equation}
    s_{0}(\x) = s_1(\x) = \frac{1}{5} \left(\sum_{i=1}^{10} \lceil 5 x_i \rceil \right).\label{eq:down_staircase}
\end{equation}
This function discretizes each element of $\x$ ($x_i$) into 5 equally-spaced bins of size $1/5$. %
If the sum of these values exceeds 5, then $y = 1$ (Eq.~\ref{eq:y}). Note that $s_a(\cdot)$ can be interpreted as a true risk function for $y$ as a function of $\x$. Except where specified, $s_0(\cdot) = s_1(\cdot)$. %
We generate $t$ similarly to $y$ based on a threshold applied to $s_a(\cdot)$, but vary $\tau_a \in \R$ as an experimental threshold parameter to control the level of undertesting (Eq.~\ref{eq:t}).
If the risk for a patient in group $a$ lies above $\tau_a$, they are tested with probability 1; otherwise, they are tested with probability $c = 0.05$. %
Lastly, $\tilde{y} \triangleq y$ if $t = 1$ and is 0 otherwise (Eq.~\ref{eq:y_obs}). This models the fact that a test result is only observable when a test is ordered. %

\begin{table}[t]
    \centering
    \begin{tabular}{c c|c c}
    \toprule
         \textbf{Setting} & $\mu$ & $s_a(\x)$ \\
    \midrule
         1 & $\mu_0 = \mu_1 = 0.45$ & $s_0(\x) = s_1(\x) = s(\x)$ \\
         2 & $\mu_0 = 0.35, \mu_1 = 0.55$ & $s_0(\x) = s_1(\x) = s(\x)$  \\
         3 & $\mu_0 = \mu_1 = 0.45$ & $s_0(\x) = s(\x), s_1(\x) = s(Rot(\x; \phi, d', \x_0))$ \\
    \bottomrule
    \end{tabular}
    \caption{Simulation settings for each of the three distributional settings studied. Note that $\sigma^2_0 = \sigma^2_1 = 0.1$ for all settings. Scoring function $s(\x)$ is defined in Eq.~\ref{eq:down_staircase}.}
    \label{tab:settings}
\end{table}

\paragraph{Simulating Distributional Differences.} To induce differences in the marginal and conditional risk distributions for each patient subgroup, we vary the simulation settings following Table~\ref{tab:settings}. We define $Rot$ as
\begin{equation}
    Rot(\x; \phi, d', \x_0) =   \begin{tikzpicture}[decoration={brace,amplitude=10pt},baseline=(current bounding box.west)]
     \matrix (magic) [matrix of math nodes,left delimiter=(,right delimiter=)] {
      R(-\phi) \\
      \vdots & \ddots \\
      0 & \dots & R(-\phi) & 0  \\
      0 & \dots & 0 & \mathbf{I}_{10-d'}\\
     };
     \draw[decorate] (magic-1-1.north) -- (magic-3-3.east) node[above=8pt,midway,sloped] {$d'/2$ times};
   \end{tikzpicture}(\x - \x_0) + \x_0,
\end{equation}
where the function $r$ applies a $2\times 2$ rotation matrix $R(-\phi)$ about $\x_0 = 0.4 \cdot \mathbf{1}$ to any number of pairs of dimensions (rotating the decision boundary by $\phi$ about the point $\x_0$ in orthogonal 2D subspace(s)). 
Setting $\phi, d' \neq 0$ satisfies $P_0(y \mid \x) \distneq P_1(y \mid \x)$.
Note that $Rot$ breaks the parallelism between the censorship and decision boundaries.

\section{Evaluation Metrics}
\label{appdx:eval}

We provide further intuition for our usage of within-group AUC and xAUC. Recall that the AUC is the probability that a randomly chosen positive example, $x_i$, has a greater risk score, than a randomly chosen negative example, $x_j$ (\textit{i.e.}, $P(s(x_i)>s(x_j))$). The within-group AUC for group $a$, written as $\text{AUC}_a$, is thus the AUC considering only examples in group $a$ (Figure~\ref{fig:auc}, cyan). The xAUC has a similar interpretation: $\text{xAUC}_{a, a'}$ is the probability that a randomly chosen positive example from group $a$ is scored above a randomly chosen negative example from group $a'$ (Figure~\ref{fig:auc}, magenta). We defer to~\cite{kallus2019fairness} for details on xAUC.

 This yields a decomposition of the overall AUC in terms of within-group AUC and xAUC (Figure~\ref{fig:auc}). Let $p_y(a) = P(A = a\mid Y = y)$ for $a \in \{0, 1\}$, $y \in \{0, 1\}$. By the law of total probability, and probabilistic definitions of AUC and xAUC, we have: 
\begin{align}
    \text{Overall AUC} &= p_0(0) \cdot p_1(0) \cdot \text{AUC}_0 + p_0(1) \cdot p_1(1) \cdot \text{AUC}_1 \nonumber  \\
    &\quad + p_0(0) \cdot p_1(1) \cdot \text{xAUC}_{1,0} + p_0(1) \cdot p_1(0) \cdot \text{xAUC}_{0,1}.
    \label{eq:auc_decomp}
\end{align}

Hence, a perfectly separable problem guarantees that an AUC of 1 is possible. When an AUC of 1 is achieved, within-group AUC and xAUC for all groups must also be 1 for Eq.~\ref{eq:auc_decomp} to hold.

\section{Model Details}
\label{appdx:models}

We provide all settings used for model training here.%

\subsection{Model Training}

We train probabilistic kernel support vector machines (SVMs)~\citep{platt1999probabilistic}, but any non-linear model suffices.
For each setting of the simulation (\textit{i.e.}, unique combination of simulation parameters), we train two SVMs with each model using one set of labels $y, \tilde{y}$ on the same 100 realizations of the simulation with 2,000 training data points and evaluate all models on a simulated sample of 20,000 test samples. As preprocessing, we apply one-hot encoding to $\x$ after discretization. 

\subsection{Model Hyperparameters}

As the focus of this paper is on evaluation, we keep all default parameters for the SVM; that is, regularization weight $C = 1$ and $\gamma = (d \cdot Var(\text{vec}(\x)))^{-1}$ where $\x \in \R^{n \times d}$ for the radial basis function kernel, where $\text{vec}$ is the matrix vectorization operator (\textit{i.e.} \texttt{auto} setting in \texttt{scikit-learn}).

\subsection{Software}

We use \texttt{scikit-learn} for the SVM implementation, which is built on LIBSVM~\citep{chang2011libsvm}.

\section{Full Results}
\label{appdx:all_results}

For MIMIC results, we report the test names as well as the results of the hypothesis test.
For the simulation study, we report the raw AUC and xAUC values by group with empirical 95\% confidence intervals for all experiments.

\subsection{Disparate Censorship in MIMIC-IV}

\paragraph{Dataset description.} We restrict the analysis to all hospital admissions involving a White ($n=337630$) or Black/African-American patients ($n=80293$; total: $n=417923$). These categories were chosen as they represented the two most frequently-appearing racial/ethnic categories in the dataset. We selected the following set of common laboratory/diagnostic tests to investigate for disparate censorship: complete blood counts (CBC), with and without differential (CBC w/ diff.), base metabolic panels (BMP), Troponin T tests, D-dimer tests, arterial blood gas (ABG) tests, blood culture orders (for any organism), brain natriuretic peptide (BNP) tests, and chest X-ray (CXR) orders.\footnote{For CXR only, since the MIMIC-CXR data is sourced from 2011-16, we limit ourselves to hospital admissions in that timeframe, yielding 122860 White patient admissions and 25968 Black/African-American admissions (total: 148828).} To obtain results for ICU and ED admissions, we cross-referenced hospital admission identifiers for ICU and ED stays from the relevant tables to obtain the relevant subset of patients. %

CBC, CBC w/ diff., BMP, Troponin T, D-dimer, ABG, and BNP test results were directly available from the publicly available MIMIC concept SQL queries. Additionally, for CBC w/ diff., we excluded rows that did not contain any non-null values in the following columns: \texttt{"basophils\_abs"}, \texttt{"eosinophils\_abs"}, \texttt{"lymphocytes\_abs"}, \texttt{"monocytes\_abs"}, \linebreak \texttt{"neutrophils\_abs"}, \texttt{"basophils"}, \texttt{"eosinophils"}, \texttt{"lymphocytes"}, \texttt{"monocytes"}, \linebreak
\texttt{"neutrophils"}, \texttt{"atypical\_lymphocytes"}, \texttt{"bands"}, \texttt{"immature\_granulocytes"}, \linebreak \texttt{"metamyelocytes"}, and \texttt{"nrbc"}. For the BMP, we excluded rows that did not contain any non-null values in the following columns: \texttt{"bicarbonate"}, \texttt{"bun"} (blood urea nitrogen), \texttt{"calcium"}, \texttt{"chloride"}, \texttt{"creatinine"}, \texttt{"glucose"}, \texttt{"sodium"}, and \texttt{"potassium"}. For CXR, we extracted information using the publicly-available MIMIC-CXR processing code at \url{https://github.com/MIT-LCP/mimic-cxr/blob/master/dcm/create-mimic-cxr-jpg-metadata.ipynb}. 

We then extracted the testing rates ($P(T)$, \% of admissions featuring at least one instance of the relevant test order) in each patient group (White vs. Black/African-American) and applied a two-sided $z$-test for equality of proportions. The null hypothesis is that testing rates are equal between groups (\emph{i.e.}, the test in question does not exhibit disparate censorship in MIMIC-IV). Specifically, this tests the hypothesis that White and Black/African-American patients were equally likely to receive a particular lab test order at any point(s) during each admission. We use a 1\% significance threshold with Bonferroni correction (9 tests total; $\alpha= 1.1 \times 10^{-3}$). All $p$-values below $10^{-4}$ are reported as ``$<10^{-4}$."  In summary, significant disparate censorship with respect to race was identified in all tests examined except for Troponin T and d-dimer tests.

\begin{table}[]
    \centering
    \begin{tabular}{c | c c | c c}
    \toprule
        \textbf{Test name} & \textbf{$P(T)$, White} & \textbf{$P(T)$, Black} & $z$ & $p$\\
    \midrule
         \multirow{1}{*}{\textbf{CBC}} & 73.71 & 68.20 & 30.46 & $<10^{-4}$ \\
         \multirow{1}{*}{\textbf{CBC w/ diff.}} & 31.67 & 28.81 & 16.01 & $<10^{-4}$ \\
         \multirow{1}{*}{\textbf{BMP}} & 71.26 & 63.72 & 40.42 & $<10^{-4}$ \\
         \multirow{1}{*}{\textbf{Blood cultures}} & 15.20 & 13.01 & 16.36 & $<10^{-4}$\\
         \multirow{1}{*}{\textbf{CXR}} & 27.61 & 26.57 & 3.43 & $6.0 \times 10^{-4}$\\
         \multirow{1}{*}{\textbf{ABG}} & 13.75 & 10.42 & 27.10 & $<10^{-4}$\\
         \multirow{1}{*}{\textbf{Troponin T}} & 8.72 & 8.58 & 1.29 & 0.20  \\
         \multirow{1}{*}{\textbf{BNP}} & 3.82 & 3.48 & 4.74 & $<10^{-4}$\\
         \multirow{1}{*}{\textbf{D-dimer}} & 0.21 & 0.25 & -1.83 & 0.07\\
    \bottomrule
    \end{tabular}
    \caption{Disparate censorship in common laboratory/diagnostic tests in White vs. Black/African-American patients, MIMIC-IV v1.0, with testing rates by group, $z$-statistics, and $p$-values.}
    \label{tab:mimic_dc}
\end{table}

\begin{sidewaystable}[t]
    \centering
    \small
    \begin{tabular}{c|c| c c c c c c}
    \toprule
         \multicolumn{2}{c|}{} & \multicolumn{6}{c}{$\tau_1$ (noise rate, a=1)}  \\
    \midrule
         $\tau_0$ (noise rate, $a=0$)  & \textbf{Group} & \textbf{5.0} (0.0)& \textbf{5.4} (8.6) & \textbf{5.8} (14.5) & \textbf{6.2} (37.7)& \textbf{6.6} (46.2) & \textbf{7.0} (68.3) \\
    \midrule
         \multirow{3}{*}{\textbf{5.0} (0.0)} & a=0 & 98.3 (98.1, 98.5) & 98.2 (98.0, 98.4) & 98.0 (97.7, 98.3) & 96.4 (95.8, 96.8) &  95.1 (94.0, 95.7) & 85.0 (83.4, 87.0)  \\
         & a=1 & 99.0 (98.8, 99.1) & 98.9 (98.8, 99.1) &  98.8 (98.6, 99.0) & 97.7 (97.3, 98.1) & 96.8 (96.2, 97.3) & 90.0 (88.7, 91.6)\\
         & \cellcolor{gray!50} $\Delta \text{AUC}$ & \cellcolor{gray!50} 0.7 (0.5, 0.9) & \cellcolor{gray!50}  0.7 (0.5, 0.9) & \cellcolor{gray!50} 0.8 (0.6, 1.0) & \cellcolor{gray!50} 1.4 (1.0, 1.8) & \cellcolor{gray!50} 1.8 (1.3, 2.3) & \cellcolor{gray!50} 5.0 (4.0, 6.2) \\
    \midrule
         \multirow{3}{*}{\textbf{5.4} (11.8)} & a=0 & 98.2 (98.0, 98.4) & 98.4 (98.1, 98.5) & 98.3 (98.0, 98.5) & 97.3 (96.9, 97.7) & 96.6 (95.9, 97.1) & 90.7 (89.0, 92.1) \\
         & a=1 & 98.9 (98.8, 99.1) & 99.0 (98.9, 99.1) & 99.0 (98.8, 99.1) & 98.4 (98.1, 98.6) & 97.9 (97.5, 98.2) & 94.0 (92.9, 95.2) \\
         & \cellcolor{gray!50} $\Delta \text{AUC}$ & \cellcolor{gray!50} 0.7 (0.6, 0.9) & \cellcolor{gray!50} 0.7 (0.5, 0.8) & \cellcolor{gray!50} 0.7 (0.5, 0.9) & \cellcolor{gray!50} 1.0 (0.8, 1.3) &  \cellcolor{gray!50} 1.3 (1.0, 1.7) & 3.5 (2.7, 4.2) \cellcolor{gray!50}\\
    \midrule
         \multirow{3}{*}{\textbf{5.8} (16.2)} & a=0 & 98.1 (97.8, 98.3) & 98.3 (98.1, 98.5) & 98.3 (98.0, 98.5) & 97.6 (97.2, 97.9) & 97.1 (96.5, 97.6) & 92.7 (90.8, 93.9) \\
         & a=1 & 98.9 (98.7, 99.0) & 99.0 (98.8, 99.1) & 99.0 (98.8, 99.1) & 98.5 (98.3, 98.7) & 98.2 (97.8, 98.5) & 95.4 (94.1, 96.3) \\
        & \cellcolor{gray!50} $\Delta \text{AUC}$ & \cellcolor{gray!50} 0.8 (0.6, 0.9) & \cellcolor{gray!50} 0.7 (0.5, 0.9) & \cellcolor{gray!50} 0.7 (0.5, 0.8) & \cellcolor{gray!50} 0.9 (0.7, 1.2) &  \cellcolor{gray!50} 1.2 (0.8, 1.5) & \cellcolor{gray!50} 2.7 (2.1, 3.6)\\
    \midrule    
         \multirow{3}{*}{\textbf{6.2} (23.6)} & a=0 & 97.6 (97.3, 97.9) & 98.0 (97.6, 98.2) & 98.0 (97.7, 98.3) & 97.7 (97.3, 98.0) & 97.4 (96.9, 97.9) & 95.0 (93.6, 96.2) \\
         & a=1 & 98.6 (98.3, 98.8) & 98.8 (98.6, 98.9) & 98.8 (98.6, 99.0) &  98.6 (98.3, 98.8) &  98.4 (98.1, 98.7) & 96.9 (95.9, 97.7) \\
         & \cellcolor{gray!50} $\Delta \text{AUC}$ & \cellcolor{gray!50} 0.9 (0.7, 1.2) & \cellcolor{gray!50} 0.8 (0.6, 1.0) & \cellcolor{gray!50} 0.8 (0.6, 1.0) & \cellcolor{gray!50} 0.9 (0.7, 1.2) &  \cellcolor{gray!50} 1.0 (0.8, 1.3) & \cellcolor{gray!50} 1.9 (1.3, 2.5)\\
    \midrule         
         \multirow{3}{*}{\textbf{6.6} (24.6)} & a=0 & 97.5 (97.1, 97.9) & 97.9 (97.5, 98.2) & 98.0 (97.6, 98.2) & 97.7 (97.2, 98.0) & 97.4 (96.9, 97.8) & 95.1 (93.6, 96.4) \\
         & a=1 & 98.5 (98.2, 98.7) & 98.7 (98.5, 98.9) & 98.8 (98.5, 99.0) & 98.6 (98.3, 98.8) & 98.4 (98.1, 98.7) & 97.0 (96.1, 97.7) \\
         & \cellcolor{gray!50} $\Delta \text{AUC}$ & \cellcolor{gray!50} 1.0 (0.8, 1.2) & \cellcolor{gray!50} 0.8 (0.6, 1.1) & \cellcolor{gray!50} 0.8 (0.6, 1.0) & \cellcolor{gray!50} 0.9 (0.7, 1.2) &  \cellcolor{gray!50} 1.0 (0.8, 1.3) & \cellcolor{gray!50} 1.9 (1.4, 2.7)\\
    \midrule     
         \multirow{3}{*}{\textbf{7.0} (26.1)} & a=0 &  97.4 (97.0, 97.8) & 97.8 (97.4, 98.1) & 97.9 (97.5, 98.1) &  97.6 (97.1, 97.9)  & 97.3 (96.7, 97.7) &  95.2 (93.3, 96.2) \\
         & a=1 & 98.4 (98.1, 98.6) & 98.7 (98.4, 98.8) & 98.7 (98.4, 98.9) & 98.5 (98.3, 98.8) & 98.4 (98.0, 98.6) & 97.0 (95.9, 97.6)\\
         & \cellcolor{gray!50} $\Delta \text{AUC}$ & \cellcolor{gray!50} 1.0 (0.8, 1.2) & \cellcolor{gray!50} 0.9 (0.7, 1.1) & \cellcolor{gray!50} 0.8 (0.7, 1.1)  & \cellcolor{gray!50} 1.0 (0.7, 1.3) &  \cellcolor{gray!50} 1.1 (0.8, 1.3) & \cellcolor{gray!50} 1.9 (1.4, 2.6)\\
    \midrule      
    \end{tabular}
    \caption{Group-wise AUC and $\Delta \text{AUC}$ under marginal distributional distance for all values of $\tau_0, \tau_1$.}
    \label{tab:auc_setting2}
\end{sidewaystable}

\begin{sidewaystable}[t]
    \centering
    \small
    \begin{tabular}{c|c| c c c c c c}
    \toprule
         \multicolumn{2}{c|}{} & \multicolumn{6}{c}{$\tau_1$ (noise rate, $a=1$)}  \\
    \midrule
         $\tau_0$ (noise rate, $a=0$)  & \textbf{Group} & \textbf{5.0} (0.0)& \textbf{5.4} (8.6) & \textbf{5.8} (14.5) & \textbf{6.2} (37.7)& \textbf{6.6} (46.2) & \textbf{7.0} (68.3) \\
    \midrule
         \multirow{3}{*}{\textbf{5.0} (0.0)} & a=0 & 99.6 (99.5, 99.7) & 99.6 (99.5, 99.6) & 99.5 (99.4, 99.6) &  99.0 (98.8, 99.1) & 98.6 (98.3, 98.8) & 93.8 (92.8, 95.0) \\
         & a=1 & 95.7 (95.0, 96.1) & 95.5 (95.0, 95.9) & 95.0 (94.5, 95.7) &  91.5 (90.4, 92.6) & 89.5 (87.8, 90.7) &  77.0 (75.1, 79.6) \\
         & \cellcolor{gray!50} $\Delta \text{xAUC}$ & \cellcolor{gray!50} 3.9 (3.6, 4.6) & \cellcolor{gray!50} 4.1 (3.7, 4.6) & \cellcolor{gray!50} 4.5 (3.9, 5.0) & \cellcolor{gray!50} 7.5 (6.4, 8.5) & \cellcolor{gray!50}  9.2 (8.0, 10.6) & \cellcolor{gray!50} 16.6 (15.3, 18.2)\\
    \midrule
         \multirow{3}{*}{\textbf{5.4} (11.8)} & a=0 & 99.6 (99.5, 99.6) & 99.6 (99.6, 99.7) & 99.6 (99.6, 99.7) & 99.4 (99.2, 99.4) &99.1 (98.9, 99.3)  & 96.8 (96.1, 97.5) \\
         & a=1 & 95.5 (94.8, 95.9) & 95.7 (95.1, 96.1) & 95.5 (94.9, 96.1) & 93.5 (92.6, 94.4) & 92.1 (90.8, 93.1) & 83.4 (81.0, 85.7) \\
         & \cellcolor{gray!50} $\Delta \text{xAUC}$ & \cellcolor{gray!50} 4.1 (3.7, 4.7) & \cellcolor{gray!50} 3.9 (3.5, 4.5) & \cellcolor{gray!50} 4.1 (3.6, 4.6) & \cellcolor{gray!50} 5.9 (5.1, 6.7) & \cellcolor{gray!50} 7.0 (6.1, 8.2) & \cellcolor{gray!50} 13.4 (11.7, 15.2)\\
    \midrule
         \multirow{3}{*}{\textbf{5.8} (16.2)} & a=0 & 99.6 (99.5, 99.6) & 99.6 (99.6, 99.7) & 99.6 (99.5, 99.7) & 99.4 (99.3, 99.5) & 99.3 (99.1, 99.4) & 97.7 (97.0, 98.2) \\
         & a=1 & 95.1 (94.5, 95.8) & 95.5 (95.0, 96.1) & 95.5 (94.9, 96.0) & 94.0 (93.1, 94.8) & 92.9 (91.6, 93.9) & 85.9 (83.2, 87.7) \\
         & \cellcolor{gray!50} $\Delta \text{xAUC}$  & \cellcolor{gray!50} 4.4 (3.8, 5.0) & \cellcolor{gray!50} 4.1 (3.5, 4.6) & \cellcolor{gray!50} 4.1 (3.6, 4.7) & \cellcolor{gray!50} 5.4 (4.7, 6.2) & \cellcolor{gray!50} 6.4 (5.5, 7.5) & \cellcolor{gray!50} 11.8 (10.3, 13.9)\\
    \midrule    
         \multirow{3}{*}{\textbf{6.2} (23.6)} & a=0 & 99.4 (99.3, 99.5) & 99.5 (99.4, 99.6) & 99.6 (99.5, 99.6)& 99.5 (99.3, 99.6) & 99.4 (99.2, 99.5) & 98.6 (98.1, 99.0) \\
         & a=1 & 94.2 (93.1, 94.9) & 94.8 (94.0, 95.5) & 95.0 (94.2, 95.5) & 94.2 (93.3, 95.0) & 93.6 (92.3, 94.6) & 89.1 (87.0, 91.4) \\
         & \cellcolor{gray!50} $\Delta \text{xAUC}$  & \cellcolor{gray!50} 5.3 (4.6, 6.3) & \cellcolor{gray!50} 4.7 (4.0, 5.4) & \cellcolor{gray!50} 4.5 (4.1, 5.3) & \cellcolor{gray!50} 5.3 (4.5, 6.1) & \cellcolor{gray!50} 5.8 (4.9, 6.9) & \cellcolor{gray!50} 9.5 (7.6, 11.3)\\
    \midrule         
         \multirow{3}{*}{\textbf{6.6} (24.6)} & a=0 & 99.4 (99.3, 99.5) & 99.5 (99.4, 99.6) & 99.5 (99.4, 99.6) & 99.5 (99.3, 99.5)  & 99.4 (99.2, 99.5) & 98.7 (98.2, 99.0) \\
         & a=1 & 94.0 (92.8, 94.7) &  94.7 (93.8, 95.3) & 94.8 (93.9, 95.4) & 94.0 (93.2, 94.9) & 93.5 (92.2, 94.5) & 89.3 (86.8, 91.6) \\
         & \cellcolor{gray!50} $\Delta \text{xAUC}$  & \cellcolor{gray!50} 5.4 (4.7, 6.5) & \cellcolor{gray!50} 4.9 (4.2, 5.7) & \cellcolor{gray!50} 4.7 (4.2, 5.5) & \cellcolor{gray!50} 5.4 (4.6, 6.2) &  \cellcolor{gray!50} 5.9 (5.0, 6.9) & \cellcolor{gray!50} 9.3 (7.4, 11.5)\\
    \midrule     
         \multirow{3}{*}{\textbf{7.0} (26.1)} & a=0 & 99.4 (99.2, 99.5) & 99.5 (99.4, 99.5) & 99.5 (99.4, 99.6) & 99.4 (99.3, 99.5) & 99.3 (99.2, 99.5) & 98.7 (98.2, 99.0) \\
         & a=1 & 93.7 (92.5, 94.5) & 94.4 (93.5, 95.1) & 94.7 (93.7, 95.2) & 93.8 (92.9, 94.6) & 93.3 (92.0, 94.3) & 89.4 (86.4, 91.1) \\
         & \cellcolor{gray!50} $\Delta \text{xAUC}$  & \cellcolor{gray!50} 5.7 (5.0, 6.8) & \cellcolor{gray!50} 5.0 (4.4, 5.9) & \cellcolor{gray!50} 4.8 (4.4, 5.7) & \cellcolor{gray!50} 5.6 (4.8, 6.4) &  \cellcolor{gray!50} 6.1 (5.2, 7.2) & \cellcolor{gray!50} 9.3 (7.8, 11.8)\\
    \midrule      
    \end{tabular}
    \caption{xAUC metrics and $\Delta \text{xAUC}$ under marginal distributional distance for all values of $\tau_0, \tau_1$.}
    \label{tab:xauc_setting2}
\end{sidewaystable}

\subsection{Setting 2: Difference in marginal risk distributions}

We provide full results with empirical 95\% confidence intervals for $\Delta \text{AUC}$ (Table~\ref{tab:auc_setting2}) and $\Delta \text{xAUC}$ (Table~\ref{tab:xauc_setting2}) under marginal distributional differences. We report results for the full cross-product of $\tau_0, \tau_1 \in \{5, 5.4, 5.8, 6.2, 6.6, 7\}$ with their associated group noise rates.

\subsection{Setting 3: Difference in conditional risk distributions}

We provide full results with empirical 95\% confidence intervals for $\Delta \text{AUC}$ and $\Delta \text{xAUC}$ under marginal distributional differences organized by $d'$, the number of dimensions rotated for group $a=1$ individuals. We plot heatmap visualizations of the performance gap as a function of noise rate (\textit{i.e.} $\tau_1$) and conditional shift ($\phi$) across levels of $d'$. We index the figures with results for all choices of $d'$ below:
\begin{itemize}
    \item $d'=2$: Figure~\ref{fig:cshift2}
    \item $d'=4$: Figure~\ref{fig:cshift4}
    \item $d'=6$: Figure~\ref{fig:cshift6}
    \item $d'=8$: Figure~\ref{fig:cshift8}
    \item $d'=10$: Figure~\ref{fig:cshift10}
\end{itemize}

We index the tables with all raw AUC and xAUC values below:

\begin{itemize}
    \item $d' = 2$: AUC (Table~\ref{tab:auc_setting3_d2}), xAUC (Table~\ref{tab:xauc_setting3_d2})
    \item $d' = 4$: AUC (Table~\ref{tab:auc_setting3_d4}), xAUC (Table~\ref{tab:xauc_setting3_d4})
    \item $d' = 6$: AUC (Table~\ref{tab:auc_setting3_d6}), xAUC (Table~\ref{tab:xauc_setting3_d6})
    \item $d' = 8$: AUC (Table~\ref{tab:auc_setting3_d8}), xAUC (Table~\ref{tab:xauc_setting3_d8})
    \item $d' = 10$: AUC (Table~\ref{tab:auc_setting3_d10}), xAUC (Table~\ref{tab:xauc_setting3_d10})
\end{itemize}

In general, performance gaps worsen as $d'$ or $\phi$ increase. Note the apparent duplicated AUC values at $d' = 10, \phi = 180$; as all 100 realizations of the simulation share the same set of random seeds across all experiments---since all dimensions are rotated, all points in group $a=1$ flip across the decision boundary, yielding the exact same censorship pattern.

\begin{figure}
    \centering
    \includegraphics[width=0.75\linewidth]{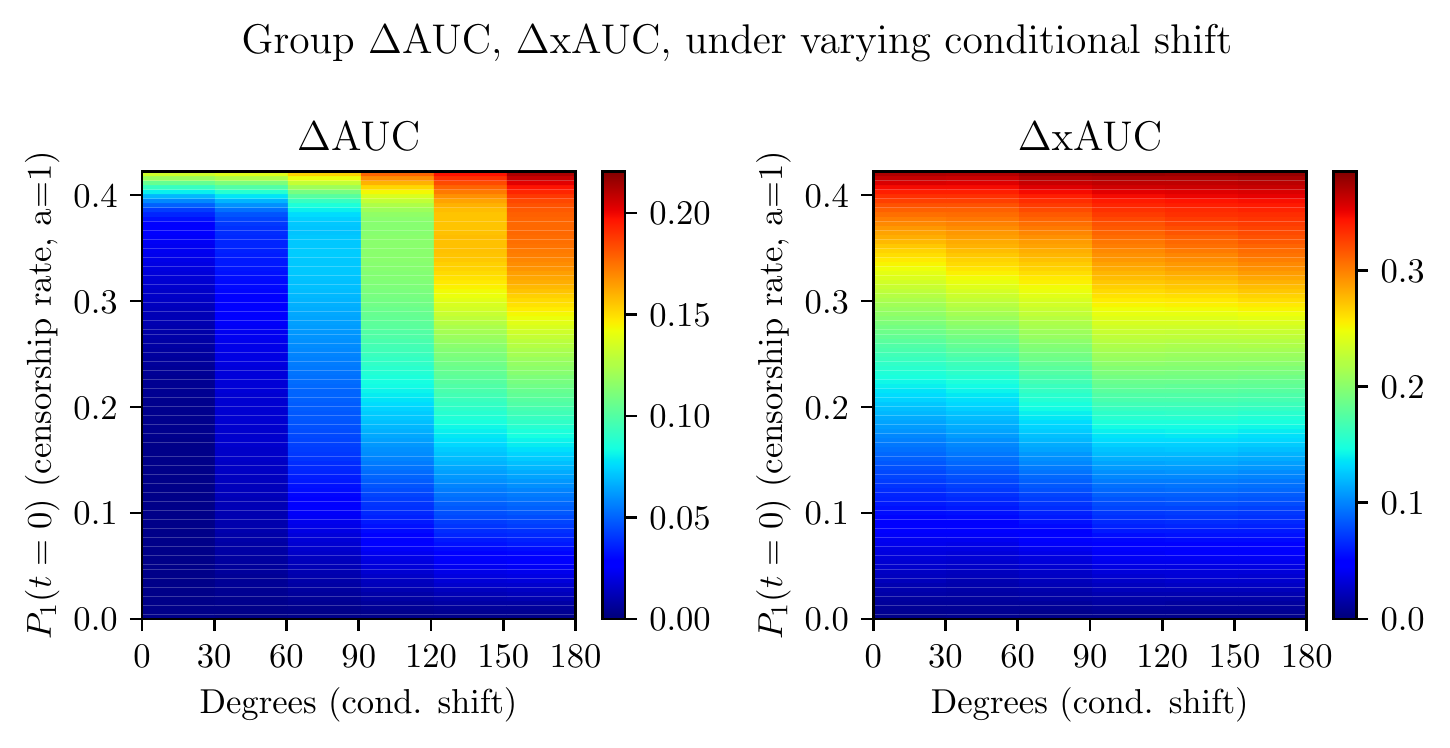}
    \caption{Heatmap showing median $\Delta \text{AUC}$ (left) and $\Delta \text{xAUC}$ (right) at varying levels of conditional shift ($x$-axis) and censorship rate in $a=1$ ($y$-axis); 2 dimensions rotated. Regions with smaller performance gap are in \textcolor{jet_min}{dark blue}, while larger performance gaps are in  \textcolor{jet_max}{dark red}.}
    \label{fig:cshift2}
\end{figure}

\begin{figure}
    \centering
    \includegraphics[width=0.75\linewidth]{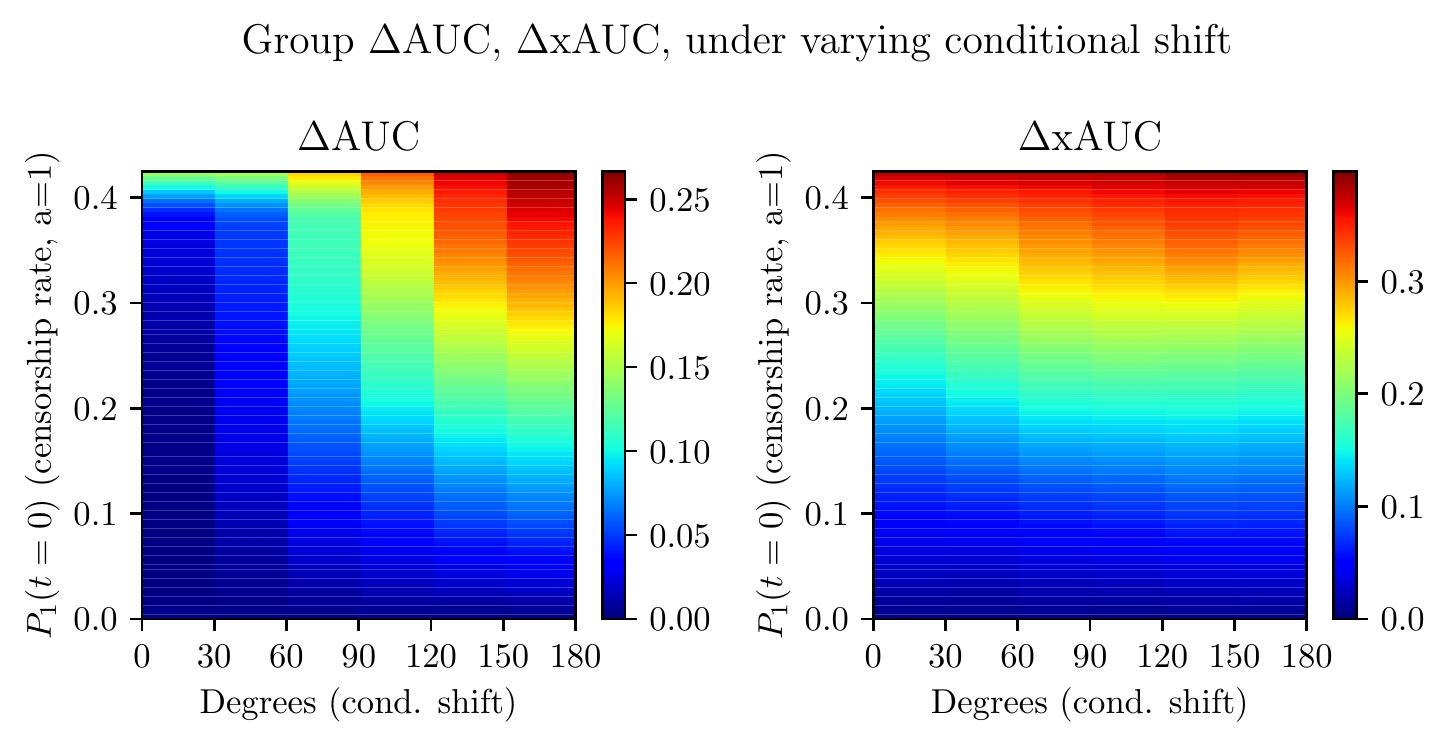}
    \caption{Heatmap showing median $\Delta \text{AUC}$ (left) and $\Delta \text{xAUC}$ (right) at varying levels of conditional shift ($\phi$; $x$-axis) and censorship rate in $a=1$ ($P_1(t=0)$, $y$-axis); 4 dimensions rotated. Regions with smaller performance gap are in \textcolor{jet_min}{dark blue}, while larger performance gaps are in  \textcolor{jet_max}{dark red}.}
    \label{fig:cshift4}
\end{figure}

\begin{figure}
    \centering
    \includegraphics[width=0.75\linewidth]{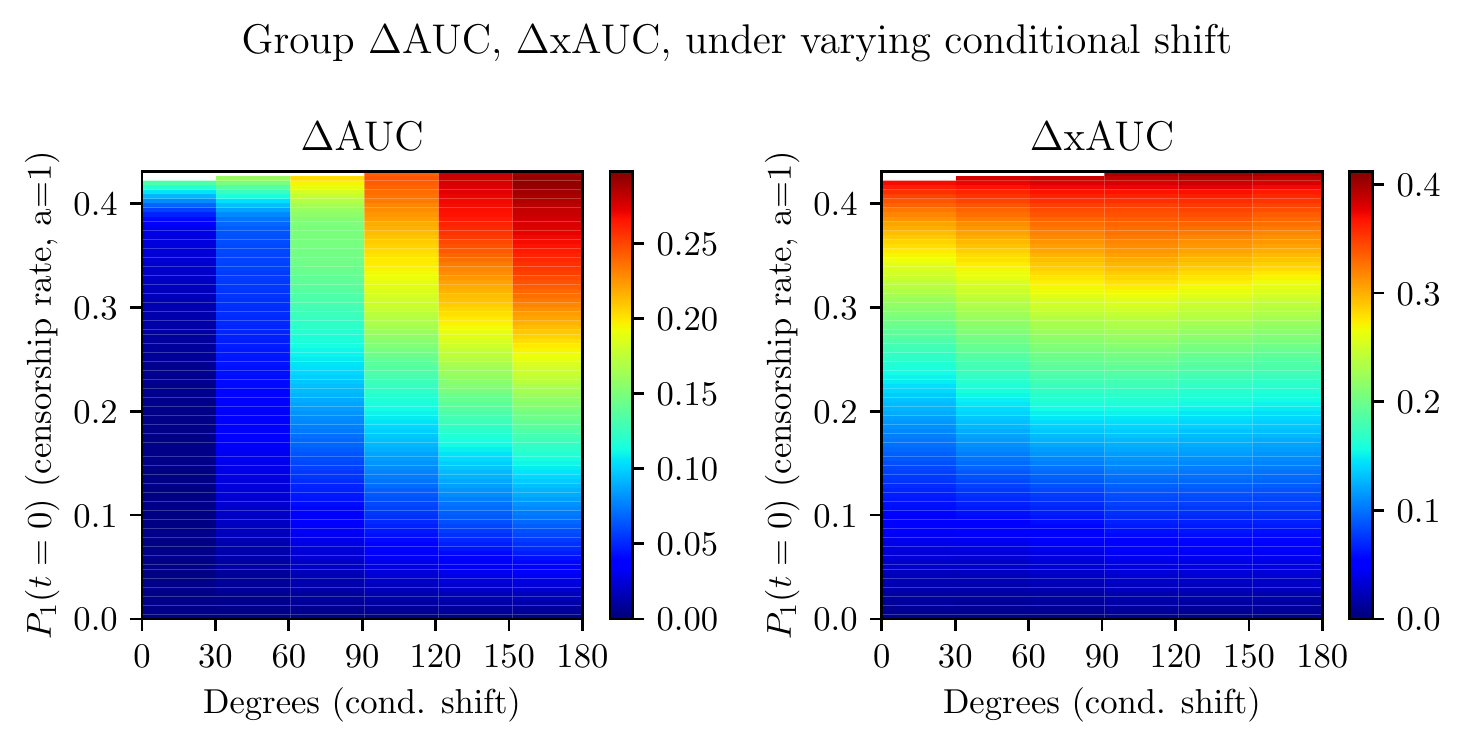}
    \caption{Heatmap showing median $\Delta \text{AUC}$ (left) and $\Delta \text{xAUC}$ (right) at varying levels of conditional shift ($\phi$; $x$-axis) and censorship rate in $a=1$ ($P_1(t=0)$, $y$-axis); 6 dimensions rotated. Regions with smaller performance gap are in \textcolor{jet_min}{dark blue}, while larger performance gaps are in  \textcolor{jet_max}{dark red}.}
    \label{fig:cshift6}
\end{figure}

\begin{figure}
    \centering
    \includegraphics[width=0.75\linewidth]{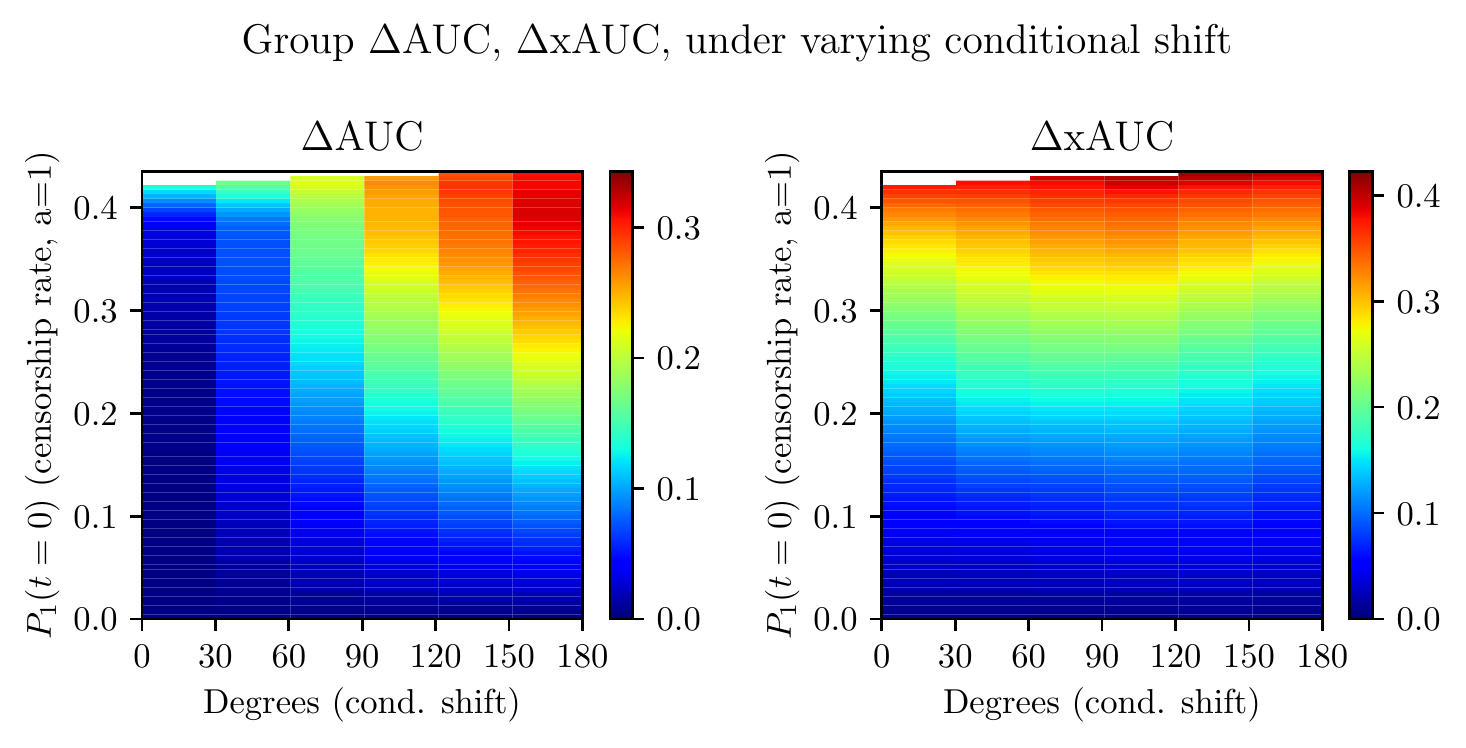}
    \caption{Heatmap showing median $\Delta \text{AUC}$ (left) and $\Delta \text{xAUC}$ (right) at varying levels of conditional shift ($\phi$; $x$-axis) and censorship rate in $a=1$ ($P_1(t=0)$, $y$-axis); 8 dimensions rotated. Regions with smaller performance gap are in \textcolor{jet_min}{dark blue}, while larger performance gaps are in  \textcolor{jet_max}{dark red}.}
    \label{fig:cshift8}
\end{figure}

\begin{figure}
    \centering
    \includegraphics[width=0.75\linewidth]{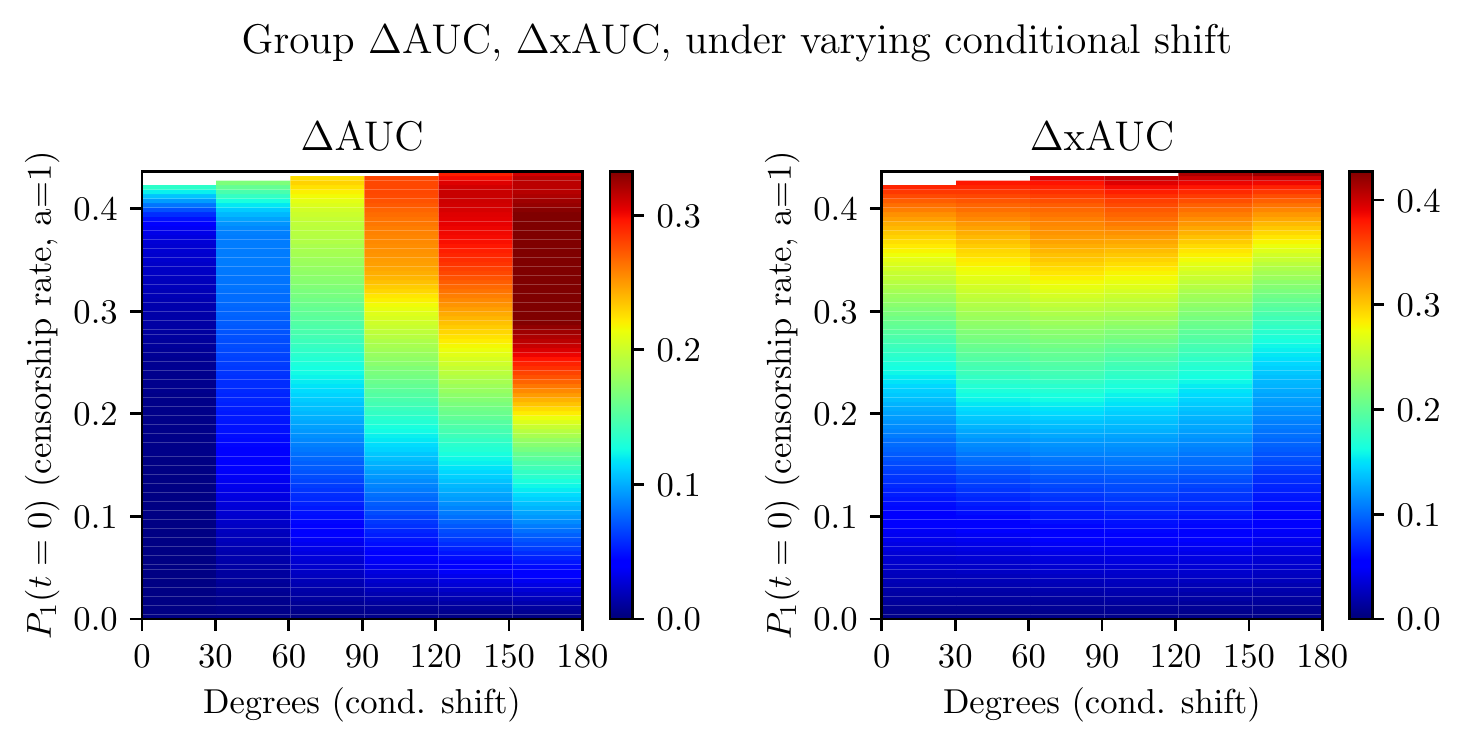}
    \caption{Heatmap showing median $\Delta \text{AUC}$ (left) and $\Delta \text{xAUC}$ (right) at varying levels of conditional shift ($\phi$; $x$-axis) and censorship rate in $a=1$ ($P_1(t=0)$, $y$-axis); 10 dimensions rotated. Regions with smaller performance gap are in \textcolor{jet_min}{dark blue}, while larger performance gaps are in  \textcolor{jet_max}{dark red}.}
    \label{fig:cshift10}
\end{figure}

\begin{sidewaystable}[t]
    \centering
    \small
    \begin{tabular}{c|c| c c c c c c}
    \toprule
         \multicolumn{2}{c|}{} & \multicolumn{6}{c}{$\tau_1$}  \\
    \midrule
         $\phi$ & \textbf{Group} & \textbf{5.0}& \textbf{5.4} & \textbf{5.8} & \textbf{6.2}& \textbf{6.6} & \textbf{7.0}  \\
    \midrule
         \multirow{3}{*}{0} & a=0 & 97.7 (97.4, 98.0) & 97.7 (97.4, 98.0) & 97.7 (97.4, 98.0) & 97.7 (97.4, 98.0) & 97.7 (97.4, 98.0) & 97.7 (97.4, 98.0) \\
         & a=1 & 97.6 (97.3, 97.9) & 97.5 (97.1, 97.8) & 97.3 (96.8, 97.6) & 97.7 (97.3, 98.1) & 93.7 (92.1, 94.9) & 83.4 (77.2, 86.9)\\
         & \cellcolor{gray!50} $\Delta \text{AUC}$ & \cellcolor{gray!50} 0.2 (0.0, 0.6) & \cellcolor{gray!50} 0.2 (0.0, 0.7) & \cellcolor{gray!50} 0.4 (0.0, 1.0) & \cellcolor{gray!50} 2.4 (1.6, 3.6) & \cellcolor{gray!50} 4.0 (2.9, 5.7) & \cellcolor{gray!50} 14.3 (10.8, 20.5) \\
    \midrule
         \multirow{3}{*}{30} & a=0 & 97.7 (97.4, 98.0) & 97.7 (97.4, 98.0) & 97.7 (97.4, 98.0) & 97.7 (97.4, 98.0) & 97.7 (97.4, 98.0) & 97.7 (97.4, 98.0) \\
         & a=1 & 97.2 (96.8, 97.7) & 96.3 (95.6, 96.9) & 96.0 (95.3, 96.7) & 94.1 (92.6, 95.0) & 92.7 (90.7, 93.8) & 83.0 (77.1, 87.3) \\
         & \cellcolor{gray!50} $\Delta \text{AUC}$ & \cellcolor{gray!50} 0.5 (0.0, 1.0) & \cellcolor{gray!50} 1.4 (0.8, 2.2) & \cellcolor{gray!50} 1.7 (1.0, 2.6) & \cellcolor{gray!50} 3.6 (2.6, 5.0) &  \cellcolor{gray!50} 5.0 (3.7, 7.1) & \cellcolor{gray!50} 14.7 (10.1, 20.5)\\
    \midrule
         \multirow{3}{*}{60} & a=0 & 97.7 (97.4, 98.0) & 97.7 (97.4, 98.0) & 97.7 (97.4, 98.0) & 97.7 (97.4, 98.0) & 97.7 (97.4, 98.0) & 97.7 (97.4, 98.0) \\
         & a=1 & 96.0 (95.2, 96.6) & 93.3 (92.3, 94.3) & 92.6 (91.2, 93.8) & 90.7 (89.1, 92.0) & 89.9 (87.5, 91.4) & 81.6 (76.8, 85.7) \\
        & \cellcolor{gray!50} $\Delta \text{AUC}$ & \cellcolor{gray!50} 1.7 (1.1, 2.5) & \cellcolor{gray!50} 4.3 (3.4, 5.5) & \cellcolor{gray!50} 5.1 (3.9, 6.5) & \cellcolor{gray!50} 7.0 (5.6, 8.7) & \cellcolor{gray!50} 7.8 (6.4, 10.2) & \cellcolor{gray!50} 16.0 (11.9, 20.8)\\
    \midrule    
         \multirow{3}{*}{90} & a=0 & 97.7 (97.4, 98.0) & 97.7 (97.4, 98.0) & 97.7 (97.4, 98.0) & 97.7 (97.4, 98.0) & 97.7 (97.4, 98.0) & 97.7 (97.4, 98.0) \\
         & a=1 & 93.9 (92.9, 94.8) & 90.6 (89.5, 92.0) & 88.9 (87.3, 90.5) &  86.5 (84.8, 88.1) & 86.0 (83.2, 88.2) & 79.5 (74.4, 83.6) \\
         & \cellcolor{gray!50} $\Delta \text{AUC}$ & \cellcolor{gray!50} 3.7 (2.9, 4.8) & \cellcolor{gray!50} 7.1 (5.9, 8.2) & \cellcolor{gray!50} 8.7 (7.2, 10.3) & \cellcolor{gray!50} 11.3 (9.5, 13.1) &  \cellcolor{gray!50} 11.7 (9.6, 14.4) & \cellcolor{gray!50} 18.3 (13.9, 23.1)\\
    \midrule         
         \multirow{3}{*}{120} & a=0 & 97.7 (97.4, 98.0) & 97.7 (97.4, 98.0) & 97.7 (97.4, 98.0) & 97.7 (97.4, 98.0) & 97.7 (97.4, 98.0) & 97.7 (97.4, 98.0) \\
         & a=1 & 92.1 (90.5, 93.2) & 88.2 (86.7, 89.4) & 86.4 (84.5, 88.2) & 82.2 (80.4, 84.5) & 82.1 (78.8, 84.4) & 77.6 (72.0, 81.8) \\
         & \cellcolor{gray!50} $\Delta \text{AUC}$ & \cellcolor{gray!50} 5.6 (4.5, 7.3) & \cellcolor{gray!50} 9.5 (8.1, 11.2) & \cellcolor{gray!50} 11.3 (9.5, 13.2) & \cellcolor{gray!50}  15.5 (13.1, 17.5) &  \cellcolor{gray!50} 15.6 (13.4, 19.1) & \cellcolor{gray!50} 20.2 (15.9, 25.7)\\
    \midrule     
         \multirow{3}{*}{150} & a=0 & 97.4 (97.0, 97.8) & 97.8 (97.4, 98.1) & 97.9 (97.5, 98.1) &  97.6 (97.1, 97.9)  & 97.3 (96.7, 97.7) &  95.2 (93.3, 96.2) \\
         & a=1 & 91.2 (89.7, 92.4) & 86.9 (85.3, 88.2) & 84.9 (83.4, 86.3) & 80.1 (77.9, 82.5) & 79.7 (76.7, 82.2) & 76.2 (70.5, 80.2)\\
         & \cellcolor{gray!50} $\Delta \text{AUC}$ & \cellcolor{gray!50} 6.6 (5.3, 8.1) & \cellcolor{gray!50} 10.8 (9.2, 12.3) & \cellcolor{gray!50} 12.8 (11.0, 14.3) & \cellcolor{gray!50} 17.5 (15.0, 19.9) &  \cellcolor{gray!50} 17.9 (15.1, 21.2) & \cellcolor{gray!50} 21.5 (17.4, 27.3)\\
    \midrule      
         \multirow{3}{*}{180} & a=0 & 97.4 (97.0, 97.8) & 97.8 (97.4, 98.1) & 97.9 (97.5, 98.1) & 97.6 (97.1, 97.9) & 97.3 (96.7, 97.7) &  95.2 (93.3, 96.2) \\
         & a=1 & 91.0 (89.5, 92.4) & 86.8 (85.0, 88.2) & 85.0 (83.0, 86.7) & 80.0 (77.6, 82.8) & 79.6 (76.6, 82.1) & 75.5 (70.0, 79.7)\\
         & \cellcolor{gray!50} $\Delta \text{AUC}$ & \cellcolor{gray!50} 6.7 (5.2, 8.1) & \cellcolor{gray!50} 11.0 (9.1, 12.8) & \cellcolor{gray!50} 12.7 (10.8, 14.4) & \cellcolor{gray!50} 17.6 (14.9, 20.2) &  \cellcolor{gray!50} 18.1 (15.3, 21.1) & \cellcolor{gray!50} 22.0 (18.0, 27.8)\\
    \midrule      
    \end{tabular}
    \caption{Group-wise AUC and $\Delta \text{AUC}$ under conditional distributional difference for all values of $\tau_1, \phi$, at $d' = 2$.}
    \label{tab:auc_setting3_d2}
\end{sidewaystable}

\begin{sidewaystable}[t]
    \centering
    \small
    \begin{tabular}{c|c| c c c c c c}
    \toprule
         \multicolumn{2}{c|}{} & \multicolumn{6}{c}{$\tau_1$}  \\
    \midrule
         $\phi$  & \textbf{Group} & \textbf{5.0}& \textbf{5.4} & \textbf{5.8} & \textbf{6.2}& \textbf{6.6} & \textbf{7.0}  \\
    \midrule
         \multirow{3}{*}{0} & a=0 & 97.6 (97.0, 98.1) & 90.9 (89.5, 92.4) & 85.8 (83.6, 88.3) & 72.3 (68.5, 77.0) & 69.0 (64.6, 73.6) & 97.6 (97.0, 98.1) \\
         & a=1 & 97.6 (97.2, 98.2) & 99.7 (99.5, 99.8) & 99.9 (99.8, 99.9) & 100.0 (99.9, 100.0) & 100.0 (99.9, 100.0) & 99.9 (99.8, 100.0)\\
         & \cellcolor{gray!50} $\Delta \text{xAUC}$ & \cellcolor{gray!50} 0.4 (0.1, 1.1) & \cellcolor{gray!50} 8.8 (7.1, 10.2) & \cellcolor{gray!50} 14.1 (11.6, 16.3) & \cellcolor{gray!50} 27.6 (23.0, 31.4) & \cellcolor{gray!50} 30.9 (26.3, 35.4) & \cellcolor{gray!50} 37.6 (32.0, 44.2) \\
    \midrule
         \multirow{3}{*}{30} & a=0 & 96.5 (95.5, 97.2) & 89.9 (88.1, 91.8) & 85.1 (82.6, 87.5) & 71.6 (66.3, 76.7) & 68.9 (63.2, 73.6) & 62.4 (56.6, 68.0) \\
         & a=1 & 98.2 (97.8, 98.6) & 99.5 (99.2, 99.6) & 99.8 (99.6, 99.9) & 99.9 (99.9, 100.0) & 100.0 (99.9, 100.0) & 99.9 (99.8, 100.0) \\
         & \cellcolor{gray!50} $\Delta \text{xAUC}$ & \cellcolor{gray!50} 1.8 (0.8, 2.9) & \cellcolor{gray!50} 9.5 (7.6, 11.4) & \cellcolor{gray!50} 14.6 (12.2, 17.2) & \cellcolor{gray!50} 28.3 (23.2, 33.7) &  \cellcolor{gray!50} 31.0 (26.3, 36.8) & \cellcolor{gray!50} 37.5 (31.9, 43.4)\\
    \midrule
         \multirow{3}{*}{60} & a=0 & 94.0 (92.8, 95.1) & 87.3 (85.1, 89.1) & 82.7 (80.1, 85.8) & 70.9 (66.5, 76.1) & 68.1 (62.4, 73.5) & 61.8 (56.5, 67.8) \\
         & a=1 & 98.6 (98.1, 98.9) & 99.2 (98.9, 99.4) & 99.5 (99.2, 99.7) & 99.9 (99.8, 99.9) & 99.9 (99.8, 100.0) & 99.9 (99.8, 100.0) \\
        & \cellcolor{gray!50} $\Delta \text{xAUC}$ & \cellcolor{gray!50} 4.5 (3.2, 5.9) & \cellcolor{gray!50} 11.9 (9.7, 14.2)  & \cellcolor{gray!50} 16.7 (13.5, 19.5) & \cellcolor{gray!50} 28.9 (23.8, 33.4)& \cellcolor{gray!50} 31.8 (26.4, 37.6) & \cellcolor{gray!50} 38.2 (32.0, 43.5)\\
    \midrule    
         \multirow{3}{*}{90} & a=0 & 91.4 (89.8, 93.1) & 83.7 (81.2, 87.3) & 79.6 (76.3, 83.5) & 69.4 (64.7, 75.2) & 67.0 (61.0, 73.1) & 61.8 (55.9, 68.6) \\
         & a=1 & 98.6 (98.1, 98.9) & 99.1 (98.8, 99.4) & 99.3 (99.0, 99.5) & 99.8 (99.6, 99.9) & 99.9 (99.7, 99.9) & 99.9 (99.8, 100.0)\\
         & \cellcolor{gray!50} $\Delta \text{xAUC}$ & \cellcolor{gray!50} 7.3 (5.4, 9.1) & \cellcolor{gray!50} 15.4 (11.6, 18.1) & \cellcolor{gray!50} 19.7 (15.5, 23.1) & \cellcolor{gray!50} 30.3 (24.5, 35.1) &  \cellcolor{gray!50} 32.8 (26.6, 38.9) & \cellcolor{gray!50} 38.1 (31.2, 44.1)\\
    \midrule         
         \multirow{3}{*}{120} & a=0 & 89.4 (87.7, 91.3) & 81.7 (78.8, 85.5) & 77.8 (74.1, 81.6) & 68.5 (62.8, 73.8) & 66.3 (61.1, 72.5) & 61.8 (56.3, 67.6) \\
         & a=1 & 98.6 (98.1, 98.9) & 99.1 (98.8, 99.3) & 99.3 (99.0, 99.5) & 99.7 (99.5, 99.8) & 99.8 (99.7, 99.9) & 99.9 (99.7, 100.0) \\
         & \cellcolor{gray!50} $\Delta \text{xAUC}$ & \cellcolor{gray!50} 9.2 (7.0, 10.9) & \cellcolor{gray!50} 17.5 (13.4, 20.4) & \cellcolor{gray!50} 21.4 (17.4, 25.2) & \cellcolor{gray!50} 31.2 (25.7, 37.0) &  \cellcolor{gray!50} 33.5 (27.2, 38.8) & \cellcolor{gray!50} 38.1 (32.2, 43.7)\\
    \midrule     
         \multirow{3}{*}{150} & a=0 & 88.3 (86.7, 90.6) & 80.9 (77.8, 84.4) & 77.1 (73.5, 81.4) &  67.6 (62.7, 73.9) & 66.1 (60.7, 72.1) &  61.7 (55.8, 67.3) \\
         & a=1 & 98.6 (98.2, 99.0) & 99.0 (98.6, 99.3) & 99.2 (98.9, 99.5) & 99.6 (99.4, 99.8) & 99.8 (99.6, 99.9) & 99.9 (99.7, 100.0) \\
         & \cellcolor{gray!50} $\Delta \text{xAUC}$ & \cellcolor{gray!50} 10.3 (7.8, 12.1) & \cellcolor{gray!50} 18.2 (14.3, 21.1) & \cellcolor{gray!50} 22.2 (17.5, 25.8) & \cellcolor{gray!50} 31.9 (25.6, 37.0) &  \cellcolor{gray!50} 33.7 (27.5, 39.1) & \cellcolor{gray!50} 38.2 (32.4, 44.1))\\
    \midrule      
         \multirow{3}{*}{180} & a=0 & 88.5 (86.5, 90.6) & 81.1 (78.3, 84.4) & 77.0 (72.9, 81.2) & 67.5 (63.2, 72.9) & 65.3 (60.4, 71.1) &  61.3 (56.0, 67.8) \\
         & a=1 & 98.6 (98.2, 98.9) & 99.0 (98.6, 99.2) & 99.2 (98.9, 99.4) & 99.6 (99.4, 99.8) & 99.8 (99.6, 99.9) & 99.9 (99.7, 100.0)\\
         & \cellcolor{gray!50} $\Delta \text{xAUC}$ & \cellcolor{gray!50} 10.1 (7.7, 12.1) & \cellcolor{gray!50} 17.8 (14.4, 20.5) & \cellcolor{gray!50} 22.3 (17.6, 26.4) & \cellcolor{gray!50} 32.2 (26.6, 36.6) &  \cellcolor{gray!50} 34.5 (28.3, 39.5) & \cellcolor{gray!50} 38.5 (32.0, 44.0)\\
    \midrule      
    \end{tabular}
    \caption{xAUC and $\Delta \text{xAUC}$ under conditional distributional difference for all values of $\tau_1, \phi$, at $d' = 2$.}
    \label{tab:xauc_setting3_d2}
\end{sidewaystable}

\begin{sidewaystable}[t]
    \centering
    \small
    \begin{tabular}{c|c| c c c c c c}
    \toprule
         \multicolumn{2}{c|}{} & \multicolumn{6}{c}{$\tau_1$}  \\
    \midrule
         $\phi$ & \textbf{Group} & \textbf{5.0}& \textbf{5.4} & \textbf{5.8} & \textbf{6.2}& \textbf{6.6} & \textbf{7.0}  \\
    \midrule
         \multirow{3}{*}{0} & a=0 & 97.7 (97.4, 98.0) & 97.7 (97.4, 98.0) & 97.7 (97.4, 98.0) & 97.7 (97.4, 98.0) & 97.7 (97.4, 98.0) & 97.7 (97.4, 98.0) \\
         & a=1 & 97.6 (97.3, 97.9) & 97.5 (97.1, 97.8) & 97.3 (96.8, 97.6) & 97.7 (97.3, 98.1) & 93.7 (92.1, 94.9) & 83.4 (77.2, 86.9)\\
         & \cellcolor{gray!50} $\Delta \text{AUC}$ & \cellcolor{gray!50} 0.2 (0.0, 0.6) & \cellcolor{gray!50} 0.2 (0.0, 0.7) & \cellcolor{gray!50} 0.4 (0.0, 1.0) & \cellcolor{gray!50} 2.4 (1.6, 3.6) & \cellcolor{gray!50} 4.0 (2.9, 5.7) & \cellcolor{gray!50} 14.3 (10.8, 20.5) \\
    \midrule
         \multirow{3}{*}{30} & a=0 & 97.7 (97.4, 98.0) & 97.7 (97.4, 98.0) & 97.7 (97.4, 98.0) & 97.7 (97.4, 98.0) & 97.7 (97.4, 98.0) & 97.7 (97.4, 98.0) \\
         & a=1 & 97.0 (96.5, 97.4) & 95.4 (94.4, 96.0) & 94.8 (93.9, 95.6) & 93.0 (91.9, 94.1) & 91.7 (90.0, 93.1)& 82.9 (77.2, 86.6) \\
         & \cellcolor{gray!50} $\Delta \text{AUC}$ & \cellcolor{gray!50} 0.7 (0.1, 1.3) & \cellcolor{gray!50} 2.3 (1.6, 3.3) & \cellcolor{gray!50} 2.9 (1.9, 3.8) & \cellcolor{gray!50} 4.7 (3.5, 5.8) &  \cellcolor{gray!50} 5.9 (4.5, 8.0) & \cellcolor{gray!50} 14.7 (11.2, 20.7)\\
    \midrule
         \multirow{3}{*}{60} & a=0 & 97.7 (97.4, 98.0) & 97.7 (97.4, 98.0) & 97.7 (97.4, 98.0) & 97.7 (97.4, 98.0) & 97.7 (97.4, 98.0) & 97.7 (97.4, 98.0) \\
         & a=1 & 94.5 (93.5, 95.3) & 91.2 (89.5, 92.2) & 89.5 (87.7, 91.0) & 86.6 (84.7, 88.3) & 85.8 (83.5, 87.7) & 79.6 (75.7, 83.8) \\
        & \cellcolor{gray!50} $\Delta \text{AUC}$ & \cellcolor{gray!50} 3.2 (2.3, 4.2) & \cellcolor{gray!50} 6.5 (5.4, 8.0) & \cellcolor{gray!50} 8.2 (6.7, 10.0) & \cellcolor{gray!50} 11.2 (9.4, 13.1) & \cellcolor{gray!50} 11.8 (9.9, 14.2) & \cellcolor{gray!50} 18.2 (14.1, 22.4)\\
    \midrule    
         \multirow{3}{*}{90} & a=0 & 97.7 (97.4, 98.0) & 97.7 (97.4, 98.0) & 97.7 (97.4, 98.0) & 97.7 (97.4, 98.0) & 97.7 (97.4, 98.0) & 97.7 (97.4, 98.0) \\
         & a=1 & 91.4 (90.1, 92.6) & 87.0 (85.5, 88.4) & 85.2 (83.2, 87.0) & 80.7 (78.7, 83.3) & 80.3 (77.7, 83.1) & 76.3 (71.9, 79.5) \\
         & \cellcolor{gray!50} $\Delta \text{AUC}$ & \cellcolor{gray!50}  6.2 (5.0, 7.8) & \cellcolor{gray!50} 10.6 (9.2, 12.2) & \cellcolor{gray!50} 12.5 (10.6, 14.6) & \cellcolor{gray!50} 17.0 (14.4, 19.0) &  \cellcolor{gray!50} 17.6 (14.5, 19.9) & \cellcolor{gray!50} 21.5 (18.1, 25.8)\\
    \midrule         
         \multirow{3}{*}{120} & a=0 & 97.7 (97.4, 98.0) & 97.7 (97.4, 98.0) & 97.7 (97.4, 98.0) & 97.7 (97.4, 98.0) & 97.7 (97.4, 98.0) & 97.7 (97.4, 98.0) \\
         & a=1 & 87.9 (85.9, 89.7) & 82.9 (81.0, 84.7) & 80.7 (78.4, 83.0) & 75.8 (73.6, 79.0) & 75.0 (72.3, 79.0) & 73.0 (68.8, 77.3) \\
         & \cellcolor{gray!50} $\Delta \text{AUC}$ & \cellcolor{gray!50} 9.8 (8.1, 11.6) & \cellcolor{gray!50} 14.7 (12.7, 16.9) & \cellcolor{gray!50} 16.9 (14.7, 19.3) & \cellcolor{gray!50}  21.9 (18.8, 24.2) &  \cellcolor{gray!50} 22.7 (18.6, 25.5) & \cellcolor{gray!50} 24.8 (20.4, 28.9)\\
    \midrule     
       \multirow{3}{*}{150} & a=0 & 97.7 (97.4, 98.0) & 97.7 (97.4, 98.0) & 97.7 (97.4, 98.0) & 97.7 (97.4, 98.0) & 97.7 (97.4, 98.0) & 97.7 (97.4, 98.0) \\
         & a=1 & 85.4 (82.9, 87.4) & 80.3 (77.6, 82.5) & 77.9 (75.4, 80.7) & 73.5 (70.7, 76.8) & 72.7 (69.9, 76.4) & 71.6 (66.1, 76.1)\\
         & \cellcolor{gray!50} $\Delta \text{AUC}$ & \cellcolor{gray!50} 12.3 (10.4, 14.7) & \cellcolor{gray!50} 17.4 (15.1, 20.2) & \cellcolor{gray!50} 19.8 (17.1, 22.0) & \cellcolor{gray!50} 24.3 (20.9, 27.2) &  \cellcolor{gray!50} 25.1 (21.2, 28.0) & \cellcolor{gray!50} 26.2 (21.7, 31.7)\\
    \midrule      
        \multirow{3}{*}{180} & a=0 & 97.7 (97.4, 98.0) & 97.7 (97.4, 98.0) & 97.7 (97.4, 98.0) & 97.7 (97.4, 98.0) & 97.7 (97.4, 98.0) & 97.7 (97.4, 98.0) \\
         & a=1 & 84.4 (82.4, 85.8) &  79.1 (76.7, 81.2) & 77.1 (75.1, 79.9) & 72.6 (69.2, 75.2) & 72.3 (69.1, 76.0) & 71.0 (66.0, 75.4)\\
         & \cellcolor{gray!50} $\Delta \text{AUC}$ & \cellcolor{gray!50} 13.4 (11.9, 15.4) & \cellcolor{gray!50} 18.5 (16.7, 21.2) & \cellcolor{gray!50} 20.6 (18.0, 22.6) & \cellcolor{gray!50} 25.1 (22.4, 28.4) &  \cellcolor{gray!50} 25.3 (21.7, 28.6) & \cellcolor{gray!50} 26.6 (22.4, 31.9)\\
    \midrule      
    \end{tabular}
    \caption{Group-wise AUC and $\Delta \text{AUC}$ under conditional distributional difference for all values of $\tau_1, \phi$, at $d' = 4$.}
    \label{tab:auc_setting3_d4}
\end{sidewaystable}

\begin{sidewaystable}[t]
    \centering
    \small
    \begin{tabular}{c|c| c c c c c c}
    \toprule
         \multicolumn{2}{c|}{} & \multicolumn{6}{c}{$\tau_1$}  \\
    \midrule
         $\phi$ & \textbf{Group} & \textbf{5.0}& \textbf{5.4} & \textbf{5.8} & \textbf{6.2}& \textbf{6.6} & \textbf{7.0}  \\
    \midrule
         \multirow{3}{*}{0} & a=0 & 97.6 (97.0, 98.1) & 90.9 (89.5, 92.4) & 85.8 (83.6, 88.3) & 72.3 (68.5, 77.0) & 69.0 (64.6, 73.6) & 62.3 (55.8, 67.7) \\
         & a=1 & 97.6 (97.2, 98.2) & 99.7 (99.5, 99.8) & 99.9 (99.8, 99.9) & 100.0 (99.9, 100.0) & 100.0 (99.9, 100.0) & 99.9 (99.8, 100.0)\\
         & \cellcolor{gray!50} $\Delta \text{xAUC}$ & \cellcolor{gray!50} 0.4 (0.1, 1.1) & \cellcolor{gray!50} 8.8 (7.1, 10.2) & \cellcolor{gray!50} 14.1 (11.6, 16.3) & \cellcolor{gray!50} 27.6 (23.0, 31.4) & \cellcolor{gray!50} 30.9 (26.3, 35.4) & \cellcolor{gray!50} 37.6 (32.0, 44.2) \\
    \midrule
         \multirow{3}{*}{30} & a=0 & 95.7 (94.8, 96.6) & 89.1 (87.0, 91.0) & 83.9 (81.6, 86.7) & 71.5 (67.0, 76.6) & 68.3 (63.7, 74.6) & 62.1 (56.4, 68.2) \\
         & a=1 & 98.5 (98.0, 98.8) & 99.4 (99.1, 99.5) & 99.7 (99.5, 99.8) & 99.9 (99.9, 100.0) & 99.9 (99.9, 100.0) & 99.9 (99.8, 100.0) \\
         & \cellcolor{gray!50} $\Delta \text{xAUC}$ & \cellcolor{gray!50} 2.7 (1.7, 3.7) & \cellcolor{gray!50} 10.2 (8.1, 12.4) & \cellcolor{gray!50} 15.7 (12.7, 18.1) & \cellcolor{gray!50} 28.4 (23.2, 33.0) &  \cellcolor{gray!50} 31.6 (25.3, 36.3) & \cellcolor{gray!50} 37.9 (31.6, 43.5)\\
    \midrule
         \multirow{3}{*}{60} & a=0 & 92.3 (90.9, 93.7) & 85.2 (82.3, 87.7) & 80.7 (77.7, 83.7) & 69.9 (64.3, 74.9) & 67.5 (62.2, 72.7) & 62.4 (56.7, 68.3) \\
         & a=1 & 98.6 (98.1, 98.9) & 99.1 (98.7, 99.3) &  99.3 (99.0, 99.5) & 99.9 (99.7, 99.9) & 99.8 (99.6, 99.9) & 99.9 (99.8, 100.0) \\
        & \cellcolor{gray!50} $\Delta \text{xAUC}$ & \cellcolor{gray!50} 6.3 (4.6, 7.9) & \cellcolor{gray!50} 14.1 (11.2, 16.9) & \cellcolor{gray!50} 18.6 (15.6, 21.6) & \cellcolor{gray!50} 29.9 (24.7, 35.6) & \cellcolor{gray!50} 32.3 (27.1, 37.7) & \cellcolor{gray!50} 37.6 (31.5, 43.3)\\
    \midrule    
         \multirow{3}{*}{90} & a=0 & 89.0 (86.9, 90.8) & 81.6 (79.4, 84.7) & 77.3 (73.7, 81.0) & 68.3 (63.5, 72.7) & 66.0 (61.0, 72.2) & 62.0 (56.1, 67.9) \\
         & a=1 & 98.6 (98.2, 99.0) & 99.0 (98.6, 99.2) & 99.2 (99.0, 99.4) & 99.7 (99.5, 99.8) & 99.8 (99.6, 99.9) & 99.9 (99.7, 100.0)\\
         & \cellcolor{gray!50} $\Delta \text{xAUC}$ & \cellcolor{gray!50} 9.7 (7.7, 11.8) & \cellcolor{gray!50} 17.4 (14.1, 19.6) & \cellcolor{gray!50} 22.0 (18.1, 25.6) & \cellcolor{gray!50} 31.4 (26.8, 36.3) &  \cellcolor{gray!50} 33.7 (27.5, 38.9) & \cellcolor{gray!50} 37.9 (31.9, 43.7)\\
    \midrule         
         \multirow{3}{*}{120} & a=0 & 97.7 (97.4, 98.0) & 97.7 (97.4, 98.0) & 97.7 (97.4, 98.0) & 97.7 (97.4, 98.0) & 97.7 (97.4, 98.0) & 97.7 (97.4, 98.0) \\
         & a=1 & 92.1 (90.5, 93.2) & 88.2 (86.7, 89.4) & 86.4 (84.5, 88.2) & 82.2 (80.4, 84.5) & 82.1 (78.8, 84.4) & 77.6 (72.0, 81.8) \\
         & \cellcolor{gray!50} $\Delta \text{xAUC}$ & \cellcolor{gray!50} 5.6 (4.5, 7.3) & \cellcolor{gray!50} 9.5 (8.1, 11.2) & \cellcolor{gray!50} 11.3 (9.5, 13.2) & \cellcolor{gray!50}  15.5 (13.1, 17.5) &  \cellcolor{gray!50} 15.6 (13.4, 19.1) & \cellcolor{gray!50} 20.2 (15.9, 25.7)\\
    \midrule     
         \multirow{3}{*}{150} & a=0 & 85.5 (83.2, 87.8) & 78.4 (75.1, 81.7) & 74.7 (70.5, 78.6) & 66.0 (62.3, 71.5) & 64.9 (59.8, 70.8) & 61.0 (56.0, 66.8) \\
         & a=1 & 98.6 (98.2, 98.9) & 99.0 (98.6, 99.2) & 99.2 (98.9, 99.4) & 99.6 (99.4, 99.8) & 99.7 (99.6, 99.9) & 99.9 (99.7, 100.0)\\
         & \cellcolor{gray!50} $\Delta \text{xAUC}$ & \cellcolor{gray!50} 13.3 (10.7, 15.5) & \cellcolor{gray!50} 20.6 (17.3, 24.0) & \cellcolor{gray!50} 24.6 (20.7, 28.7) & \cellcolor{gray!50} 33.6 (28.0, 37.5) &  \cellcolor{gray!50} 34.9 (28.8, 40.0) & \cellcolor{gray!50} 38.9 (33.0, 43.8)\\
    \midrule      
         \multirow{3}{*}{180} & a=0 & 83.2 (80.8, 85.9) & 76.3 (72.2, 80.2) & 72.7 (68.4, 76.7) & 65.4 (60.7, 70.9) & 64.2 (58.7, 69.4) & 60.1 (54.7, 65.6) \\
         & a=1 & 98.5 (98.0, 98.9) & 98.9 (98.5, 99.2) & 99.2 (98.8, 99.5) & 99.6 (99.4, 99.8) & 99.8 (99.6, 99.9) & 99.9 (99.8, 100.0)\\
         & \cellcolor{gray!50} $\Delta \text{xAUC}$ & \cellcolor{gray!50} 15.5 (12.2, 17.9) & \cellcolor{gray!50} 22.6 (18.7, 26.8) & \cellcolor{gray!50} 26.6 (22.2, 30.7) & \cellcolor{gray!50} 34.3 (28.7, 39.0) &  \cellcolor{gray!50} 35.6 (30.3, 41.1) & \cellcolor{gray!50} 39.8 (34.2, 45.3)\\
    \midrule      
    \end{tabular}
    \caption{Group-wise xAUC and $\Delta \text{xAUC}$ under conditional distributional difference for all values of $\tau_1, \phi$, at $d' = 4$.}
    \label{tab:xauc_setting3_d4}
\end{sidewaystable}

\begin{sidewaystable}[t]
    \centering
    \small
    \begin{tabular}{c|c| c c c c c c}
    \toprule
         \multicolumn{2}{c|}{} & \multicolumn{6}{c}{$\tau_1$}  \\
    \midrule
         $\phi$ & \textbf{Group} & \textbf{5.0}& \textbf{5.4} & \textbf{5.8} & \textbf{6.2}& \textbf{6.6} & \textbf{7.0}  \\
    \midrule
         \multirow{3}{*}{0} & a=0 & 97.7 (97.4, 98.0) & 97.7 (97.4, 98.0) & 97.7 (97.4, 98.0) & 97.7 (97.4, 98.0) & 97.7 (97.4, 98.0) & 97.7 (97.4, 98.0) \\
         & a=1 & 97.6 (97.3, 97.9) & 97.5 (97.1, 97.8) & 97.3 (96.8, 97.6) & 97.7 (97.3, 98.1) & 93.7 (92.1, 94.9) & 83.4 (77.2, 86.9)\\
         & \cellcolor{gray!50} $\Delta \text{AUC}$ & \cellcolor{gray!50} 0.2 (0.0, 0.6) & \cellcolor{gray!50} 0.2 (0.0, 0.7) & \cellcolor{gray!50} 0.4 (0.0, 1.0) & \cellcolor{gray!50} 2.4 (1.6, 3.6) & \cellcolor{gray!50} 4.0 (2.9, 5.7) & \cellcolor{gray!50} 14.3 (10.8, 20.5) \\
    \midrule
         \multirow{3}{*}{30} & a=0 & 97.7 (97.4, 98.0) & 97.7 (97.4, 98.0) & 97.7 (97.4, 98.0) & 97.7 (97.4, 98.0) & 97.7 (97.4, 98.0) & 97.7 (97.4, 98.0) \\
         & a=1 & 96.5 (95.8, 97.1) & 94.6 (93.4, 95.4) & 93.7 (92.4, 94.6) & 91.9 (90.2, 93.0) & 90.7 (88.6, 92.0) & 82.1 (77.1, 86.8) \\
        & \cellcolor{gray!50} $\Delta \text{AUC}$ & \cellcolor{gray!50} 1.2 (0.6, 1.9) & \cellcolor{gray!50} 3.1 (2.3, 4.2) & \cellcolor{gray!50} 3.9 (2.9, 5.3) & \cellcolor{gray!50} 5.8 (4.6, 7.5) & \cellcolor{gray!50} 7.1 (5.7, 9.1) & \cellcolor{gray!50} 15.6 (11.2, 20.6)\\
    \midrule
         \multirow{3}{*}{60} & a=0 & 97.7 (97.4, 98.0) & 97.7 (97.4, 98.0) & 97.7 (97.4, 98.0) & 97.7 (97.4, 98.0) & 97.7 (97.4, 98.0) & 97.7 (97.4, 98.0) \\
         & a=1 & 92.9 (91.4, 93.8) & 89.0 (87.6, 90.3) & 87.0 (85.5, 88.7) & 83.1 (81.0, 85.2) & 89.9 (87.5, 91.4) & 77.9 (72.6, 81.6) \\
        & \cellcolor{gray!50} $\Delta \text{AUC}$ & \cellcolor{gray!50} 4.7 (3.9, 6.3) & \cellcolor{gray!50} 8.6 (7.3, 10.0) & \cellcolor{gray!50} 10.7 (8.9, 12.2) & \cellcolor{gray!50} 14.6 (12.4, 16.8) & \cellcolor{gray!50} 15.1 (13.2, 18.0) & \cellcolor{gray!50} 19.8 (16.1, 25.2)\\
    \midrule    
         \multirow{3}{*}{90} & a=0 & 97.7 (97.4, 98.0) & 97.7 (97.4, 98.0) & 97.7 (97.4, 98.0) & 97.7 (97.4, 98.0) & 97.7 (97.4, 98.0) & 97.7 (97.4, 98.0) \\
         & a=1 & 88.1 (86.4, 89.7) & 83.4 (81.6, 85.4) & 80.9 (78.8, 83.2) & 76.4 (74.1, 79.3) & 75.3 (71.9, 78.7) & 79.5 (74.4, 83.6) \\
         & \cellcolor{gray!50} $\Delta \text{AUC}$ & \cellcolor{gray!50} 9.7 (8.0, 11.3) & \cellcolor{gray!50} 14.4 (12.3, 16.0) & \cellcolor{gray!50} 16.9 (14.4, 18.9) & \cellcolor{gray!50} 11.3 (9.5, 13.1) &  \cellcolor{gray!50} 21.3 (18.4, 23.6) & \cellcolor{gray!50} 22.5 (18.6, 25.9)\\
    \midrule         
         \multirow{3}{*}{120} & a=0 & 97.7 (97.4, 98.0) & 97.7 (97.4, 98.0) & 97.7 (97.4, 98.0) & 97.7 (97.4, 98.0) & 97.7 (97.4, 98.0) & 97.7 (97.4, 98.0) \\
         & a=1 & 82.7 (81.0, 84.9) & 77.6 (75.7, 80.4) & 75.6 (72.9, 77.8) & 71.2 (68.0, 75.3) & 70.8 (67.4, 75.5) & 69.8 (65.1, 74.7) \\
         & \cellcolor{gray!50} $\Delta \text{AUC}$ & \cellcolor{gray!50} 14.9 (12.6, 16.7) & \cellcolor{gray!50} 20.0 (17.1, 21.9) & \cellcolor{gray!50} 22.1 (19.9, 24.8) & \cellcolor{gray!50} 26.5 (22.5, 29.6) &  \cellcolor{gray!50} 26.9 (22.2, 30.6) & \cellcolor{gray!50} 27.8 (22.8, 32.9)\\
    \midrule     
         \multirow{3}{*}{150} & a=0 & 97.7 (97.4, 98.0) & 97.7 (97.4, 98.0) & 97.7 (97.4, 98.0) & 97.7 (97.4, 98.0) & 97.7 (97.4, 98.0) & 97.7 (97.4, 98.0) \\
         & a=1 & 78.6 (76.4, 80.9) & 73.2 (70.5, 76.2) & 71.4 (68.4, 74.3) & 68.9 (64.7, 72.9) & 68.5 (64.4, 73.5) & 69.0 (32.8, 74.5)\\
         & \cellcolor{gray!50} $\Delta \text{AUC}$ & \cellcolor{gray!50} 19.3 (16.7, 21.2) & \cellcolor{gray!50} 24.5 (21.3, 27.3) & \cellcolor{gray!50} 26.2 (23.4, 29.2) & \cellcolor{gray!50} 28.7 (24.8, 33.0) &  \cellcolor{gray!50} 29.1 (24.1, 33.4) & \cellcolor{gray!50} 28.5 (23.4, 64.9)\\
    \midrule      
         \multirow{3}{*}{180} & a=0 & 97.7 (97.4, 98.0) & 97.7 (97.4, 98.0) & 97.7 (97.4, 98.0) & 97.7 (97.4, 98.0) & 97.7 (97.4, 98.0) & 97.7 (97.4, 98.0) \\
         & a=1 & 77.0 (74.2, 79.3) & 71.4 (68.4, 74.3) & 69.8 (66.6, 72.8) & 67.8 (63.0, 73.2) & 68.0 (63.3, 72.9) & 67.9 (33.0, 72.9)\\
         & \cellcolor{gray!50} $\Delta \text{AUC}$ & \cellcolor{gray!50} 20.8 (18.3, 23.6) & \cellcolor{gray!50} 26.4 (23.5, 29.3) & \cellcolor{gray!50} 27.8 (24.9, 31.0) & \cellcolor{gray!50} 29.8 (24.6, 34.8) &  \cellcolor{gray!50} 29.6 (24.7, 34.5) & \cellcolor{gray!50} 29.8 (24.9, 64.7)\\
    \midrule      
    \end{tabular}
    \caption{Group-wise AUC and $\Delta \text{AUC}$ under conditional distributional difference for all values of $\tau_1, \phi$, at $d' = 6$.}
    \label{tab:auc_setting3_d6}
\end{sidewaystable}

\begin{sidewaystable}[t]
    \centering
    \small
    \begin{tabular}{c|c| c c c c c c}
    \toprule
         \multicolumn{2}{c|}{} & \multicolumn{6}{c}{$\tau_1$}  \\
    \midrule
         $\phi$ & \textbf{Group} & \textbf{5.0}& \textbf{5.4} & \textbf{5.8} & \textbf{6.2}& \textbf{6.6} & \textbf{7.0}  \\
    \midrule
         \multirow{3}{*}{0} & a=0 & 97.6 (97.0, 98.1) & 90.9 (89.5, 92.4) & 85.8 (83.6, 88.3) & 72.3 (68.5, 77.0) & 69.0 (64.6, 73.6) & 62.3 (55.8, 67.7) \\
         & a=1 & 97.6 (97.2, 98.2) & 99.7 (99.5, 99.8) & 99.9 (99.8, 99.9) & 100.0 (99.9, 100.0) & 100.0 (99.9, 100.0) & 99.9 (99.8, 100.0)\\
         & \cellcolor{gray!50} $\Delta \text{xAUC}$ & \cellcolor{gray!50} 0.4 (0.1, 1.1) & \cellcolor{gray!50} 8.8 (7.1, 10.2) & \cellcolor{gray!50} 14.1 (11.6, 16.3) & \cellcolor{gray!50} 27.6 (23.0, 31.4) & \cellcolor{gray!50} 30.9 (26.3, 35.4) & \cellcolor{gray!50} 37.6 (32.0, 44.2) \\
    \midrule
         \multirow{3}{*}{30} & a=0 & 95.1 (94.1, 96.1) & 88.4 (86.4, 90.6) & 83.5 (81.2, 86.4) & 71.1 (66.4, 75.9) & 68.1 (63.8, 73.7) & 62.1 (56.5, 67.9) \\
         & a=1 & 98.5 (98.0, 98.8) & 99.3 (99.0, 99.4) & 99.6 (99.3, 99.7) & 99.9 (99.8, 99.9) & 99.9 (99.8, 100.0) & 99.9 (99.8, 100.0) \\
         & \cellcolor{gray!50} $\Delta \text{xAUC}$ & \cellcolor{gray!50} 3.3 (2.3, 4.5) & \cellcolor{gray!50} 10.9 (8.7, 13.0) & \cellcolor{gray!50} 16.1 (12.9, 18.4) & \cellcolor{gray!50} 28.8 (23.9, 33.5) & \cellcolor{gray!50} 31.9 (26.2, 36.2) & \cellcolor{gray!50} 37.8 (31.9, 43.4)\\
    \midrule
         \multirow{3}{*}{60} & a=0 & 90.4 (88.7, 92.0) & 82.9 (80.1, 85.7) & 79.3 (76.3, 81.7) & 69.9 (64.3, 74.9) & 66.6 (60.5, 72.3) & 61.6 (56.1, 67.8) \\
         & a=1 & 98.5 (98.1, 99.0) & 99.0 (98.6, 99.3) & 99.2 (98.9, 99.4) & 99.7 (99.5, 99.8) & 99.8 (99.6, 99.9) & 99.9 (99.7, 100.0)\\
        & \cellcolor{gray!50} $\Delta \text{xAUC}$ & \cellcolor{gray!50} 8.2 (6.3, 10.5) & \cellcolor{gray!50} 16.1 (13.0, 19.1) & \cellcolor{gray!50} 20.0 (17.4, 23.1) & \cellcolor{gray!50} 30.8 (25.5, 36.8) & \cellcolor{gray!50} 33.3 (27.3, 39.4) & \cellcolor{gray!50} 38.3 (31.9, 43.9)\\
    \midrule    
         \multirow{3}{*}{90} & a=0 & 89.0 (86.9, 90.8) & 81.6 (79.4, 84.7) & 75.4 (71.3, 78.9) & 66.8 (59.3, 71.6) & 65.5 (58.9, 70.8) & 61.3 (54.3, 66.9) \\
         & a=1 & 98.6 (98.2, 99.0) & 99.0 (98.6, 99.2) & 99.1 (98.8, 99.4) & 99.7 (99.4, 99.8) & 99.7 (99.5, 99.9) & 99.9 (99.7, 100.0)\\
         & \cellcolor{gray!50} $\Delta \text{xAUC}$ & \cellcolor{gray!50} 9.7 (7.7, 11.8) & \cellcolor{gray!50} 17.4 (14.1, 19.6) & \cellcolor{gray!50} 23.8 (20.1, 28.0) & \cellcolor{gray!50} 32.8 (27.8, 40.3) &  \cellcolor{gray!50} 34.2 (28.8, 41.0) & \cellcolor{gray!50} 38.6 (32.7, 45.6) \\
    \midrule         
         \multirow{3}{*}{120} & a=0 & 81.8 (79.1, 85.0) & 75.2 (71.3, 79.1) & 71.9 (66.7, 76.1) & 65.1 (59.1, 70.9) & 63.6 (57.3, 70.2) & 60.0 (53.4, 66.2) \\
         & a=1 & 98.5 (98.1, 98.9) & 99.0 (98.5, 99.3) & 99.2 (98.8, 99.4) & 99.6 (99.4, 99.8) & 99.8 (99.5, 99.9) & 99.9 (99.7, 100.0) \\
         & \cellcolor{gray!50} $\Delta \text{xAUC}$ & \cellcolor{gray!50} 16.8 (13.3, 19.6) & \cellcolor{gray!50} 23.9 (19.7, 27.7) & \cellcolor{gray!50} 27.4 (23.0, 32.6) & \cellcolor{gray!50} 34.6 (28.5, 40.7) &  \cellcolor{gray!50} 36.3 (29.4, 42.5) & \cellcolor{gray!50} 39.9 (33.5, 46.5)\\
    \midrule     
         \multirow{3}{*}{150} & a=0 & 78.8 (76.3, 82.8) & 72.5 (67.9, 77.0) & 69.5 (65.1, 74.6) & 63.7 (57.9, 69.4) & 63.7 (57.9, 69.4) & 59.0 (52.2, 65.8) \\
         & a=1 & 98.6 (98.1, 99.0) & 99.1 (98.6, 99.4) & 99.3 (98.9, 99.6) & 99.7 (99.5, 99.9) & 99.7 (99.5, 99.9) & 99.9 (99.7, 100.0)\\
         & \cellcolor{gray!50} $\Delta \text{xAUC}$ & \cellcolor{gray!50} 19.6 (15.7, 22.3) & \cellcolor{gray!50} 26.6 (21.9, 31.1) & \cellcolor{gray!50} 29.9 (24.6, 34.3) & \cellcolor{gray!50} 36.1 (30.1, 41.9) &  \cellcolor{gray!50} 36.1 (30.1, 41.9) & \cellcolor{gray!50} 40.9 (33.8, 47.8)\\
    \midrule      
         \multirow{3}{*}{180} & a=0 & 78.4 (75.4, 81.6) & 71.6 (67.6, 76.7) & 68.8 (64.1, 73.0) & 63.4 (57.7, 68.5) & 61.9 (55.8, 67.4) & 58.7 (52.6, 64.9) \\
         & a=1 & 98.5 (98.1, 98.9) & 99.1 (98.6, 99.3) & 99.4 (99.0, 99.6) & 99.8 (99.5, 99.9) & 99.8 (99.6, 100.0) & 99.9 (99.7, 100.0)\\
         & \cellcolor{gray!50} $\Delta \text{xAUC}$ & \cellcolor{gray!50} 20.3 (16.7, 23.2) & \cellcolor{gray!50} 27.5 (22.2, 31.6) & \cellcolor{gray!50} 30.6 (26.3, 35.2) & \cellcolor{gray!50} 36.4 (30.9, 42.2) &  \cellcolor{gray!50} 38.0 (32.2, 44.1) & \cellcolor{gray!50} 41.2 (35.0, 47.4)\\
    \midrule      
    \end{tabular}
    \caption{Group-wise xAUC and $\Delta \text{xAUC}$ under conditional distributional difference for all values of $\tau_1, \phi$, at $d' = 6$.}
    \label{tab:xauc_setting3_d6}
\end{sidewaystable}

\begin{sidewaystable}[t]
    \centering
    \small
    \begin{tabular}{c|c| c c c c c c}
    \toprule
         \multicolumn{2}{c|}{} & \multicolumn{6}{c}{$\tau_1$}  \\
    \midrule
         $\phi$ & \textbf{Group} & \textbf{5.0}& \textbf{5.4} & \textbf{5.8} & \textbf{6.2}& \textbf{6.6} & \textbf{7.0}  \\
    \midrule
         \multirow{3}{*}{0} & a=0 & 97.7 (97.4, 98.0) & 97.7 (97.4, 98.0) & 97.7 (97.4, 98.0) & 97.7 (97.4, 98.0) & 97.7 (97.4, 98.0) & 97.7 (97.4, 98.0) \\
         & a=1 & 97.6 (97.3, 97.9) & 97.5 (97.1, 97.8) & 97.3 (96.8, 97.6) & 97.7 (97.3, 98.1) & 93.7 (92.1, 94.9) & 83.4 (77.2, 86.9)\\
         & \cellcolor{gray!50} $\Delta \text{AUC}$ & \cellcolor{gray!50} 0.2 (0.0, 0.6) & \cellcolor{gray!50} 0.2 (0.0, 0.7) & \cellcolor{gray!50} 0.4 (0.0, 1.0) & \cellcolor{gray!50} 2.4 (1.6, 3.6) & \cellcolor{gray!50} 4.0 (2.9, 5.7) & \cellcolor{gray!50} 14.3 (10.8, 20.5) \\
    \midrule
         \multirow{3}{*}{30} & a=0 & 97.7 (97.4, 98.0) & 97.7 (97.4, 98.0) & 97.7 (97.4, 98.0) & 97.7 (97.4, 98.0) & 97.7 (97.4, 98.0) & 97.7 (97.4, 98.0) \\
         & a=1 & 96.2 (95.7, 96.9) & 93.9 (92.8, 94.8) & 92.8 (91.1, 94.1) & 90.8 (89.2, 92.0) & 89.8 (87.5, 91.3) & 81.4 (76.4, 85.6) \\
        & \cellcolor{gray!50} $\Delta \text{AUC}$ & \cellcolor{gray!50} 1.5 (0.8, 2.0) & \cellcolor{gray!50} 3.8 (2.6, 4.9) & \cellcolor{gray!50} 4.9 (3.3, 6.6) & \cellcolor{gray!50} 7.0 (5.7, 8.6) & \cellcolor{gray!50} 7.1 (5.7, 9.1) & \cellcolor{gray!50} 16.4 (12.1, 21.4)\\
    \midrule
         \multirow{3}{*}{60} & a=0 & 97.7 (97.4, 98.0) & 97.7 (97.4, 98.0) & 97.7 (97.4, 98.0) & 97.7 (97.4, 98.0) & 97.7 (97.4, 98.0) & 97.7 (97.4, 98.0) \\
         & a=1 & 91.5 (90.2, 92.7) & 87.3 (85.5, 88.8) & 85.2 (83.0, 87.1) & 81.1 (78.4, 83.5) & 80.3 (77.3, 82.6) & 76.5 (72.5, 79.9) \\
        & \cellcolor{gray!50} $\Delta \text{AUC}$ & \cellcolor{gray!50} 6.1 (5.0, 7.3) & \cellcolor{gray!50} 10.4 (8.8, 12.4) & \cellcolor{gray!50} 12.4 (10.4, 14.6) & \cellcolor{gray!50} 16.5 (14.2, 19.5) & \cellcolor{gray!50} 17.5 (14.7, 20.3) & \cellcolor{gray!50} 21.2 (17.9, 25.1)\\
    \midrule    
         \multirow{3}{*}{90} & a=0 & 97.7 (97.4, 98.0) & 97.7 (97.4, 98.0) & 97.7 (97.4, 98.0) & 97.7 (97.4, 98.0) & 97.7 (97.4, 98.0) & 97.7 (97.4, 98.0) \\
         & a=1 & 84.9 (83.2, 86.9) & 79.9 (77.5, 81.9) & 77.7 (74.9, 79.9) & 73.2 (70.0, 76.7) & 72.9 (69.6, 75.9) & 71.6 (66.3, 75.1) \\
         & \cellcolor{gray!50} $\Delta \text{AUC}$ & \cellcolor{gray!50} 12.8 (10.9, 14.5) & \cellcolor{gray!50} 17.8 (15.8, 20.0) & \cellcolor{gray!50} 20.0 (17.7, 22.9) & \cellcolor{gray!50} 24.5 (21.2, 27.6) &  \cellcolor{gray!50} 24.7 (21.8, 28.0) & \cellcolor{gray!50} 26.1 (22.6, 31.4)\\
    \midrule         
         \multirow{3}{*}{120} & a=0 & 97.7 (97.4, 98.0) & 97.7 (97.4, 98.0) & 97.7 (97.4, 98.0) & 97.7 (97.4, 98.0) & 97.7 (97.4, 98.0) & 97.7 (97.4, 98.0) \\
         & a=1 & 78.2 (75.7, 80.9) & 73.1 (69.4, 76.1) & 71.1 (67.7, 73.6) & 69.0 (64.0, 73.0) & 68.5 (63.8, 73.1) & 69.5 (58.7, 73.5) \\
         & \cellcolor{gray!50} $\Delta \text{AUC}$ & \cellcolor{gray!50} 19.5 (17.0, 22.0) & \cellcolor{gray!50} 24.6 (21.6, 28.4) & \cellcolor{gray!50} 26.5 (24.0, 30.0)& \cellcolor{gray!50} 28.7 (24.7, 33.7) &  \cellcolor{gray!50} 29.1 (24.7, 34.0) & \cellcolor{gray!50} 28.1 (24.4, 38.7)\\
    \midrule     
          \multirow{3}{*}{150} & a=0 & 97.7 (97.4, 98.0) & 97.7 (97.4, 98.0) & 97.7 (97.4, 98.0) & 97.7 (97.4, 98.0) & 97.7 (97.4, 98.0) & 97.7 (97.4, 98.0) \\
         & a=1 & 71.0 (68.3, 74.8) & 73.2 (70.5, 76.2) & 66.2 (61.9, 70.6) & 67.0 (61.0, 72.6) & 67.3 (31.5, 72.4) & 68.1 (32.4, 73.7)\\
         & \cellcolor{gray!50} $\Delta \text{AUC}$ & \cellcolor{gray!50} 26.7 (23.0, 29.5) & \cellcolor{gray!50} 30.6 (26.9, 34.7) & \cellcolor{gray!50} 31.4 (26.9, 36.0) & \cellcolor{gray!50} 30.7 (25.0, 36.5) &  \cellcolor{gray!50} 30.4 (25.2, 66.4) & \cellcolor{gray!50} 29.5 (24.0, 65.5)\\
    \midrule      
         \multirow{3}{*}{180} & a=0 & 97.7 (97.4, 98.0) & 97.7 (97.4, 98.0) & 97.7 (97.4, 98.0) & 97.7 (97.4, 98.0) & 97.7 (97.4, 98.0) & 97.7 (97.4, 98.0) \\
         & a=1 & 67.7 (63.4, 70.7) & 63.6 (59.9, 68.4) & 63.8 (59.0, 69.0) & 65.8 (35.2, 70.8) & 66.3 (31.7, 72.3) & 67.2 (31.0, 72.7)\\
         & \cellcolor{gray!50} $\Delta \text{AUC}$ & \cellcolor{gray!50} 30.0 (27.0, 34.3) & \cellcolor{gray!50} 34.3 (29.3, 38.0) & \cellcolor{gray!50} 33.8 (28.7, 38.6) & \cellcolor{gray!50} 31.7 (26.6, 62.6) &  \cellcolor{gray!50} 31.6 (25.4, 66.0) & \cellcolor{gray!50} 30.5 (24.8, 66.6)\\
    \midrule      
    \end{tabular}
    \caption{Group-wise AUC and $\Delta \text{AUC}$ under conditional distributional difference for all values of $\tau_1, \phi$, at $d' = 8$.}
    \label{tab:auc_setting3_d8}
\end{sidewaystable}

\begin{sidewaystable}[t]
    \centering
    \small
    \begin{tabular}{c|c| c c c c c c}
    \toprule
         \multicolumn{2}{c|}{} & \multicolumn{6}{c}{$\tau_1$}  \\
    \midrule
         $\phi$ & \textbf{Group} & \textbf{5.0}& \textbf{5.4} & \textbf{5.8} & \textbf{6.2}& \textbf{6.6} & \textbf{7.0}  \\
    \midrule
         \multirow{3}{*}{0} & a=0 & 97.6 (97.0, 98.1) & 90.9 (89.5, 92.4) & 85.8 (83.6, 88.3) & 72.3 (68.5, 77.0) & 69.0 (64.6, 73.6) & 62.3 (55.8, 67.7) \\
         & a=1 & 97.6 (97.2, 98.2) & 99.7 (99.5, 99.8) & 99.9 (99.8, 99.9) & 100.0 (99.9, 100.0) & 100.0 (99.9, 100.0) & 99.9 (99.8, 100.0)\\
         & \cellcolor{gray!50} $\Delta \text{xAUC}$ & \cellcolor{gray!50} 0.4 (0.1, 1.1) & \cellcolor{gray!50} 8.8 (7.1, 10.2) & \cellcolor{gray!50} 14.1 (11.6, 16.3) & \cellcolor{gray!50} 27.6 (23.0, 31.4) & \cellcolor{gray!50} 30.9 (26.3, 35.4) & \cellcolor{gray!50} 37.6 (32.0, 44.2) \\
    \midrule
         \multirow{3}{*}{30} & a=0 & 94.6 (93.6, 95.5) & 87.9 (85.6, 89.8) & 82.8 (80.7, 85.9) & 70.9 (66.2, 75.6) & 68.2 (62.9, 73.9) & 62.4 (56.7, 68.1) \\
         & a=1 & 98.5 (98.1, 98.9) & 99.2 (98.9, 99.4) & 99.5 (99.3, 99.7) & 99.9 (99.8, 99.9) & 99.9 (99.8, 100.0) & 99.9 (99.8, 100.0) \\
         & \cellcolor{gray!50} $\Delta \text{xAUC}$ & \cellcolor{gray!50} 3.9 (2.9, 5.2) & \cellcolor{gray!50} 11.5 (9.3, 13.6) & \cellcolor{gray!50} 16.6 (13.6, 18.9) & \cellcolor{gray!50} 29.0 (24.2, 33.7) & \cellcolor{gray!50} 31.7 (25.9, 37.0) & \cellcolor{gray!50} 37.5 (31.7, 43.2)\\
    \midrule
         \multirow{3}{*}{60} & a=0 & 89.1 (86.8, 90.6) & 81.8 (78.7, 84.4) & 77.9 (74.5, 80.3) & 68.4 (61.9, 72.9) & 66.2 (60.5, 71.1) & 62.0 (55.6, 67.4) \\
         & a=1 & 98.6 (98.1, 98.9) & 99.0 (98.6, 99.3) & 99.2 (98.8, 99.4) & 99.7 (99.4, 99.8) & 99.8 (99.6, 99.9) & 99.9 (99.7, 100.0)\\
        & \cellcolor{gray!50} $\Delta \text{xAUC}$ & \cellcolor{gray!50} 9.5 (7.6, 11.9) & \cellcolor{gray!50} 17.2 (14.4, 20.5) & \cellcolor{gray!50} 21.2 (18.6, 24.9) & \cellcolor{gray!50} 31.3 (26.6, 37.7) & \cellcolor{gray!50} 33.5 (28.5, 39.3) & \cellcolor{gray!50} 38.0 (32.4, 44.3)\\
    \midrule    
         \multirow{3}{*}{90} & a=0 & 83.6 (80.8, 86.0) & 76.5 (72.6, 79.7) & 73.1 (67.8, 76.2) & 66.1 (59.3, 70.5) & 64.6 (58.2, 69.7) & 60.5 (54.2, 66.1) \\
         & a=1 & 98.5 (98.1, 98.9) & 98.9 (98.5, 99.2) & 99.1 (98.8, 99.4) & 99.6 (99.4, 99.8) & 99.7 (99.5, 99.9) & 99.9 (99.7, 100.0)\\
         & \cellcolor{gray!50} $\Delta \text{xAUC}$ & \cellcolor{gray!50} 15.0 (12.2, 17.6) & \cellcolor{gray!50} 22.4 (19.0, 26.4) & \cellcolor{gray!50} 26.1 (22.2, 31.5) & \cellcolor{gray!50} 33.5 (28.9, 40.3) &  \cellcolor{gray!50} 35.2 (29.8, 41.5) & \cellcolor{gray!50} 39.4 (33.6, 45.6) \\
    \midrule         
         \multirow{3}{*}{120} & a=0 & 79.0 (75.1, 81.9) & 72.3 (67.0, 76.5) & 69.4 (63.4, 74.0) & 63.8 (57.9, 69.5) & 62.1 (55.9, 67.5) & 58.9 (52.6, 65.0) \\
         & a=1 & 98.6 (98.1, 99.0) & 99.1 (98.7, 99.4) & 99.3 (98.9, 99.6) & 99.8 (99.5, 99.9) & 99.8 (99.6, 100.0) & 99.9 (99.7, 100.0) \\
         & \cellcolor{gray!50} $\Delta \text{xAUC}$ & \cellcolor{gray!50} 19.7 (16.3, 23.6) & \cellcolor{gray!50} 26.9 (22.5, 32.0) & \cellcolor{gray!50} 29.8 (25.2, 36.2) & \cellcolor{gray!50} 35.9 (30.1, 42.1) &  \cellcolor{gray!50} 36.3 (29.4, 42.5) & \cellcolor{gray!50} 41.1 (34.8, 47.4)\\
    \midrule     
         \multirow{3}{*}{150} & a=0 & 75.1 (71.4, 79.3) & 68.6 (63.4, 73.3) & 66.3 (61.0, 71.5) & 60.7 (54.6, 66.0) & 59.4 (52.3, 65.4) & 58.0 (51.9, 63.8) \\
         & a=1 & 98.8 (98.3, 99.1) & 99.4 (99.1, 99.6) & 99.6 (99.3, 99.8) & 99.9 (99.7, 100.0) & 99.9 (99.7, 100.0) & 99.9 (99.8, 100.0)\\
         & \cellcolor{gray!50} $\Delta \text{xAUC}$ & \cellcolor{gray!50} 19.6 (15.7, 22.3) & \cellcolor{gray!50} 30.7 (25.9, 35.8) & \cellcolor{gray!50} 33.4 (27.9, 38.8) & \cellcolor{gray!50} 39.2 (33.7, 45.3) &  \cellcolor{gray!50} 40.5 (34.4, 47.7) & \cellcolor{gray!50} 41.9 (36.1, 48.1)\\
    \midrule      
         \multirow{3}{*}{180} & a=0 & 74.0 (70.3, 77.8) & 67.5 (62.1, 72.2) & 64.4 (58.9, 69.7) & 59.4 (53.0, 65.2) & 58.3 (51.6, 63.5) & 57.7 (50.7, 63.7) \\
         & a=1 & 98.9 (98.4, 99.3) & 99.5 (99.2, 99.7) & 99.7 (99.5, 99.9) & 99.9 (99.7, 100.0) & 99.9 (99.7, 100.0) & 99.9 (99.8, 100.0)\\
         & \cellcolor{gray!50} $\Delta \text{xAUC}$ & \cellcolor{gray!50} 24.9 (20.7, 29.0) & \cellcolor{gray!50} 31.9 (27.1, 37.5) & \cellcolor{gray!50} 35.2 (29.8, 40.9) & \cellcolor{gray!50} 40.5 (34.5, 46.9) &  \cellcolor{gray!50} 41.7 (36.3, 48.3) & \cellcolor{gray!50} 42.2 (36.1, 49.3)\\
    \midrule      
    \end{tabular}
    \caption{Group-wise xAUC and $\Delta \text{xAUC}$ under conditional distributional difference for all values of $\tau_1, \phi$, at $d' = 8$.}
    \label{tab:xauc_setting3_d8}
\end{sidewaystable}

\begin{sidewaystable}[t]
    \centering
    \small
    \begin{tabular}{c|c| c c c c c c}
    \toprule
         \multicolumn{2}{c|}{} & \multicolumn{6}{c}{$\tau_1$}  \\
    \midrule
         $\phi$ & \textbf{Group} & \textbf{5.0}& \textbf{5.4} & \textbf{5.8} & \textbf{6.2}& \textbf{6.6} & \textbf{7.0}  \\
    \midrule
         \multirow{3}{*}{0} & a=0 & 97.7 (97.4, 98.0) & 97.7 (97.4, 98.0) & 97.7 (97.4, 98.0) & 97.7 (97.4, 98.0) & 97.7 (97.4, 98.0) & 97.7 (97.4, 98.0) \\
         & a=1 & 97.6 (97.3, 97.9) & 97.5 (97.1, 97.8) & 97.3 (96.8, 97.6) & 97.7 (97.3, 98.1) & 93.7 (92.1, 94.9) & 83.4 (77.2, 86.9)\\
         & \cellcolor{gray!50} $\Delta \text{AUC}$ & \cellcolor{gray!50} 0.2 (0.0, 0.6) & \cellcolor{gray!50} 0.2 (0.0, 0.7) & \cellcolor{gray!50} 0.4 (0.0, 1.0) & \cellcolor{gray!50} 2.4 (1.6, 3.6) & \cellcolor{gray!50} 4.0 (2.9, 5.7) & \cellcolor{gray!50} 14.3 (10.8, 20.5) \\
    \midrule
         \multirow{3}{*}{30} & a=0 & 97.7 (97.4, 98.0) & 97.7 (97.4, 98.0) & 97.7 (97.4, 98.0) & 97.7 (97.4, 98.0) & 97.7 (97.4, 98.0) & 97.7 (97.4, 98.0) \\
         & a=1 & 95.9 (95.0, 96.6) & 93.2 (91.9, 94.4) & 92.0 (90.3, 93.2) & 989.4 (87.7, 91.2) & 88.7 (86.1, 90.3) & 81.2 (76.0, 84.5) \\
        & \cellcolor{gray!50} $\Delta \text{AUC}$ & \cellcolor{gray!50} 1.8 (1.1, 2.7) & \cellcolor{gray!50} 4.5 (3.2, 5.9) & \cellcolor{gray!50} 5.7 (4.5, 7.4) & \cellcolor{gray!50} 8.2 (6.6, 10.0) & \cellcolor{gray!50} 9.0 (7.3, 11.5) & \cellcolor{gray!50} 16.5 (13.1, 21.7))\\
    \midrule
         \multirow{3}{*}{60} & a=0 & 97.7 (97.4, 98.0) & 97.7 (97.4, 98.0) & 97.7 (97.4, 98.0) & 97.7 (97.4, 98.0) & 97.7 (97.4, 98.0) & 97.7 (97.4, 98.0) \\
         & a=1 & 89.7 (88.5, 91.4) & 85.3 (83.1, 87.1) & 83.1 (80.5, 85.5) & 78.8 (76.0, 81.0) & 78.1 (75.1, 80.5) & 74.9 (71.0, 78.5) \\
        & \cellcolor{gray!50} $\Delta \text{AUC}$ & \cellcolor{gray!50} 8.0 (6.3, 9.2) & \cellcolor{gray!50} 12.4 (10.5, 14.5) & \cellcolor{gray!50} 14.6 (12.3, 17.1) & \cellcolor{gray!50} 18.7 (16.7, 21.5) & \cellcolor{gray!50} 19.6 (17.3, 22.6) & \cellcolor{gray!50} 22.8 (19.3, 26.6)\\
    \midrule    
         \multirow{3}{*}{90} & a=0 & 97.7 (97.4, 98.0) & 97.7 (97.4, 98.0) & 97.7 (97.4, 98.0) & 97.7 (97.4, 98.0) & 97.7 (97.4, 98.0) & 97.7 (97.4, 98.0) \\
         & a=1 & 81.3 (79.2, 83.1) & 76.0 (73.4, 78.3) & 73.8 (70.8, 76.2) & 71.1 (67.1, 73.2) & 70.3 (65.2, 72.9) & 70.0 (65.0, 73.9) \\
         & \cellcolor{gray!50} $\Delta \text{AUC}$ & \cellcolor{gray!50} 16.4 (14.6, 18.6) & \cellcolor{gray!50} 21.6 (19.2, 24.2) & \cellcolor{gray!50} 24.0 (21.5, 26.6) & \cellcolor{gray!50} 26.7 (24.2, 30.7)  &  \cellcolor{gray!50} 27.4 (24.7, 32.5) & \cellcolor{gray!50} 27.7 (24.0, 32.4)\\
    \midrule         
         \multirow{3}{*}{120} & a=0 & 97.7 (97.4, 98.0) & 97.7 (97.4, 98.0) & 97.7 (97.4, 98.0) & 97.7 (97.4, 98.0) & 97.7 (97.4, 98.0) & 97.7 (97.4, 98.0) \\
         & a=1 & 72.4 (69.2, 75.7) & 68.1 (65.2, 71.9) & 67.6 (63.0, 71.8) & 68.4 (61.7, 71.8) & 68.5 (63.8, 73.1) & 68.9 (31.4, 72.8) \\
         & \cellcolor{gray!50} $\Delta \text{AUC}$ & \cellcolor{gray!50} 25.4 (22.2, 28.5) & \cellcolor{gray!50} 29.6 (25.7, 32.6) & \cellcolor{gray!50} 30.4 (26.7, 33.9) & \cellcolor{gray!50} 30.0 (26.1, 34.6) &  \cellcolor{gray!50} 29.3 (26.0, 35.7) & \cellcolor{gray!50} 28.9 (24.9, 66.1)\\
    \midrule     
          \multirow{3}{*}{150} & a=0 & 97.7 (97.4, 98.0) & 97.7 (97.4, 98.0) & 97.7 (97.4, 98.0) & 97.7 (97.4, 98.0) & 97.7 (97.4, 98.0) & 97.7 (97.4, 98.0) \\
         & a=1 & 64.4 (60.8, 68.5) & 66.2 (57.1, 70.2) & 66.8 (32.9, 71.1) & 67.8 (31.0, 71.6) & 67.8 (31.0, 71.6) & 67.8 (31.0, 71.6)\\
         & \cellcolor{gray!50} $\Delta \text{AUC}$ & \cellcolor{gray!50} 33.2 (28.9, 37.0) & \cellcolor{gray!50} 31.4 (27.1, 40.0) & \cellcolor{gray!50} 31.0 (26.4, 64.7) & \cellcolor{gray!50} 30.1 (26.0, 66.7) &  \cellcolor{gray!50} 30.1 (26.0, 66.7) & \cellcolor{gray!50} 30.1 (26.0, 66.7)\\
    \midrule      
         \multirow{3}{*}{180} & a=0 & 97.7 (97.4, 98.0) & 97.7 (97.4, 98.0) & 97.7 (97.4, 98.0) & 97.7 (97.4, 98.0) & 97.7 (97.4, 98.0) & 97.7 (97.4, 98.0) \\
         & a=1 & 66.0 (31.4, 72.1) & 66.0 (31.4, 72.1) & 63.8 (59.0, 69.0) & 65.8 (35.2, 70.8) & 66.3 (31.7, 72.3) & 67.2 (31.0, 72.7)\\
         & \cellcolor{gray!50} $\Delta \text{AUC}$ & \cellcolor{gray!50} 31.7 (25.6, 66.2) & \cellcolor{gray!50} 31.7 (25.6, 66.2) & \cellcolor{gray!50} 31.7 (25.6, 66.2) & \cellcolor{gray!50} 31.7 (25.6, 66.2) &  \cellcolor{gray!50} 31.7 (25.6, 66.2)  & \cellcolor{gray!50} 31.7 (25.6, 66.2) \\
    \midrule      
    \end{tabular}
    \caption{Group-wise AUC and $\Delta \text{AUC}$ under conditional distributional difference for all values of $\tau_1, \phi$, at $d' = 10$.}
    \label{tab:auc_setting3_d10}
\end{sidewaystable}

\begin{sidewaystable}[t]
    \centering
    \small
    \begin{tabular}{c|c| c c c c c c}
    \toprule
         \multicolumn{2}{c|}{} & \multicolumn{6}{c}{$\tau_1$}  \\
    \midrule
         $\phi$ & \textbf{Group} & \textbf{5.0}& \textbf{5.4} & \textbf{5.8} & \textbf{6.2}& \textbf{6.6} & \textbf{7.0}  \\
    \midrule
         \multirow{3}{*}{0} & a=0 & 97.6 (97.0, 98.1) & 90.9 (89.5, 92.4) & 85.8 (83.6, 88.3) & 72.3 (68.5, 77.0) & 69.0 (64.6, 73.6) & 62.3 (55.8, 67.7) \\
         & a=1 & 97.6 (97.2, 98.2) & 99.7 (99.5, 99.8) & 99.9 (99.8, 99.9) & 100.0 (99.9, 100.0) & 100.0 (99.9, 100.0) & 99.9 (99.8, 100.0)\\
         & \cellcolor{gray!50} $\Delta \text{xAUC}$ & \cellcolor{gray!50} 0.4 (0.1, 1.1) & \cellcolor{gray!50} 8.8 (7.1, 10.2) & \cellcolor{gray!50} 14.1 (11.6, 16.3) & \cellcolor{gray!50} 27.6 (23.0, 31.4) & \cellcolor{gray!50} 30.9 (26.3, 35.4) & \cellcolor{gray!50} 37.6 (32.0, 44.2) \\
    \midrule
         \multirow{3}{*}{30} & a=0 & 94.0 (93.0, 95.0) & 87.1 (85.1, 89.2) & 82.2 (79.6, 85.1) & 70.6 (65.8, 75.3) & 67.8 (62.4, 73.5) & 62.2 (56.7, 67.9) \\
         & a=1 & 98.5 (98.1, 98.9) & 99.1 (98.8, 99.4) & 99.4 (99.2, 99.6) & 99.9 (99.7, 99.9) & 99.9 (99.8, 100.0) & 99.9 (99.8, 100.0) \\
         & \cellcolor{gray!50} $\Delta \text{xAUC}$ & \cellcolor{gray!50} 4.5 (3.4, 5.8) & \cellcolor{gray!50} 12.0 (9.8, 14.0) & \cellcolor{gray!50} 17.3 (14.2, 19.8) & \cellcolor{gray!50} 29.2 (24.4, 34.1) & \cellcolor{gray!50} 32.1 (26.3, 37.5) & \cellcolor{gray!50} 37.6 (31.8, 43.2)\\
    \midrule
         \multirow{3}{*}{60} & a=0 & 87.4 (85.4, 89.5) & 80.4 (77.2, 83.0) & 77.9 (74.5, 80.3) & 67.7 (61.4, 72.3) & 66.2 (59.3, 70.8) & 61.7 (55.9, 67.9) \\
         & a=1 & 98.6 (98.1, 98.9) & 98.9 (98.6, 99.2) & 99.1 (98.8, 99.4) & 99.6 (99.5, 99.8) & 99.8 (99.5, 99.9) & 99.9 (99.7, 100.0)\\
        & \cellcolor{gray!50} $\Delta \text{xAUC}$ & \cellcolor{gray!50} 11.1 (9.1, 13.3) & \cellcolor{gray!50} 18.8 (15.5, 22.1) & \cellcolor{gray!50} 22.4 (18.9, 26.1) & \cellcolor{gray!50} 31.9 (27.2, 38.4) & \cellcolor{gray!50} 33.5 (28.8, 40.5) & \cellcolor{gray!50} 38.1 (31.9, 44.0)\\
    \midrule    
         \multirow{3}{*}{90} & a=0 & 81.0 (77.5, 83.8) & 74.1 (70.1, 77.7) & 71.2 (66.4, 75.6) & 64.7 (57.9, 70.0) & 63.2 (57.0, 69.3) & 59.8 (52.5, 65.5) \\
         & a=1 & 98.5 (98.1, 98.9) & 99.0 (98.5, 99.2) & 99.2 (98.8, 99.4) & 99.7 (99.5, 99.9) & 99.8 (99.6, 99.9) & 99.9 (99.7, 100.0)\\
         & \cellcolor{gray!50} $\Delta \text{xAUC}$ & \cellcolor{gray!50} 17.6 (14.2, 21.1) & \cellcolor{gray!50} 22.4 (19.0, 26.4) & \cellcolor{gray!50} 27.9 (23.5, 32.8) & \cellcolor{gray!50} 35.0 (29.5, 41.8) &  \cellcolor{gray!50} 36.5 (30.2, 42.9) & \cellcolor{gray!50} 40.1 (34.3, 47.4) \\
    \midrule         
         \multirow{3}{*}{120} & a=0 & 75.5 (71.5, 79.6) & 69.2 (63.2, 74.1) & 66.4 (58.9, 71.5) & 60.7 (53.8, 66.7) & 59.5 (51.8, 65.5) & 58.2 (49.7, 63.4) \\
         & a=1 & 98.8 (98.2, 99.2) & 99.4 (98.9, 99.6) & 99.6 (99.2, 99.8) & 99.9 (99.6, 100.0) & 99.9 (99.7, 100.0) & 99.9 (99.7, 100.0) \\
         & \cellcolor{gray!50} $\Delta \text{xAUC}$ & \cellcolor{gray!50} 23.2 (18.8, 27.9) & \cellcolor{gray!50} 30.1 (24.9, 36.5) & \cellcolor{gray!50} 33.1 (28.0, 40.7) & \cellcolor{gray!50} 39.2 (32.8, 46.1) &  \cellcolor{gray!50} 40.3 (34.3, 48.1) & \cellcolor{gray!50} 41.7 (36.3, 50.3)\\
    \midrule     
         \multirow{3}{*}{150} & a=0 & 67.9 (62.8, 72.4) & 61.0 (54.3, 67.3) & 58.7 (50.6, 64.8) & 60.7 (54.6, 66.0) & 59.4 (52.3, 65.4) & 57.3 (50.5, 63.9) \\
         & a=1 & 99.5 (99.1, 99.7) & 99.9 (99.6, 100.0) & 99.6 (99.3, 99.8) & 99.9 (99.7, 100.0) & 99.9 (99.8, 100.0) & 99.9 (99.8, 100.0)\\
         & \cellcolor{gray!50} $\Delta \text{xAUC}$ & \cellcolor{gray!50} 31.6 (26.9, 37.0) & \cellcolor{gray!50} 38.8 (32.4, 45.7) & \cellcolor{gray!50} 33.4 (27.9, 38.8) & \cellcolor{gray!50} 41.2 (35.0, 49.4) &  \cellcolor{gray!50} 42.5 (35.9, 49.5) & \cellcolor{gray!50} 42.5 (35.9, 49.5)
         \\
    \midrule      
         \multirow{3}{*}{180} & a=0 & 57.2 (48.3, 63.9) & 57.2 (48.3, 63.9) & 57.2 (48.3, 63.9) & 57.2 (48.3, 63.9) & 57.2 (48.3, 63.9) & 57.2 (48.3, 63.9) \\
         & a=1 & 99.9 (99.7, 100.0) & 99.9 (99.7, 100.0) & 99.9 (99.7, 100.0) & 99.9 (99.7, 100.0) & 99.9 (99.7, 100.0) & 99.9 (99.7, 100.0)\\
         & \cellcolor{gray!50} $\Delta \text{xAUC}$ & \cellcolor{gray!50} 42.7 (35.8, 51.7) & \cellcolor{gray!50} 42.7 (35.8, 51.7) & \cellcolor{gray!50} 42.7 (35.8, 51.7) & \cellcolor{gray!50} 42.7 (35.8, 51.7) &  \cellcolor{gray!50} 42.7 (35.8, 51.7) & \cellcolor{gray!50} 42.7 (35.8, 51.7)\\
    \midrule      
    \end{tabular}
    \caption{Group-wise xAUC and $\Delta \text{xAUC}$ under conditional distributional difference for all values of $\tau_1, \phi$, at $d' = 10$.}
    \label{tab:xauc_setting3_d10}
\end{sidewaystable}

\subsection{Additional Results for Varying Marginal Distributional Distance (Setting 2)}

\paragraph{Increasing divergence between the marginal risk distributions correlates with larger performance gaps under disparate censorship.} We set $\tau_0 = 5, \tau_1 = 6.6$, a setting in which we observed performance gaps under disparate censorship, and vary $\Delta\mu$ (defined as $\mu_1 - \mu_0$) and $\sigma$. We also fix $\mu_0 + \mu_1 = 0.35 + 0.55 = 0.9$. We are interested whether varying the KL divergence between the marginals, given by $\frac{1}{2 \sigma} \lVert \Delta \mu \rVert_2^2$, impacts the magnitude of performance gaps. To that end, we first explore the impacts of changing $\sigma$ for both groups on subgroup performance gaps.
We then vary $\Delta\mu$, or the difference of the means between the marginal risk distributions, on subgroup performance gaps. In summary, we find that increasing the KL divergence by either decreasing the within-group variance or increasing the mean difference exacerbates the negative impacts of disparate censorship.

\begin{figure}[t]
    \centering
    \includegraphics[width=0.85\linewidth]{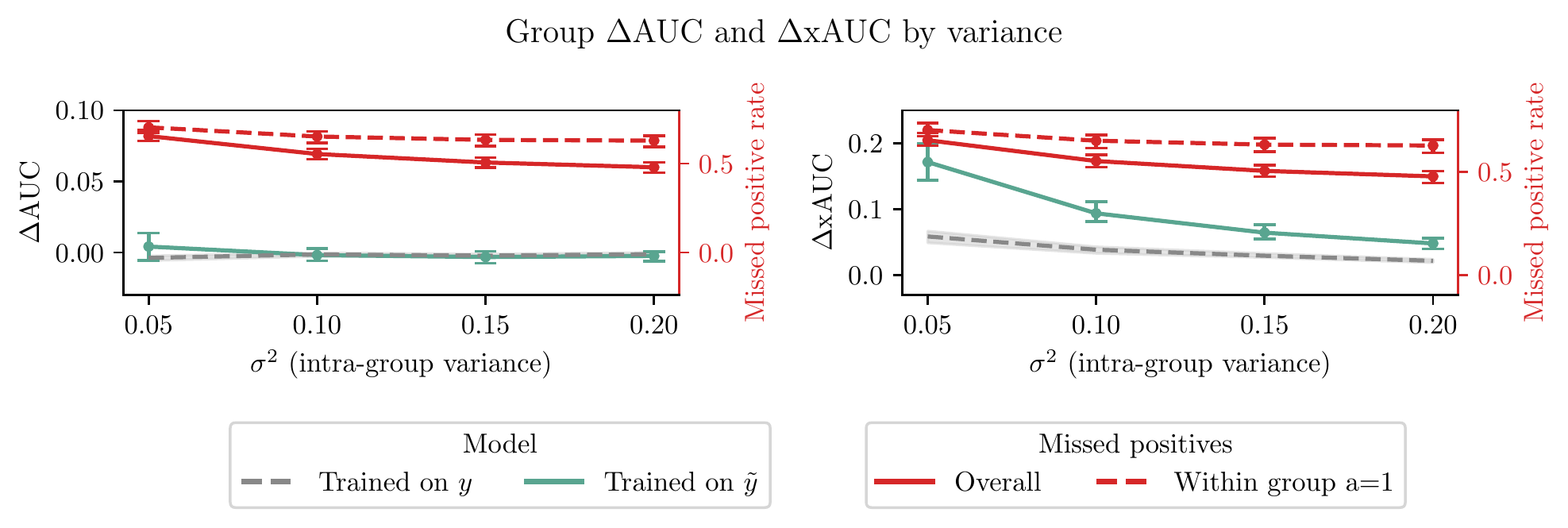}
     \caption{
     $\Delta \text{AUC}$ (left) and $\Delta \text{xAUC}$ (right) with 95\% empirical CIs for the model trained on $y$ (gray) and on $\tilde{y}$ (green). As $\sigma^2$ increases to the right, increasing overlap between the distributions, the number of missed positives increases, and $\Delta \text{AUC}, \Delta \text{xAUC}$ decrease.} %
    \label{fig:variance}
\end{figure}

We first evaluate the effect of within-group variance. As we decrease $\sigma^2$ from 0.2 to 0.05, we observe increasing AUC and xAUC gaps. As seen in Figure~\ref{fig:variance}, at a variance of 0.05 (left side, all graphs), the median AUC and xAUC gaps are 0.09 and 0.22. However, as we increase the within-group variance to 0.2 (right side, all graphs), the median AUC and xAUC gaps narrow to 0.01 and 0.06. Decreasing $\sigma^2$ also decreases the overlap between the two distributions, leading to greater performance disparities. %

\begin{figure}[t]
    \centering
    \includegraphics[width=0.85\linewidth]{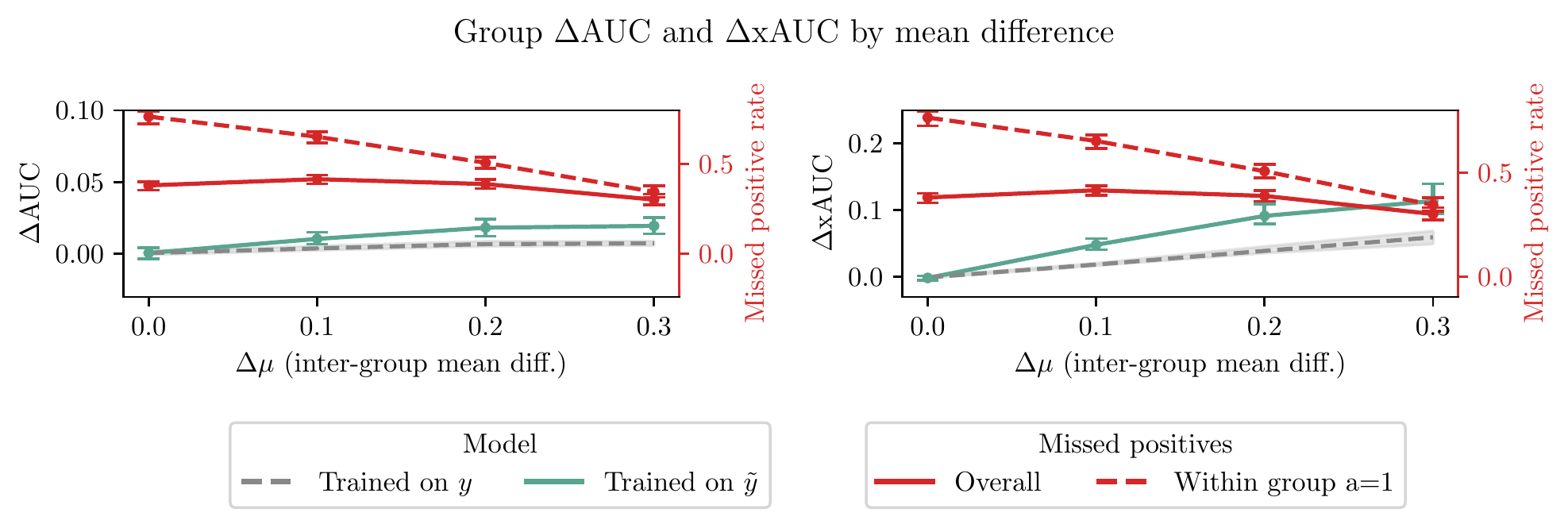}
    \caption{
    $\Delta \text{AUC}$ (left) and $\Delta \text{xAUC}$ (right) with 95\% empirical CIs for the model trained on $y$ (gray) and on $\tilde{y}$ (green). As $\Delta\mu$ increases to the right, the AUC and xAUC gaps \textit{widen}. Note that $P_0(\tilde{y} = 0 \mid y=1) = 0$---there are no missed positives in group $a=0$. The number of missed positives first increases, then decreases, since sufficiently high $\Delta\mu$ means that untested patients' risk largely lies in one tail of the risk distribution.}%
    \label{fig:md}
\end{figure}

We then evaluate the effect of mean difference on performance disparities.
As between-group mean difference increases, disparities between groups generally worsen with respect to oracle performance. We increase the per-element mean difference between the distributions from 0 to 0.4, keeping the decision boundary constant. Figure~\ref{fig:md} shows that as mean difference $\Delta\mu$ increases, disparities in AUC, xAUC emerge. At $\Delta\mu = 0$, the performance is identical to the oracle, as expected. At $\Delta\mu = 0$, we are in Setting 1: there are no marginal or conditional differences between groups. However, as $\Delta\mu$ grows to 0.3, the AUC and xAUC gaps increase to 0.03, 0.14, respectively. %

Our results suggest that distributional distance between patient subgroups as measured using mean difference and variance is correlated with performance disparities. Characterizing distributional differences in patient subgroup risk could provide expected levels of performance disparities under disparate censorship. Full results are provided in Table~\ref{tab:md} (differences in $\Delta\mu$) and Table~\ref{tab:var} (differences in $\sigma^2$).

\begin{table}[t]
    \centering
    \small
    \begin{tabular}{c|c | cc}
    \toprule
        \textbf{Intergroup mean diff. ($\Delta \mu$)} & \textbf{Group} & \textbf{AUC} & \textbf{xAUC} \\
    \midrule
        \multirow{3}{*}{0} & $a=0$ & 95.8 (95.2, 96.2) & 95.7 (95.0, 96.2)\\ 
        & $a=1$ & 95.8 (95.2, 96.3) & 95.9 (95.2, 96.3)\\
        & \cellcolor{gray!50} $\Delta$ & \cellcolor{gray!50} 0.1 (0.0, 0.4) & \cellcolor{gray!50} 0.2 (0.0, 0.5) \\
    \midrule
        \multirow{3}{*}{0.1} & $a=0$ & 94.7 (93.8, 95.4) & 97.1 (96.5, 97.5) \\ 
        & $a=1$ & 95.8 (95.0, 96.3) & 92.4 (91.1, 93.2) \\
        & \cellcolor{gray!50} $\Delta$ & \cellcolor{gray!50} 1.0 (0.7, 1.5) & \cellcolor{gray!50} 4.8 (4.2, 5.6) \\
    \midrule
        \multirow{3}{*}{0.2} & $a=0$ & 95.1 (94.0, 95.7) & 98.6 (98.3, 98.8)\\ 
        & $a=1$ & 96.8 (96.2, 97.3) & 89.5 (87.8, 90.7)\\
        & \cellcolor{gray!50} $\Delta$ & \cellcolor{gray!50} 1.8 (1.3, 2.3) & \cellcolor{gray!50} 9.2 (8.0, 10.6)\\
    \midrule
        \multirow{3}{*}{0.3} & $a=0$ & 96.5 (95.8, 97.1) & 99.6 (99.5, 99.7)\\ 
        & $a=1$ & 98.5 (98.1, 98.7) & 88.2 (86.0, 89.9)\\
        & \cellcolor{gray!50} $\Delta$ & \cellcolor{gray!50} 1.9 (1.5, 2.5) & \cellcolor{gray!50} 11.4 (9.8, 13.4)\\    
    \bottomrule
    \end{tabular}
    \caption{AUC, xAUC with empirical 95\% CIs at varying settings of $\Delta\mu$, $\tau_0 = 5, \tau_1 = 6.6$.}
    \label{tab:md}
\end{table}

\begin{table}[t]
    \centering
    \small
    \begin{tabular}{c|c | cc}
    \toprule
        \textbf{Within-group var. ($\sigma^2$)} & \textbf{Group} & \textbf{AUC} & \textbf{xAUC} \\
    \midrule
        \multirow{3}{*}{0.05} & $a=0$ & 93.7 (91.6, 95.5) & 98.2 (97.5, 98.7)\\ 
        & $a=1$ & 94.1 (92.7, 95.3) & 81.0 (78.3, 83.5) \\
        & \cellcolor{gray!50} $\Delta$ & \cellcolor{gray!50} 0.5 (0.0, 1.3) & 17.2 (15.2, 19.2) \cellcolor{gray!50} \\
    \midrule
        \multirow{3}{*}{0.1} & $a=0$ & 95.9 (95.1, 96.6) & 98.4 (98.1, 98.7) \\ 
        & $a=1$ & 95.7 (95.0, 96.4) & 89.0 (87.3, 90.4) \\
        & \cellcolor{gray!50} $\Delta$ & \cellcolor{gray!50} 0.2 (0.0, 0.5) & \cellcolor{gray!50} 9.4 (8.2, 10.7) \\
    \midrule
        \multirow{3}{*}{0.15} & $a=0$ & 96.4 (95.9, 96.9) & 98.3 (98.0, 98.6) \\ 
        & $a=1$ & 96.1 (95.5, 96.6) & 91.8 (90.5, 92.9) \\
        & \cellcolor{gray!50} $\Delta$ & \cellcolor{gray!50} 0.3 (0.0, 0.7) &  \cellcolor{gray!50} 6.5 (5.6, 7.5) \\
    \midrule
        \multirow{3}{*}{0.2} & $a=0$ & 96.7 (96.3, 97.2) & 98.2 (97.9, 98.6)\\ 
        & $a=1$ &  96.4 (95.9, 97.0) & 93.4 (92.5, 94.4) \\
        & \cellcolor{gray!50} $\Delta$ & \cellcolor{gray!50} 0.2 (0.0, 0.6) & \cellcolor{gray!50} 4.8 (4.1, 5.4)\\    
    \bottomrule
    \end{tabular}
    \caption{AUC, xAUC with empirical 95\% CIs at varying settings of $\sigma^2$, $\tau_0 = 5, \tau_1 = 6.6$.}
    \label{tab:var}
\end{table}

\section{Code}
For reproducibility, code used to generate all figures and experimental results after review is provided at the MLD3 Github.

\end{document}